\documentclass{article}

\PassOptionsToPackage{numbers, compress}{natbib}
\usepackage[preprint]{neurips_2023}

\usepackage[utf8]{inputenc} 
\usepackage[T1]{fontenc}    
\usepackage{hyperref}       
\usepackage{url}            
\usepackage{booktabs}       
\usepackage{amsfonts}       
\usepackage{nicefrac}       
\usepackage{microtype}      
\usepackage{xcolor}         
\usepackage[pdftex]{graphicx}
\usepackage{amsthm}
\usepackage{amsmath}
\newtheorem{prop}{Proposition}
\newtheorem{remark}{Remark}
\usepackage{algorithm} 
\usepackage{algpseudocode} 
\newtheorem{lemma}{Lemma} 
\usepackage{threeparttable} 
\usepackage{makecell} 
\usepackage{multirow} 
\usepackage{soul} 
\usepackage{float}
\usepackage{subfig}
\usepackage{wrapfig}
\usepackage{authblk} 

\usepackage[commandnameprefix=always, final]{changes} 

\algdef{SE}[REPEATN]{RepeatN}{EndRepeatN}[1]{\algorithmicrepeat\ #1 \textbf{times}}{\algorithmicend \textbf{ repeat}} 

\DeclareSubrefFormat{kuo}{#1(#2)}   
\robustify\subref

\usepackage{enumitem} 

\title{
Guided Cooperation in Hierarchical Reinforcement Learning via Model-based Rollout
}

\author[1]{
  Haoran~Wang
}
\author[1]{
  Zeshen~Tang
}
\author[1]{
  Leya~Yang
}
\author[ ,1]{
  Yaoru~Sun\thanks{Corresponding author: Yaoru Sun.}
}
\author[2]{
  Fang~Wang
}
\author[1]{
  \\Siyu~Zhang
}
\author[1]{
  Yeming~Chen
}
 \affil[1]{Department of Computer Science and Technology, Tongji University\protect\\ \texttt{\{1910664, 2011610, 2152827, yaoru, zsyzsy, 2130769\}@tongji.edu.cn}}
 
 \affil[2]{Department of Computer Science, Brunel University\protect\\ \texttt{fang.wang@brunel.ac.uk}}

\begin{document}

\maketitle

\begin{abstract}
Goal-conditioned hierarchical reinforcement learning (HRL) presents a promising approach for enabling effective exploration in complex, long-horizon reinforcement learning (RL) tasks \chreplaced{through}{via} temporal abstraction. \chreplaced{Empirically,}{In essence, for} \chreplaced{heightened}{increased} inter-level communication and coordination can induce more stable and robust policy improvement in hierarchical systems. Yet, most existing goal-conditioned HRL algorithms \chadded{have primarily} focused on the subgoal discovery, \chreplaced{neglecting}{regardless of} \chreplaced{inter-level cooperation.}{inter-level coupling.} Here, we propose a goal-conditioned HRL framework named Guided Cooperation via Model-based Rollout (GCMR)\footnote[1]{Code is available at \url{https://github.com/HaoranWang-TJ/GCMR_ACLG_official} \label{link_us}}, \chreplaced{aiming to bridge inter-layer information synchronization and cooperation by exploiting forward dynamics.}{which estimates forward dynamics to promote inter-level cooperation.} \chadded{Firstly, }\chreplaced{t}{T}he GCMR \chreplaced{mitigates}{alleviates} the state-transition error within off-policy correction \chreplaced{via}{through a} model-based rollout, \chreplaced{thereby enhancing}{further improving the} sample efficiency. \chreplaced{Secondly}{Meanwhile}, \chreplaced{to prevent disruption by}{to avoid being disrupted by} \chdeleted{these corrected }\chreplaced{the}{but possibly} unseen\chdeleted{ or faraway} subgoals and states, lower-level Q-function gradients are constrained using a gradient penalty with a model-inferred upper bound, leading to a more stable behavioral policy \chadded{conducive to effective exploration}. \chreplaced{Thirdly}{Besides}, we propose a one-step rollout-based planning\chdeleted{ to further facilitate inter-level cooperation}, \chreplaced{using}{wherein the } higher-level \chreplaced{critics}{Q-function} \chreplaced{to guide}{is used to guide} the lower-level policy\chreplaced{. Specifically, we estimate}{ by estimating } the value of future states \chadded{of the lower-level policy} \chadded{using the higher-level \chreplaced{critic }{Q-}function, thereby transmitting} \chdeleted{so that }global task information\chdeleted{ is transmitted} downwards to avoid local pitfalls. \chadded{These three critical components in GCMR are expected to facilitate inter-level cooperation significantly.}
Experimental results demonstrate that incorporating the proposed GCMR framework with a disentangled variant of HIGL, namely ACLG, yields more stable and robust policy improvement \chreplaced{compared to}{than} various baselines and \chreplaced{significantly}{substantially} outperforms previous state-of-the-art \chdeleted{(SOTA) HRL  }algorithms\chdeleted{ in both hard-exploration problems and robotic control}.
\end{abstract}

\section{Introduction}
Hierarchical reinforcement learning (HRL) has made significant contributions toward solving complex and long-horizon tasks with sparse rewards. Among HRL frameworks, goal-conditioned HRL is an especially promising paradigm for goal-directed learning in a divide-and-conquer manner\chdeleted{\cite{dayan1992feudal}}\cite{vezhnevets2017feudal}. In goal-conditioned HRL, multiple goal-conditioned policies are stacked hierarchically, where the higher-level policy assigns a directional subgoal to the lower-level policy, and the lower strives toward it. Recently, related advances \cite{zhang2020generating, zhang2022adjacency, li2021learning, leed2022hrl, guo2021state, huang2019mapping, eysenbach2019search, zhang2021world} have made significant progress in improving exploration efficiency via reachable subgoal generation with adjacency constraint \cite{zhang2020generating, zhang2022adjacency}\chadded{, long-term decision-making with state-temporal compression \cite{guo2021state},} and graph-based planning \cite{csimcsek2005identifying, huang2019mapping, eysenbach2019search, zhang2021world}. \chadded{Furthermore, the challenges posed by local optima and transient traps \cite{gao2022partial} can also be mitigated by boosting exploration through auxiliary rewards\cite{burda2019exploration, ren2021orientation, kim2024accelerating}.} Yet, integrating these strategies with an off-policy learning method still struggles with being sample efficient. Previous research handles this issue by relabeling experiences with more faithful goals \cite{nachum2018data, levy2019learning, zhu2021mapgo}, in which the relabeling aims to explore how to match the higher-level intent and the actual outcome of the lower-level subroutine. The HER-style approaches \cite{andrychowicz2017hindsight, levy2019learning, zhu2021mapgo} overwrite the former with the latter, while the HIRO \cite{nachum2018data} modifies the past instruction to adapt to the current behavioral policy, thereby improving the data-efficiency. \chadded{Another effective approach for enhancing data efficiency is to leverage model-based transition dynamics in planning. Recent advancements in dynamics generalization \cite{lee2020context, seo2020trajectory} have demonstrated that unseen latent dynamics empower agents with robust generalization capabilities, effectively adapting to unseen scenarios. Related research in autonomous navigation systems \cite{kahn2021badgr, kahn2021land} has indicated that learning dynamics can ensure the robustness of long-term planning against unpredictable changes and perturbations in open environments.
However, there has been limited literature \cite{zhang2016learning, mordatch2016combining, nair2020hierarchical, zhu2021mapgo} studying the model exploitation in the goal-conditioned RL. Several studies have attempted to deploy model-based hierarchical reinforcement learning on various real-world control tasks, such as autonomous aerial vehicles\cite{zhang2016learning} and physical humanoids\cite{mordatch2016combining}. These studies employed model predictive control (MPC) for high-level motion planning and model-free RL for motion primitive learning\cite{xie2016model, mordatch2016combining}. However, to our knowledge, there is no prior work studying inter-level model-based cooperation in goal-conditioned HRL.}

\chadded{Within a hierarchical architecture, promoting inter-level cooperation is a notably more effective approach to accelerating reinforcement learning\cite{setyawan2022cooperative, setyawan2022depth, setyawan2022combinations}.} \chreplaced{To achieve}{It can be seen that, empirically, an inter-level association and communication mechanism bridging the gap between the intent and outcome is essential for robust policy improvement. For achieving} inter-level cooperation and communication, addressing the following questions is essential: 1) how does the lower level comprehend and synchronize with the higher level? 2) how does the lower level \chreplaced{maintain robustness in the face of}{handle} errors from the higher level? 3) how does the lower level directly understand the overall task without relying on higher-level proxies? \chadded{Empirically, a learned dynamics model can function as the inter-level communication mechanism to bridge} the gap between the \chadded{high-level} intent and \chadded{final (low-level)} outcome, leading to robust policy improvement.

In this paper, we propose a novel goal-conditioned HRL framework to systematically \chdeleted{address these questions}facilitate inter-level cooperation, which mainly consists of three crucial components: 1) the goal-relabelling for synchronizing, 2) the gradient penalty for enhancing robustness against high-level errors, and 3) the one-step rollout-based planning for \chreplaced{transmitting global tasks downwards}{collectively working towards achieving a global task}. The key insight in the proposed framework is to modularly integrate a forward dynamics prediction model into HRL framework for improving the data efficiency and enhancing learning efficiency. Specifically, our framework, named \textbf{G}uided \textbf{C}ooperation via \textbf{M}odel-based \textbf{R}ollout (GCMR), brings together three crucial ingredients:
\begin{itemize}[leftmargin=2em]
\item \chreplaced{We propose a novel model-based rollout-based}{A novel} off-policy correction\chdeleted{ based on weighted model-based rollout}\chadded{, which was deployed to mitigate the cumulative state-transition error in HIRO \cite{nachum2018data}}. Additionally, we \chreplaced{also propose \chreplaced{a trick, }{two tricks, the exponential weighting and }soft goal-relabeling, to \chdeleted{suppress the cumulative error in long-horizon rollouts and }make the correction more robust to outliers}{ employed exponential weighting to suppress the cumulative error in long-horizon rollouts, assigning higher weights to the initial transitions. We also propose a convenient trick, soft goal-relabeling, to make correction robust to outliers. All of these contribute to further improving the data efficiency}. 
\item We propose a gradient penalty to suppress sharp lower-level Q-function gradients, which clamps the Q-function gradient by means of an inferred upper bound. \chreplaced{The gradient penalty implicitly constrained the behavioral policy to change steadily, enhancing the stability of the optimization}{Consequently, the behavioral policy is implicitly constrained to change steadily, yielding conservative action. This, in turn, enhances the stability of the hierarchical optimization}.
\item Meanwhile, we designed a one-step rollout-based planning method to prevent the lower-level policy from getting stuck in local optima\chreplaced{, wherein the values of future transitions of lower-level agents were evaluated using the higher-level critics. Such foresight helps the lower-level policy cooperate better with the goal planner.}{. Specifically, we use the learned dynamics to simulate one-step rollouts and evaluate task-specific values of future transitions through the higher-level Q-function. Such foresight and planning help the lower-level policy cooperate better with the goal planner.}
    
\end{itemize}

To justify the superiority of the proposed \chreplaced{GCMR}{method}\chdeleted{ and achieve a remarkable SOTA}, we integrate \chreplaced{it}{the GCMR} with a strong baseline: ACLG, a disentangled variant of HIGL \cite{kim2021landmark}. Experimental results show that incorporating the proposed framework \chreplaced{achieved}{improves the} \chreplaced{state-of-the-art performance}{performance of state-of-the-art HRL algorithms} in complex and sparse reward environments. 
\chadded{The contributions of this article are summarized as follows:}
\begin{itemize}[leftmargin=2.2em]
\item[1)] \chadded{This article proposes GCMR, a novel method designed to facilitate inter-level cooperation in HRL, thereby accelerating learning.}
\item[2)] \chadded{We benchmark our method on various long-horizon control and planning tasks, which are commonly used in the HRL literature\cite{florensa2017stochastic, nachum2018data, nachum2019near, zhang2020generating, kim2021landmark, leed2022hrl,yang2021hierarchical, zhang2022adjacency, zeng2023ahegc}. Extensive experiments demonstrate the superior performance of our proposed method. Moreover, we conduct sufficient ablation studies to validate the contribution of different components in GCMR.}
\item[3)] \chadded{Our study emphasizes the importance of inter-level cooperation, contributing to new insights in hierarchical RL.}
\end{itemize}
\chadded{The rest of this article is structured as follows. Section \ref{sec:Preliminaries} introduces the preliminaries on HRL, adjacency constraint, landmark-based planning, and our proposed disentangled variant of HIGL. Section \ref{related_work} provides a review of related works, focusing on transition relabeling and model exploitation in goal-conditioned HRL. Section \ref{methods} offers a detailed implementation of the proposed GCMR method. In Section \ref{experiments}, we outline our experimental environments, explore the impact of different parameters in ACLG and GCMR, and present our main experimental results along with ablation analyses. Section \ref{Discussion} discusses the principal findings, highlights several limitations of this study, and outlines potential directions for future research. Section \ref{Conclusion} presents our conclusions.}

\section{Preliminaries}
\label{sec:Preliminaries}

Consider a finite-horizon, goal-conditioned Markov decision process (MDP) represented by a tuple $\left(\mathcal{S}, \mathcal{G}, \mathcal{A}, \mathcal{P}, \mathcal{R} \right)$, where $\mathcal{S}$, $\mathcal{G}$, and $\mathcal{A}$ denote the state space, goal space, and action space, respectively. The transition function $\mathcal{P}: \mathcal{S} \times \mathcal{A} \rightarrow \mathcal{S}$ defines the transition dynamics of environment, and the $\mathcal{R}: \mathcal{S} \times \mathcal{A} \times \mathcal{G} \rightarrow \mathbb{R}$ is the reward function. Specifically, the environment will transition from $s_t \in \mathcal{S}$ to a new state $s_{t+1} \in \mathcal{S}$ while yielding a reward $R_t \in \mathcal{R}$ once it takes an action $a_t \in \mathcal{A}$, where $s_{t+1} \sim \mathcal{P}\left(s_{t+1}|a_{t+1}, a_{t}\right)$ and $R_t$ is conditioned on a final goal $g \in \mathcal{G}$. In most real-world scenarios, complex tasks can often be decomposed into a sequence of simpler movements and interactions. Therefore, we formulate a hierarchical reinforcement learning framework, which typically has two \chreplaced{layers}{hierarchies}: higher- and lower-level policies, to deal with these challenging tasks. The higher-level policy observes the state $s_{t}$ of environment and produces a high-level action $sg_t$, i.e., a subgoal indicating a desired change of state or absolute location to reach every $c$ time steps. The lower-level policy attempts to reach these assigned subgoals within a $c$ time interval. Suppose that the higher-level policy and lower-level policy are parameterized by neural networks with parameters $\theta_{hi}$ and $\theta_{lo}$, respectively. The above procedure of the higher-level controller can be formulated: $sg_t \sim \pi \left(sg|s_t, g;\theta_{hi}\right) \in \mathcal{G}$ when $t \equiv 0$ (mod $c$). The lower-level policy observes the state $s_t$ as well as subgoal $sg_t$ and then yields a low-level atomic action to interact directly with the environment: $a_t \sim \pi \left(a|s_t,sg_t;\theta_{lo}\right) \in \mathcal{A}$. Notably, for the relative subgoal scheme, subgoals evolve following a pre-defined subgoal transition process $sg_t = h\left(sg_{t-1},s_{t-1},s_t\right) = sg_{t-1} + \varphi(s_{t-1} - s_t)$ when $t \not\equiv 0$ (mod $c$), where $\varphi: \mathcal{S} \rightarrow \mathcal{G}$ is a known mapping function that transforms a state into the goal space. The pre-defined transition makes the lower-level agent seem completely self-contained and like an autonomous dynamical system.

\subsection{Parameterized Rewards} During interaction with the environment, the higher-level agent makes a plan using subgoals and receives entire feedback by accumulating all external rewards within the planning horizon:
\begin{equation}
r^{hi}_t=\sum_{i=t}^{t+c-1} R_i \left(s_i,a_i,g\right)
\end{equation}
The lower-level agent is intrinsically motivated in the form of internal reward that evaluates subgoal-reaching performance:
\begin{equation}
r^{lo}_t=-\Vert sg_{t+1} -  \eta \varphi \left( s_{t+1} \right)\Vert_2
\label{intrinsic_reward}
\end{equation}
Where $\eta$ denotes a Boolean hyper-parameter whose value is 0/1 for the relative/absolute subgoal scheme.

\subsection{Experience Replay for Off-Policy Learning}
Experience replay has been the fundamental component for off-policy RL algorithms, which greatly improves the sample efficiency by reusing previously collected experiences. Here, there is no dispute that the lower-level agent can collect the experience $\tau_{lo} = \left( \left \langle s_t, sg_t \right \rangle, a_t, r^{lo}_t, \left \langle s_{t+1}, sg_{t+1} \right \rangle \right)$ by using the behavioral policy to directly interact with the environment. The higher-level agent interacts indirectly with it through the lower-level proxy and then stores a series of state-action transitions as well as a cumulative reward, i.e., $\tau_{hi} = \left( \left \langle s_{t:t+c-1}, g \right \rangle, sg_{t:t+c-1}, r^{hi}_t, \left \langle s_{t+c}, g \right \rangle \right)$, into the high-level replay buffer. The lower- and higher-level policies can be trained by sampling transitions stored in these experience replay buffers $\mathcal{D}_{lo}$, $\mathcal{D}_{hi}$. The aim of optimization is to maximize the expected discounted reward $\mathbb{E}_{L\in \left \{ lo, hi\right \} } \left[ \sum^{\infty}_{t=0} \gamma^i r^L_t\right]$, where $\gamma \in \left[0, 1 \right]$ is the discount factor. In practice, we instantiate lower- and higher-level agents based on the TD3 algorithm \cite{fujimoto2018addressing}, each having a pair of online critic networks with parameters $\phi_1$ and $\phi_2$, along with a pair of target critic networks with parameters $\phi^{\prime}_1$ and $\phi^{\prime}_2$. Additionally, TD3 has a single online actor parameterized by $\theta$ and a target actor parameterized by $\theta^{\prime}$. All target networks are updated using a soft update approach. Then, the Q-network can be updated by minimizing the mean squared temporal-difference (TD) error over all sampled transitions. To simplify notation, we adopt unified symbols $o^{L}_t$ and $a^L_t$ to indicate the observation and performed action, where $\left \langle o^{L}_t, a^L_t\big|_{L=lo}\right \rangle=\left \langle \left \langle s_t, sg_t \right \rangle, a_t \right \rangle$ for lower-level while $\left \langle o^{L}_t, a^L_t\big|_{L=hi}\right \rangle=\left \langle \left \langle s_t, g \right \rangle, sg_t \right \rangle$ for higher-level. Hence, the Q-learning loss can be written as follows:
\begin{equation}
\mathcal{L}(\phi_{i, L}) = \mathbb{E}_{\tau_{L} \sim \mathcal{D}_{L}} \left[Q \left(o^{L}_t, a^L_t; \phi_{i, L} \right) - y^{L}_t \right]^2 \Big|_{\substack{L\in \left \{ lo, hi\right \} \\ i\in \left \{1, 2\right \}}}
\label{critic_loss}
\end{equation}
Where $y^{L}_t$, i.e., $y^{lo}_t$ or $y^{hi}_t$, is dependent on $\theta_{lo}$ or $\theta_{hi}$ correspondingly because target policies map states to the "optimal" actions in an almost deterministic manner:
\begin{equation}
\begin{aligned}
y^{L}_t=r^{L}_t + \gamma \min_{i=1,2} &Q \left(o^{L}_{t^{\prime}}, \pi \left(o^{L}_{t^{\prime}};\theta^{\prime}_{L} \right)+\varepsilon; \phi^{\prime}_{i,L} \right) \Big|_{L\in \left \{ lo, hi\right \}}\\
{\rm with} \quad \varepsilon & \sim {\rm clip}(\mathcal{N}(0,\sigma), -a_c, a_c)
\label{critic_loss_q}
\end{aligned}
\end{equation}
Where $\sigma$ is the s.d. of the Gaussian noise, $a_c$ defines the range of the auxiliary noise, and $o^{L}_{t^{\prime}}$ refers to the next obtained observation after taking an action. It is noteworthy that $t^{\prime}=t+c$ with respect to the higher-level while $t^{\prime}=t+1$ for the lower-level. Drawing support from Q-network, the policy can be optimized by minimizing the following loss:
\begin{equation}
\mathcal{L}(\theta_{L}) = -\mathbb{E}_{\tau_{L} \sim \mathcal{D}_{L}}\left[ Q \left(o^{L}_{t}, \pi \left(o^{L}_{t};\theta_{L} \right); \phi_{1,L} \right) \right] \Big|_{L\in \left \{ lo, hi\right \}}
\label{actor_loss}
\end{equation}
As mentioned above, we outline the common actor-critic approach with the deterministic policy algorithms \cite{timothy2016continuous, fujimoto2018addressing}. For more details, please refer to \cite{fujimoto2018addressing}.

\subsection{Adjacency Constraint}
For the high-level subgoal generation, reachability within $c$ steps is a sufficient condition for facilitating reasonable exploration.  Zhang et al. \cite{zhang2020generating, zhang2022adjacency} mined the adjacency information from trajectories gathered by the changing behavioral policy over time\chdeleted{ during the training procedure}. In that study, a $c$-step adjacency matrix was constructed to memorize $c$-step adjacent state-pairs appearing in these trajectories. To ensure this procedure is differentiable and can be generalized to newly-visited states, the adjacency information stored in such matrix was further distilled into an adjacency network $\psi$ parameterized by $\Phi$. Specifically, the adjacency network approximates a mapping from a goal space into an adjacency space. Subsequently, the resulting embeddings can be utilized to measure whether two states are $c$-step adjacent using the Euclidean distance. For example, the $c$-step adjacent estimation (or shortest transition distance) of two states $s_i$ and $s_j$ can be calculated as: $d_{st}(s_i, s_j; \Phi) \approx \frac{c}{\zeta_c}||\psi_{\Phi}(\varphi(s_i))-\psi_{\Phi}(\varphi(s_j))||_2$, where $\zeta_c$ is a scaling factor and the $\varphi$ function maps states into the goal space. Such an adjacency network can be learned by minimizing the following contrastive-like loss: $\mathcal{L}_{{\rm adj}}(\Phi)=\mathbb{E}_{s_i,s_j\in\mathcal{S}}[l\cdot \max(||\psi_{\Phi}(\varphi(s_i))-\psi_{\Phi}(\varphi(s_j))||_2-\zeta_c, 0)+(1-l)\cdot \max(\zeta_c+\delta_{\rm adj}-||\psi_{\Phi}(\varphi(s_i))-\psi_{\Phi}(\varphi(s_j))||_2, 0)]$, where $\delta_{\rm adj} > 0 $ is a hyper-parameter indicating a margin between embeddings and $l \in \{0,1\}$ is the label indicating whether $s_i$ and $s_j$ are \chreplaced{$c$}{$k$}-step adjacent.

\subsection{Landmark-based Planning}
Graph-based navigation has become a popular technique for solving complex and sparse reward tasks by providing a long-term horizon. The relevant frameworks \cite{savinov2018semi, huang2019mapping, eysenbach2019search, emmons2020sparse, yang2020plan2vec, kim2021landmark, zhang2021world,leed2022hrl, kim2022imitating} commonly contain two components: (a) a graph built by sampling landmarks and (b) a graph planner to select waypoints. In a graph, each node corresponds to an observation state, while edges between nodes are weighted using a distance estimation. Specifically, a set of observations randomly subsampling from the replay buffer are organized as nodes, where high-dimensional samples (e.g., images) may be embedded into low-dimensional representations \cite{savinov2018semi, eysenbach2019search, liu2020hallucinative, yang2020plan2vec, zhang2021world}. However, operations over the direct subsampling of the replay buffer will be costly. The state aggregation \cite{emmons2020sparse} and landmark sampling based on farthest point sampling (FPS) were proposed for further sparsification \cite{huang2019mapping,kim2021landmark,leed2022hrl,kim2022imitating}. Our study follows prior works of Kim et al. \cite{kim2021landmark, kim2022imitating}, in which FPS was employed to select a collection of landmarks, i.e., \textit{coverage-based landmarks ${LM}^{\rm cov}$}, from the replay buffer. In addition to this, HIGL \cite{kim2021landmark} used random network distillation \cite{burda2019exploration} to explicitly sample novel landmarks, i.e., \textit{novelty-based landmarks ${LM}^{\rm nov}$}, a set of states rarely visited in the past. Hence, the final collection of landmarks was ${LM}={LM}^{\rm cov} \cup {LM}^{\rm nov}$. Once landmarks are added to the graph, the edge weight between any two vertices can be estimated by a lower-level value function \cite{huang2019mapping, eysenbach2019search,kim2021landmark,leed2022hrl,kim2022imitating} or the (Euclidean-based or contrastive-loss-based) distance between low-dimensional embeddings of states \cite{savinov2018semi,yang2020plan2vec,liu2020hallucinative,zhang2020generating}. Following prior works \cite{huang2019mapping,kim2021landmark,kim2022imitating}, in this study, we estimated the edge weight via the lower-level value function, i.e., $-V_{lo}(s_i, \varphi(s_j)) \approx -Q(s_i, \varphi(s_j), \pi \left(a|s_i,\varphi(s_j);\theta_{lo}\right); \phi_{lo}), \forall s_i, s_j \in LM$. After that, unreachable edges were clipped  by a preset threshold \cite{huang2019mapping}. In the end, the shortest path planning algorithm was run to plan the next subgoal, the very first landmark in the shortest path from the current state $s_t$ to the goal $g$: 
\begin{equation}
sg_{t}^{\rm plan}=\arg\min_{\varphi(s_i)}-[V_{lo}(s_t, \varphi(s_i))+V_{lo}(s_i, g)], \qquad\text{s.t.} \quad \forall s_i \in {LM}^{\rm cov} \cup {LM}^{\rm nov}
\label{HIGL_landmarks}
\end{equation}

\subsection{\chadded{Disentangled Variant of HIGL \cite{kim2021landmark}: \textbf{A}djacency \textbf{C}onstraint and \textbf{L}andmark-\chreplaced{\textbf{G}uided p}{based \textbf{P}}lanning (\textbf{ACLG})}}

\begin{figure*}[htbp]
\captionsetup[subfloat]{format=hang, justification=centering}
\centering
\includegraphics[width=0.9\textwidth]{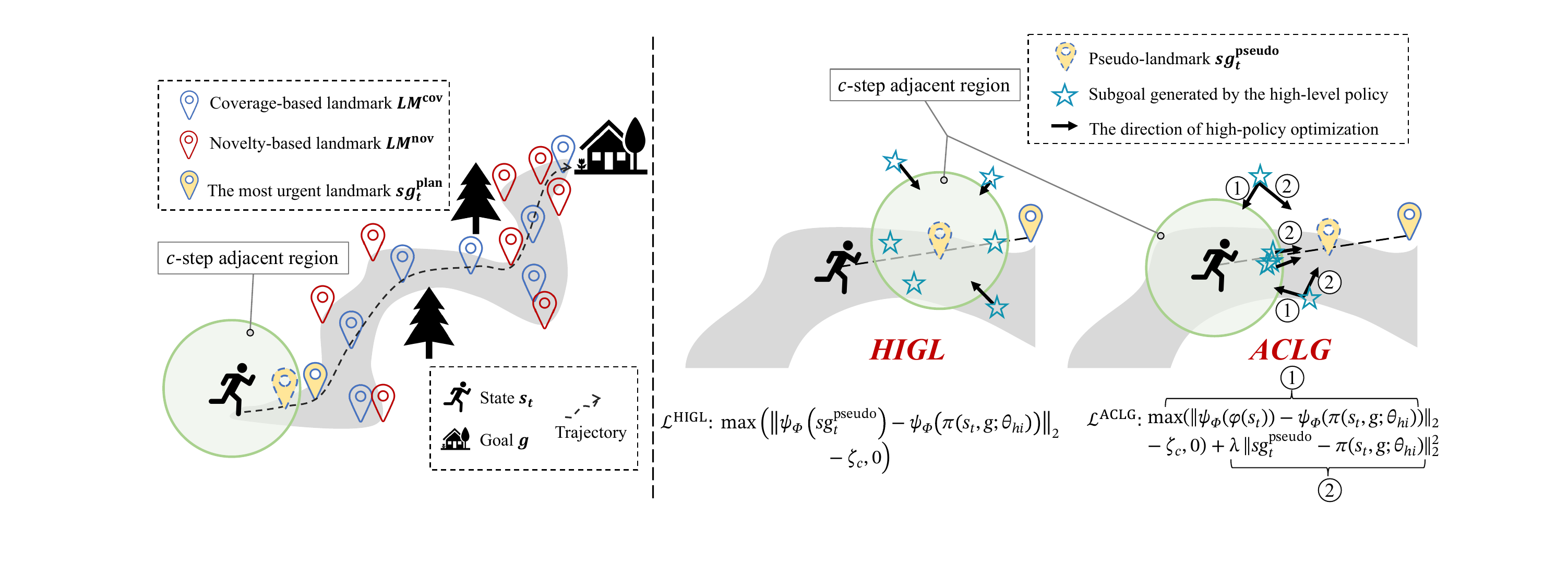}
\caption{\chadded{Illustrations of HIGL \cite{kim2021landmark} and adjacency constraint and landmark-guided planning (ACLG). In HIGL, coverage- and novelty-based landmarks are selected to form a map, from which the most urgent landmark is chosen as the next expected subgoal. Meanwhile, to ensure the reachability of subgoals, HIGL introduces the adjacent constraint. However, in HIGL, the entanglement between the adjacency constraint and landmark-based planning only compels the subgoals to move towards the selected landmark, without guaranteeing $c$-step adjacency with the current state. The ACLG decouples the two to provide a better balance between the adjacency and landmark-based planning.}}
\label{higl_aclg_vis}
\end{figure*}

In HIGL, after finding a landmark through landmark-based planning (see Equation~\ref{HIGL_landmarks}), the raw selected landmark was shifted towards the current state $s_t$ for reachability: $sg_{t}^{\rm pseudo}= sg_{t}^{\rm plan} + \delta_{\rm pseudo} \cdot \frac{sg_{t}^{\rm plan}-\varphi(s_t)}{||sg_{t}^{\rm plan}-\varphi(s_t)||_2}$, where $\delta_{\rm pseudo}$ denotes the shift magnitude. Then, the higher-level policy was guided to generate subgoals adjacent to the planned landmarks and the landmark loss of HIGL was formulated as:
\begin{equation}
\mathcal{L}^{\rm HIGL}_{\rm landmark}(\theta_{hi}) = \lambda^{\rm HIGL}_{\rm landmark} \cdot \max(||\psi_{\Phi}(sg_{t}^{\rm pseudo}) -\psi_{\Phi}(\pi(s_t, g;\theta_{hi}))||\chadded{_{2}} - \zeta_c, 0)
\label{higl_loss}
\end{equation}
Here, Kim et al.\chadded{\cite{kim2021landmark}} employed the adjacency constraint to encourage the generated subgoals to be in the $c$-step adjacent region to the planned landmark. However, the entanglement between the adjacency constraint and landmark-based planning limited the performance of HIGL. 

Inspired by PIG \cite{kim2022imitating}, we proposed a disentangled variant of HIGL. We only made minor modifications to the landmark loss (see Equation \ref{higl_loss}) of HIGL:
\begin{equation}
\begin{aligned}
\mathcal{L}^{\rm ACLG}&(\theta_{hi}) \\= & \lambda_{\rm adj} \cdot \max(||\psi_{\Phi}(\varphi(s_t))-\psi_{\Phi}(\pi(s_t, g;\theta_{hi}))||\chadded{_{2}} - \zeta_c, 0) \\ \qquad &\qquad+ \lambda^{\rm ACLG}_{\rm landmark} \cdot ||sg_{t}^{\rm pseudo} - \pi(s_t, g;\theta_{hi}) ||^2_2
\end{aligned}
\end{equation}
The former term is the adjacency constraint and the latter is the landmark-guided loss, so the proposed method was called ACLG. The hyper-parameters $\lambda_{\rm adj}$ and $\lambda^{\rm ACLG}_{\rm landmark}$ were introduced to better balance the adjacency constraint and landmark-based planning. \chadded{The illustrations (see Fig.~\ref{higl_aclg_vis}) visually demonstrate the differences between HIGL and ACLG.}

\section{Related work}
\label{related_work}

\subsection{Transition Relabeling}
Training a hierarchy using an off-policy algorithm remains a prominent challenge due to the non-stationary state transitions \cite{nachum2018data, levy2019learning, zhu2021mapgo}. Specifically, the higher-level policy takes the same action under the same state but could receive markedly different outcomes because of the low-level policy changing, so the previously collected transition tuple is no longer valid. To address the issues, HIRO \cite{nachum2018data} deployed an off-policy correction to maintain the validity of past experiences, which relabeled collected transitions with appropriate high-level actions chosen to maximize the probability of the past lower-level actions. Alternative approaches, HAC \cite{andrychowicz2017hindsight, levy2019learning}, replaced the original high-level action with the achieved state (projected to the goal space) in the form of hindsight. However, HAC-style relabeled subgoals are \chdeleted{only }compatible with the past low-level policy rather than the current one, deteriorating the non-stationarity issue.

Our work is related to HIRO \cite{nachum2018data}, and the majority of modification is that we roll out the off-policy correction using learned transition dynamics to suppress the accumulative error. The closest work is the MapGo \cite{zhu2021mapgo}, a model-based HAC-style framework in which the original goal was replaced with a foresight goal by reasoning using an ensemble transition model. Our work differs in that we screen out a faithful subgoal that induces rollout-based action sequence similar to the past transitions, while the MapGo overwrites the subgoal with a foresight goal based on the model-based rollout. Meanwhile, our framework proposes a gradient penalty with model-inferred upper bound to prohibit the disturbance caused by relabeling to the behavioral policy. 

\subsection{Model Exploitation in Goal-conditioned HRL}
The promises of model-based RL (MBRL) have been extensively discussed in past research \cite{moerland2023model}. The well-known model-based RL algorithm, Dyna \cite{sutton1991dyna}, leveraged a dynamics model to generate one-step transitions and then update the value function using these imagined data, thus  accelerating the learning. Recently, instantiating environment dynamics\chdeleted{model} using an ensemble of probabilistic networks has become quite popular because of its ability to model both aleatory uncertainty and epistemic uncertainty \cite{chua2018deep}. Hence, a handful of Dyna-style methods proposed to simulate multi-step rollouts by using ensemble models, such as SLBO \cite{luo2019algorithmic} and MBPO \cite{janner2019trust}. Alternatively, the  model-based value expansion methods performed multi-step simulation and estimated the future transitions using the Q-function, which helped to reduce the value estimation error. The representative \chdeleted{value estimation }algorithms include MVE \cite{feinberg2018model} and STEVE \cite{buckman2018sample}. Besides, in fact, the estimated value of states can directly provide gradients to the policy when the learned dynamic models are differentiable, like Guided Policy Search \cite{levine2013guided} and Imagined Value Gradients \cite{byravan2020imagined}. Our work differs from these works since we use the higher-level Q-function to estimate the value of future lower-level transitions. As stated in a recent survey \cite{a2023luo}, there have been only a few works \cite{nair2020hierarchical, zhu2021mapgo} involving the model exploitation in the goal-conditioned RL. To our knowledge, there is no prior work studying such inter-level planning.

\section{Methods}
\label{methods}
This section explains how our framework with Guided Cooperation via Model-based Rollout (GCMR) promotes inter-level cooperation. The GCMR involves three critical components: 1) the off-policy correction via model-based rollouts, 2) gradient penalty with a model-inferred upper bound, and 3) one-step rollout-based planning. Below, we \chdeleted{first introduce a strong baseline, ACLG, and then }detail the architecture of the dynamics model and such three critical components.

\subsection{Forward Dynamics Modeling}
A bootstrapped ensemble of dynamics models is constructed to approximate the true transition dynamics of environment: $f(s_{t+1}|s_t,a_t)$, which has been demonstrated in several studies \cite{chua2018deep, kurutach2018model, janner2019trust, shen2020model, yu2020mopo, yu2021combo}. We denote the dynamics approximators as $\Gamma_{\xi} = \{\hat{f}^1_{\xi}, \dots, \hat{f}^B_{\xi}\}$, where $B$ is the ensemble size and $\xi$ denotes the parameters of models. Each model of the ensemble projects the state $s_t$ conditioned on the action $a_t$ to a Gaussian distribution of the next state, i.e., $\hat{f}^b_{\xi}(s_{t+1}|s_t,a_t)=\mathcal{N}(\mu^b_{\xi}(s_t,a_t),\Sigma^b_{\xi}(s_t,a_t))$, with $b \in \{1,\dots,B\}$. In usage, a model is picked out uniformly at random to predict the next state. Note that, here, we do not learn the 
reward function because the compounding error from multi-step rollouts makes it infeasible for higher-level to infer the future cumulative rewards. As for the lower-level agent, the reward can be computed through the  intrinsic reward function (see Equation \ref{intrinsic_reward}) on the fly. Finally, such dynamics models are trained via maximum likelihood and are incorporated to encourage inter-level cooperation and stabilize the policy optimization process.

\subsection{Off-Policy Correction via Model-based Rollouts}
With well-trained dynamics models, we expand the vanilla off-policy correction in HIRO \cite{nachum2018data} by using the model-generated state transitions to bridge the gap between the past and current behavioral policies. Recall a stored high-level transition $\tau_{hi} = \left( \left \langle s_{t:t+c-1}, g \right \rangle, sg_{t:t+c-1}, r^{hi}_t, \left \langle s_{t+1:t+c}, g \right \rangle \right)$ in the replay buffer, which is converted into a state-action-reward transition: $\tau_{hi} = \left( \left \langle s_t, g \right \rangle, sg_t, r^{hi}_t, \left \langle s_{t+c}, g \right \rangle \right)$ during training. Relabeling either the cumulative rewards or the final state via $c$-step rollouts, resembling the FGI in MapGo \cite{zhu2021mapgo}, substantially suffers from the high variance of long-horizon prediction. In essence, both the final state $s_{t+c}$ and the reward sequence $R_{t:t+c-1}$ are explicitly affected by the action sequence $a_{t:t+c-1}$. Hence, relabeling the $sg_t$, instead of the $s_{t+c}$ or $r^{hi}_t$, with an action sequence-based maximum likelihood is a promising way to improve sample efficiency. Following prior work \cite{nachum2018data}, we consider the maximum likelihood-based action relabeling:
\begin{equation}
\log \pi(a_{t:t+c-1}|s_{t:t+c-1},\tilde{sg}_{t:t+c-1};\theta_{lo}) \propto -\frac{1}{2} \sum^{i+c-1}_{i=t} \Vert a_i - \pi(s_i,\tilde{sg}_i;\theta_{lo}) \Vert^2_2+\rm{const}
\label{opc_hiro}
\end{equation}
Where $\tilde{sg}_t$ indicates the candidate subgoals sampled randomly from a Gaussian centered at $\varphi(s_{t+c})$. Meanwhile, the original goal $sg_t$ and the achieved state (in goal space) $\varphi(s_{t+c})$ are also taken into consideration. Specifically, according to Equation~\ref{opc_hiro}, the current low-level policy performed $c$-step rollouts conditioned on these candidate subgoals. These sub-goals maximizing the similarity between original and rollout-based action sequences will be selected as optimal. Yet, we find that the current behavioral policy cannot produce the same action as in the past, so the $s_{t+1}$\chreplaced{ may not be revisited. }{ will not be reached at all.}Therefore, the vanilla off-policy correction still suffers from the cumulative error due to the gap between the $s_{t+1:t+c}$ and the unknown transitions $\hat{s}_{t+1:t+c}$. In view of this fact, we roll out these transitions using the learned dynamics models $\Gamma_{\xi}$ to mitigate the issue. Besides, we employ an exponential weighting function along the time axis to highlight shorter rollouts and \chreplaced{slowly shift the original states to the rollout-based ones.}{eliminate the effect of cumulative error.} Then Equation~\ref{opc_hiro} is rewritten as:
\begin{equation}
\begin{aligned}
 &\log\pi(a_{t:t+c-1}|s_{t},\tilde{sg}_{t};\theta_{lo}) \propto -\mathbb{E}_{\hat{a}_{i}} \chdeleted{\rho^{i-t} }\cdot \Vert a_i-\hat{a}_i \Vert^2_2 +\rm{const} \\
 &\quad\text{s.t.} \quad \hat{a}_{i} \sim \pi(\hat{s}_{i}, \tilde{sg}_i;\theta_{lo}) ;\\ \chdeleted{\text{ and }} &\quad\quad\quad\hat{s}_{i+1} \sim \chadded{(1-\rho^{i-t})\cdot}\Gamma_{\xi}(\hat{s}_{i}, \hat{a}_{i}) \chadded{+ \rho^{i-t}\cdot{s}_{i+1}}
\end{aligned}
\end{equation}
Where $t \leq i \le t+c-1$ and $\hat{s}_{i}\big|_{i=t} = s_t$. $\rho \in \mathbb{R}$. $\rho \in \mathbb{R}$ is a hyper-parameter indicating the base of the exponential function, where in practice, we set $\rho$ to 0.95.

\textit{\textbf{Soft-Relabeling:}}
Inspired by the pseudo-landmark shift of HIGL \cite{kim2021landmark}, instead of an immediate overwrite, we use a soft mechanism to smoothly update subgoals:
\begin{equation}
 sg_{t}\leftarrow  sg_{t} + \chreplaced{\delta_{sg}}{\delta_{g}}\frac{\Delta sg_{t}}{\Vert \Delta sg_{t} \Vert_2}; \qquad
\Delta sg_{t} := \tilde{sg}_t - sg_{t}
\end{equation}
Where $\chreplaced{\delta_{sg}}{\delta_{g}}$ represents the shift magnitude\chreplaced{ from the original subgoals}{ updated at the constant speed $\epsilon$}. The soft update is expected to be robust to outliers\chdeleted{ and protect the Q-learning, which is highly sensitive to noise}.

\subsection{Gradient Penalty with a Model-Inferred Upper Bound}
Apparently, from the perspective of the lower-level policy, the subgoal relabeling implicitly brings in a distributional shift of observation. Specifically, these relabeled subgoals are sampled from the goal space but are not executed in practice. The behavioral policy is prone to produce unreliable actions under such an unseen or faraway goal, resulting in ineffective exploration. Motivated by \chreplaced{\cite{blonde2022lipschitzness,kumar2019stabilizing, gao2022robust}}{\cite{kumar2019stabilizing, gao2022robust}}, we pose the Lipschitz constraint on the Q-function gradients to stabilize the Q-learning of behavioral policy. To understand the effect of the gradient penalty, we highlight the \chreplaced{Lipschitz}{Llipschitz} property of the learned Q-function.

\begin{prop}
\label{gradient_penalty}
Let $\pi^*(a_t|s_t)$ and $r^*(s_t,a_t)$ be the policy and the reward function in an MDP. Suppose there are the upper bounds of Frobenius norm of the policy and reward gradients w.r.t. input actions, i.e., $\Vert \frac{\partial \pi^*(a_{t+1}|s_{t+1})}{\partial a_t}\Vert_F \leq L_{\pi} < 1 $ and $\Vert \frac{\partial r^*(\chreplaced{s_{t}, a_{t}}{s_{t+1}, a_{t+1}})}{\partial a_t}\Vert_F \leq L_{r}$. Then the gradient of the learned Q-function w.r.t. action can be upper-bounded as:
\begin{equation}
\Vert \nabla_{a_t}Q_{\pi^*}(s_t,a_t) \Vert_F \leq \frac{\sqrt{N}L_r}{1-\gamma L_{\pi}}
\end{equation}
Where $N$ denotes the dimension of the action and $\gamma$ is the discount factor.
\end{prop}

\begin{proof}
See the \nameref{sec:mgp_proof}.
\end{proof}

\begin{remark}
Proposition \ref{gradient_penalty} proposes a tight upper bound. A more conservative upper bound can be obtained by employing the inequality pertaining to $L_{\pi}$:
\begin{equation}
\Vert \nabla_{a_t}Q_{\pi^*}(s_t,a_t) \Vert_F < (1-\gamma)^{-1} \sqrt{N} L_r
\end{equation}
Hence, a core challenge in the applications is how to estimate the upper bound of reward gradients w.r.t. input actions.
\end{remark}

Now, we propose an approximate upper-bound approach grounded on the learned dynamics $\Gamma_{\xi}$. Fortunately, the lower-level reward function is specified in the form of L2 distance and is immune to environment stochasticity. Naturally, the upper bound of reward gradients w.r.t. input actions can be estimated as:
\begin{equation}
\begin{aligned}
\hat{L}_r =& \sup \left \{ \Vert \nabla_{a_t} \Vert sg_{t+1} -  \eta \varphi ( s_{t+1} )\Vert_2 \Vert_F \right \}  \\
 &\text{s.t.}\quad s_{t+1} \in \mathcal{S}, sg_t \in \mathcal{G}, a_t \in \mathcal{A}
 \end{aligned}
\end{equation}
In practice, we approximate the upper bound using a mini-batch of lower-level observations independently sampled from the replay buffer $\mathcal{D}_{lo}$, yielding a tighter upper bound and, in turn, more forcefully penalizing the gradient:
\begin{equation}
\begin{aligned}
\hat{L}_r \simeq& \max \left \{ \Vert \nabla_{a_t} \Vert sg_t+ \varphi(s_t-\Gamma_{\xi}(s_t,a_{t})) \right. \left. -  \eta \varphi (\Gamma_{\xi}(s_t,a_{t}) )\Vert_2 \Vert_F \right \} \\
 &\text{s.t.} \quad s_t, sg_t \sim \mathcal{D}_{lo} \text{ and } a_{t} \sim \pi(s_{t}, sg_t;\theta_{lo})
 \end{aligned}
\end{equation}
Then, following prior works \cite{gao2022robust}, we plug the gradient penalty term into the lower-level Q-learning loss (see Equation \ref{critic_loss}), which can be formulated as:
\begin{equation}
\begin{aligned}
\mathcal{L}_{gp}(\phi_{lo}) &= \lambda_{gp} \cdot \mathbb{E}_{s_t, sg_t}[{\rm ReLU}(\Vert \nabla_{a_t}Q_{\pi}(s_t, sg_t,a_t; \phi_{lo}) \Vert_F - (1-\gamma)^{-1}\sqrt{N} \cdot \hat{L}_r)]^2 \\
&\text{s.t.} \quad s_t, sg_t \sim \mathcal{D}_{lo} \text{ and } a_{t} \sim \pi(s_{t}, sg_t;\theta_{lo})
\end{aligned}
\end{equation}
Where $\lambda_{gp}$ is a hyper-parameter controlling the effect of the gradient penalty term. Because the gradient penalty enforces the Lipschitz constraint on the critic, limiting its update, we had to increase the number of critic training iterations to 5, a recommended value in WGAN-GP \cite{gulrajani2017improved}, per actor iteration. Considering the computational efficiency, we apply the gradient penalty every 5 training steps.

\newpage
\subsection{One-Step Rollout-based Planning}

\begin{wrapfigure}{r}{0.5\textwidth}
\centering
\includegraphics[width=0.5\textwidth]{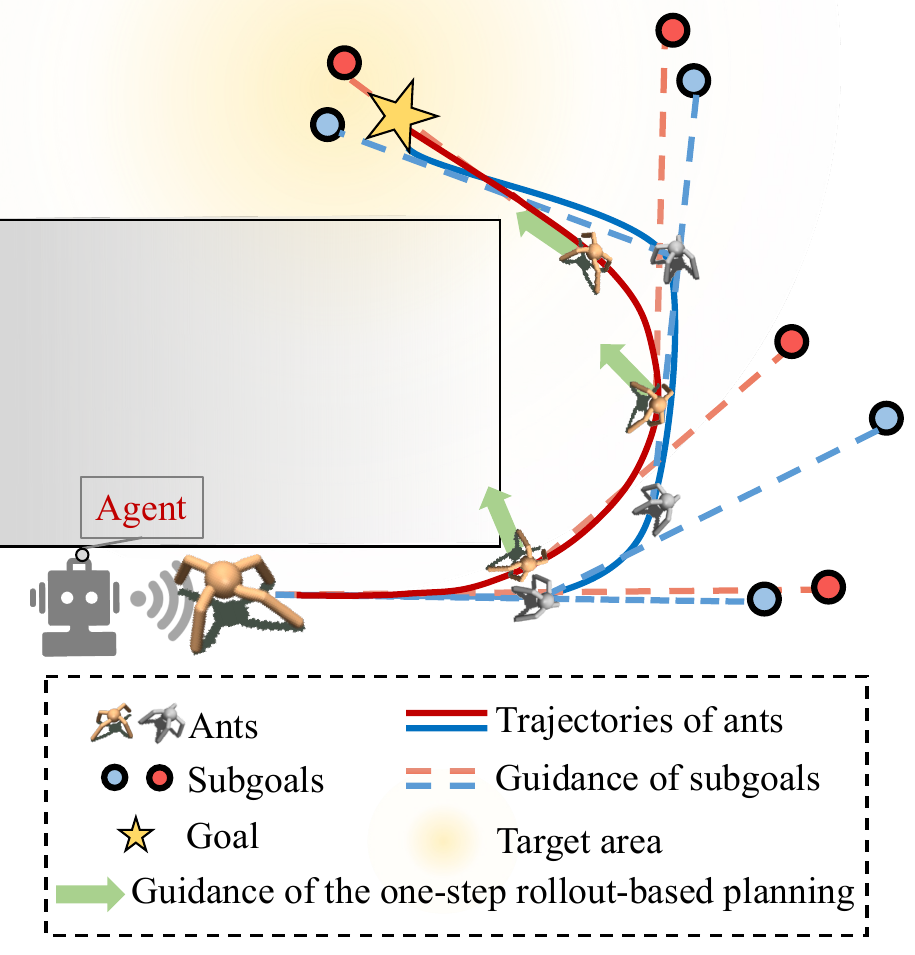}
\caption{\chadded{One-step rollout-based planning endeavors to utilize global information to direct the behavioral policy. Our method steers lower-level policy towards valuable highlands with respect to the final goal (see red trajectory), surpassing the performance of general HRL (see blue trajectory).}}
\label{osrp_workflow}
\end{wrapfigure}
In a flat model-based RL framework, model-based value expansion-style methods \cite{feinberg2018model, buckman2018sample} use dynamics models to simulate short rollouts and evaluate future transitions using the Q-function. Here\chadded{, as shown in Fig.~\ref{osrp_workflow}}, we steer the behavioral policy towards globally valuable states, i.e., having a higher higher-level Q-value. Specifically, we perform a one-step rollout and evaluate the next transition using the higher-level critics. The objective is to minimize the following loss:
\begin{equation}
\begin{aligned}
\mathcal{L}_{osrp} =& -\lambda_{osrp}\cdot\\ 
&\mathbb{E}_{s_t, g, sg_t}\left[Q(\Gamma_{\xi}(s_t, a_t),g,sg_{t+1};\phi_{hi}) \right]\\
&\text{s.t.} \qquad s_t \in \mathcal{S} \\
&\text{\quad} \qquad g, sg_t \in \mathcal{G}  \\ 
&\text{\quad}\qquad a_{t} \sim \pi(s_{t}, sg_t;\theta_{lo}))
\end{aligned}
\label{osrp_1}
\end{equation}
Where $\lambda_{osrp}$ is a hyper-parameter to weigh the planning loss. Note that the $sg_t$ is not determined by higher-level policy solely. Meanwhile, considering that the higher-level policy is also \chdeleted{constantly }changing over time, the $sg_t$ is sampled randomly from a Gaussian distribution centered at $\pi(s_{t}, g;\theta_{hi})$. In practice, a pool of $s_t$ and $g$ is sampled from the buffer $\mathcal{D}_{hi}$, and then they are repeated ten times with shuffling the $g$. Next, these samples are duplicated again to accommodate the variance of $sg_t$. On the other hand, the next step’s subgoal $sg_{t+1}$ is also produced by the fixed goal transition function or by the higher-level policy conditioning on the observation. But, from the perspective of lower-level policy, the probability of such two events is equal because of the property of Markov decision process, i.e.,
\begin{equation}
Pr\{sg_{t+1}=h\left(sg_{t},s_{t},s_{t+1}\right)|s_t, sg_t, a_t\} = Pr\{sg_{t+1}=\pi(s_t,g;\theta_{hi})|s_t, sg_t, a_t\} = 0.5
\end{equation}
Hence, the Equation \ref{osrp_1} is instantiated:
\begin{equation}
\begin{aligned}
&\mathcal{L}_{osrp} \\
&= -\lambda_{osrp} \cdot \mathbb{E}_{s_t, g, sg_t \atop sg_{t+1} \in \{h, \pi(\theta_{hi})\}} Q(\Gamma_{\xi}(s_t, a_t),g,sg_{t+1};\phi_{hi}) \\
&= -\frac{1}{2}\lambda_{osrp} \cdot \mathbb{E}_{s_t, g, sg_t} \left[Q(\Gamma_{\xi}(s_t, a_t),g,h;\phi_{hi}) \right.+ \underbrace{\left. Q(\Gamma_{\xi}(s_t, a_t),g,\pi(\theta_{hi});\phi_{hi}) \right] }_{\textcircled{a}} \\
 &\text{s.t.} \quad s_t, g, sg_t \sim \mathcal{D}_{hi} \text{ and } a_{t} \sim \pi(s_{t}, sg_t;\theta_{lo})
 \end{aligned}
\end{equation}

Obviously, the second term $\textcircled{a}$ is too dependent on current higher-level policy. The TD3 \cite{fujimoto2018addressing} seeks to smoothen the value estimate by bootstrapping off of nearby state-action pairs. Similarly, we add clipped noise to keep the value estimate robust. This makes our modified term $\textcircled{a}$:
\begin{equation}
\begin{aligned}
\textcircled{a} := Q(&\Gamma_{\xi}(s_t, a_t),g,\pi(\theta_{hi})+\varepsilon;\phi_{hi}) \\
{\rm with} \quad \varepsilon & \sim {\rm clip}(\mathcal{N}(0,\sigma), -a_c, a_c)
 \end{aligned}
\end{equation}
Where the hyper-parameters $\sigma$ and $a_c$ are common in the TD3 algorithm (see Equation~\ref{critic_loss_q}).
In the end, $\mathcal{L}_{osrp}$ is incorporated into lower-level actor loss (see Equation~\ref{actor_loss}) to guide the lower-level policy towards valuable highlands with respect to the overall task. Here, in the same way, we employ the one-step rollout-based planning every 10 training steps.

\section{Experiments}
\label{experiments}
We evaluated the proposed GCMR on challenging continuous control tasks, as shown in Fig.~\ref{environments}. Specifically, the following simulated robotics environments are considered:
\begin{itemize}[leftmargin=2em]
	\item Point Maze \cite{kim2021landmark}: In this environment, a simulated ball starts at the bottom left corner and navigates to the top left corner in a '$\sqsupset$'-shaped corridor.
	\item Ant Maze (W-shape) \cite{kim2021landmark}: In a '$\exists$'-shaped corridor, a simulated ant starts from a random position and must navigate to the target location at the middle left corner.
	\item Ant Maze (U-shape) \cite{nachum2018data, kim2021landmark}, Stochastic Ant Maze (U-shape) \cite{zhang2020generating, zhang2022adjacency}, and Large Ant Maze (U-shape): A simulated ant starts at the bottom left corner in a '$\sqsupset$'-shaped corridor and seeks to reach the top left corner. As for the randomized variation, \textit{Stochastic} Ant Maze (U-shape) introduces environmental stochasticity by replacing the agent's action at each step with a random action (with a probability of 0.25).
	\item Ant Maze-Bottleneck \cite{leed2022hrl}: The environment is almost the same as the Ant Maze (U-shape). Yet, in the middle of the maze, there is a very narrow bottleneck so that the ant can barely pass through it.
	\item Pusher \cite{kim2021landmark}: A 7-DOF robotic arm is manipulated into pushing a (puck-shaped) object on a plane to a target position.
	\item Reacher \cite{kim2021landmark}: \chreplaced{A 7-DOF robotic arm is manipulated to make the end-effector reach a spherical object that is randomly placed in mid-air.}{A 7-DOF robotic arm is manipulated to make the end-effector reach a target position.}
\end{itemize}
    
\begin{figure}[htbp]
\captionsetup[subfloat]{format=hang, justification=centering}
\centering
\subfloat[Point Maze]{\includegraphics[width=0.23\textwidth]{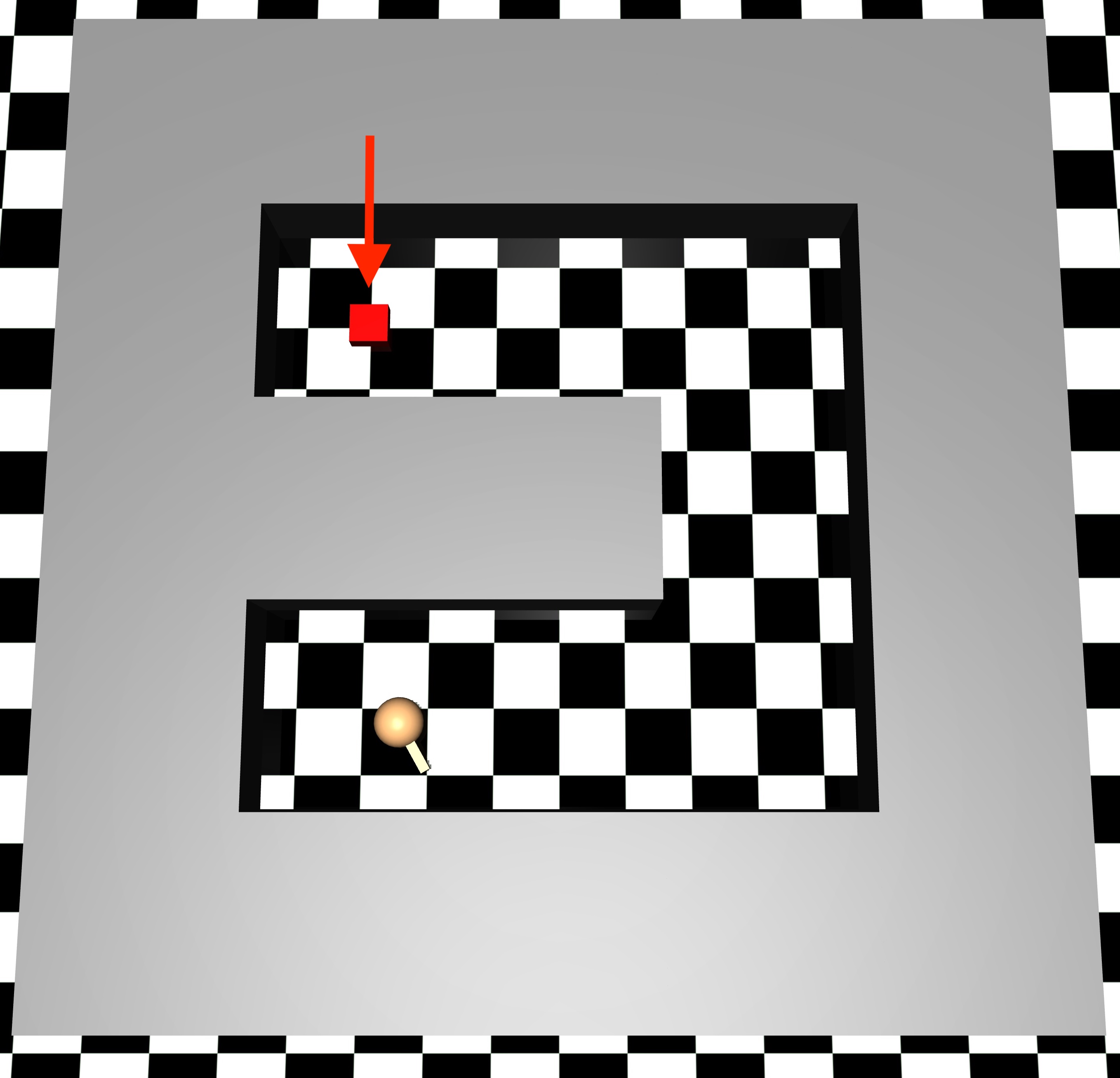}} 
\subfloat[Ant Maze \protect\\(W-shape)]{\includegraphics[width=0.24\textwidth]{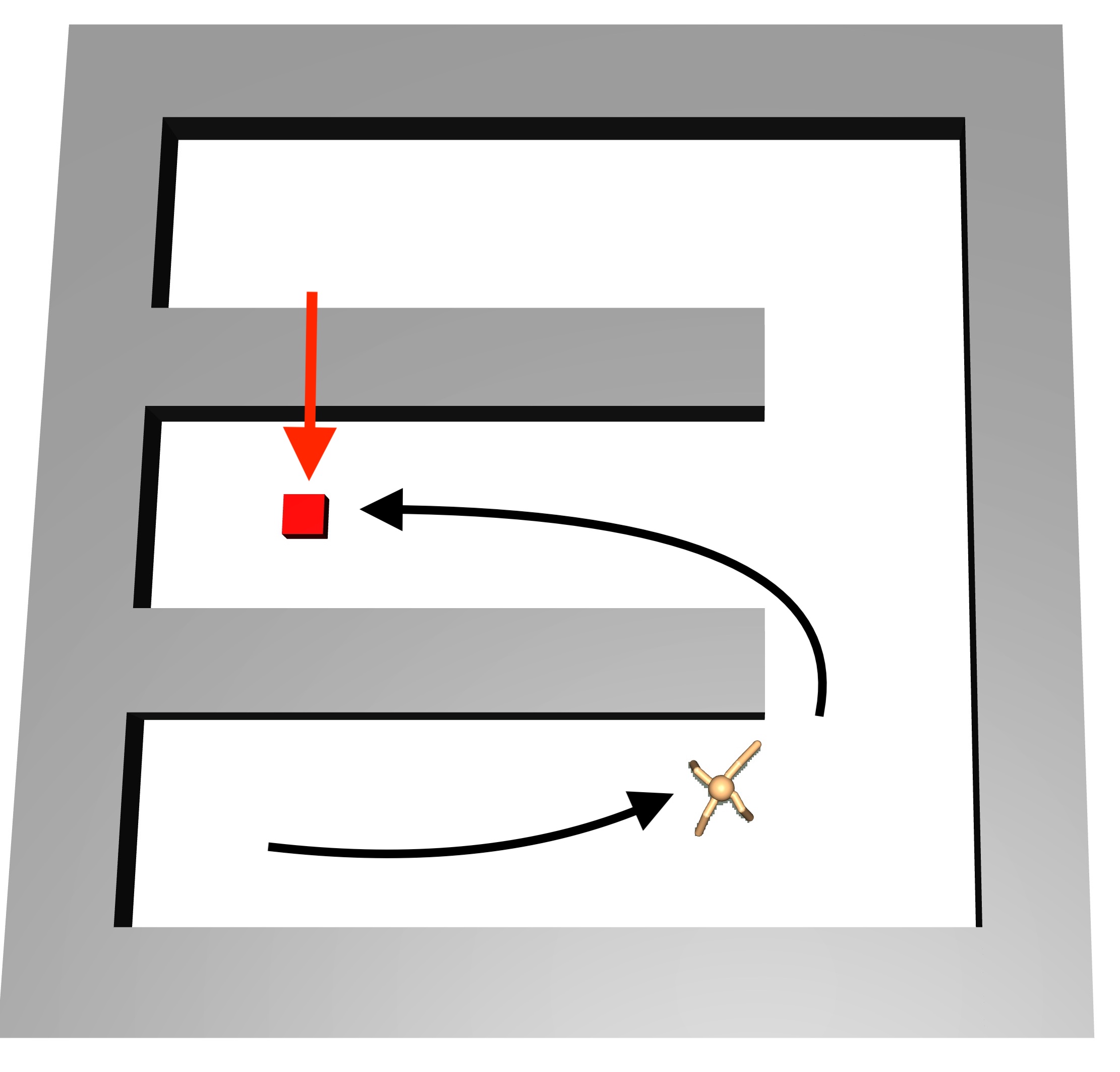}} 
\subfloat[Ant Maze \protect\\(U-shape)]{\includegraphics[width=0.26\textwidth]{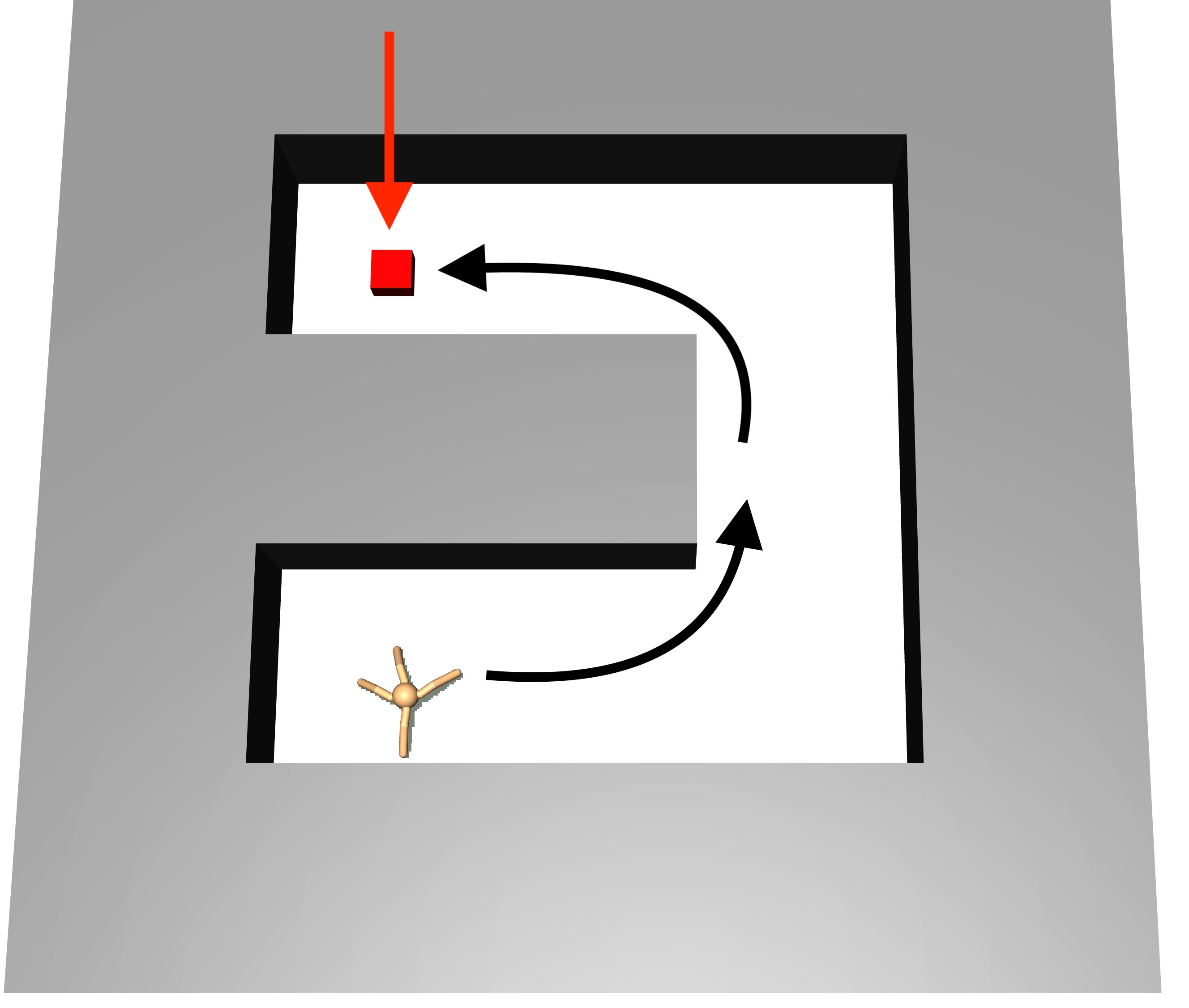}}
\subfloat[Ant Maze-\protect\\Bottleneck]{\includegraphics[width=0.24\textwidth]{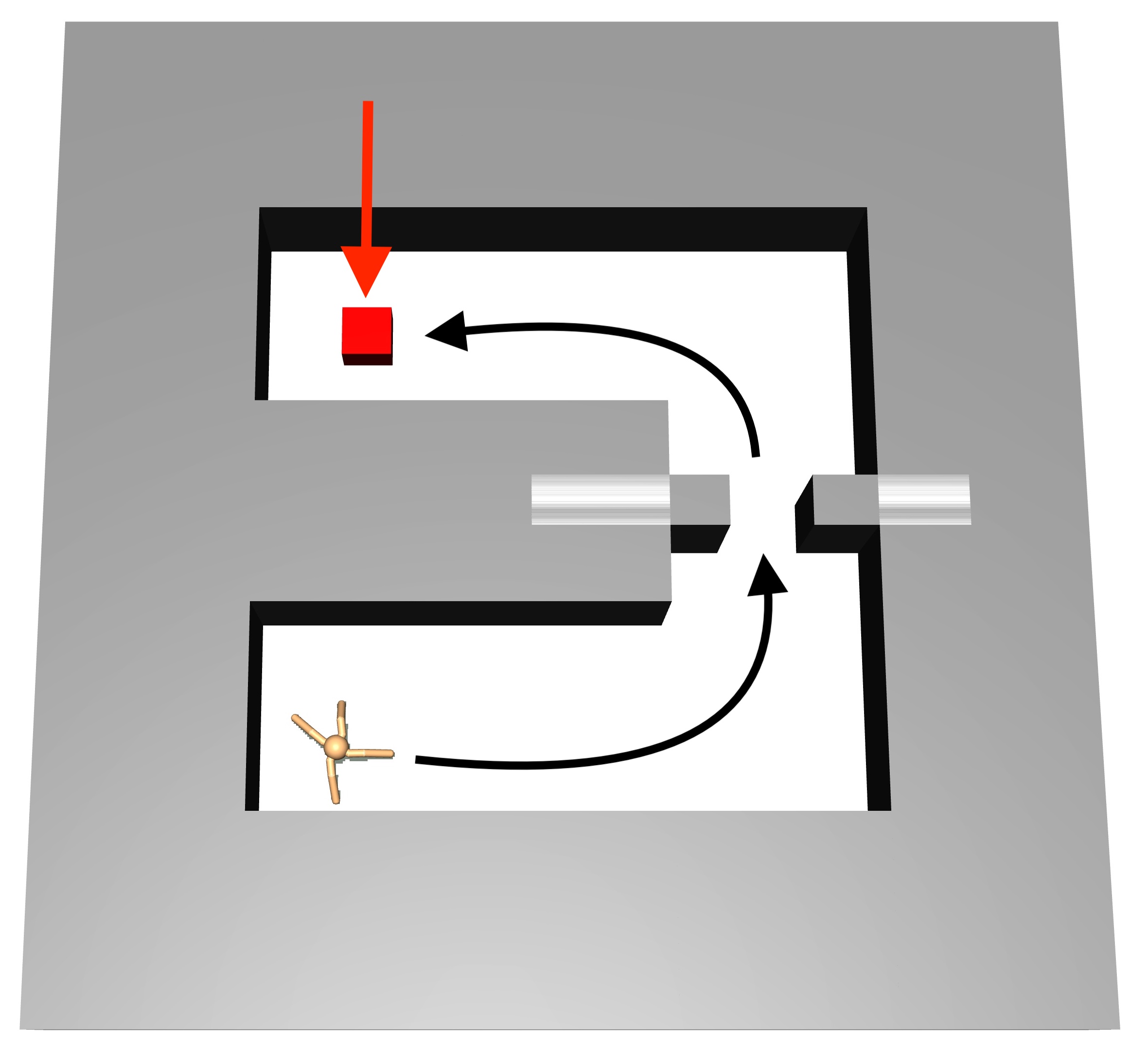}}\\
\subfloat[Pusher]{\includegraphics[width=0.24\textwidth]{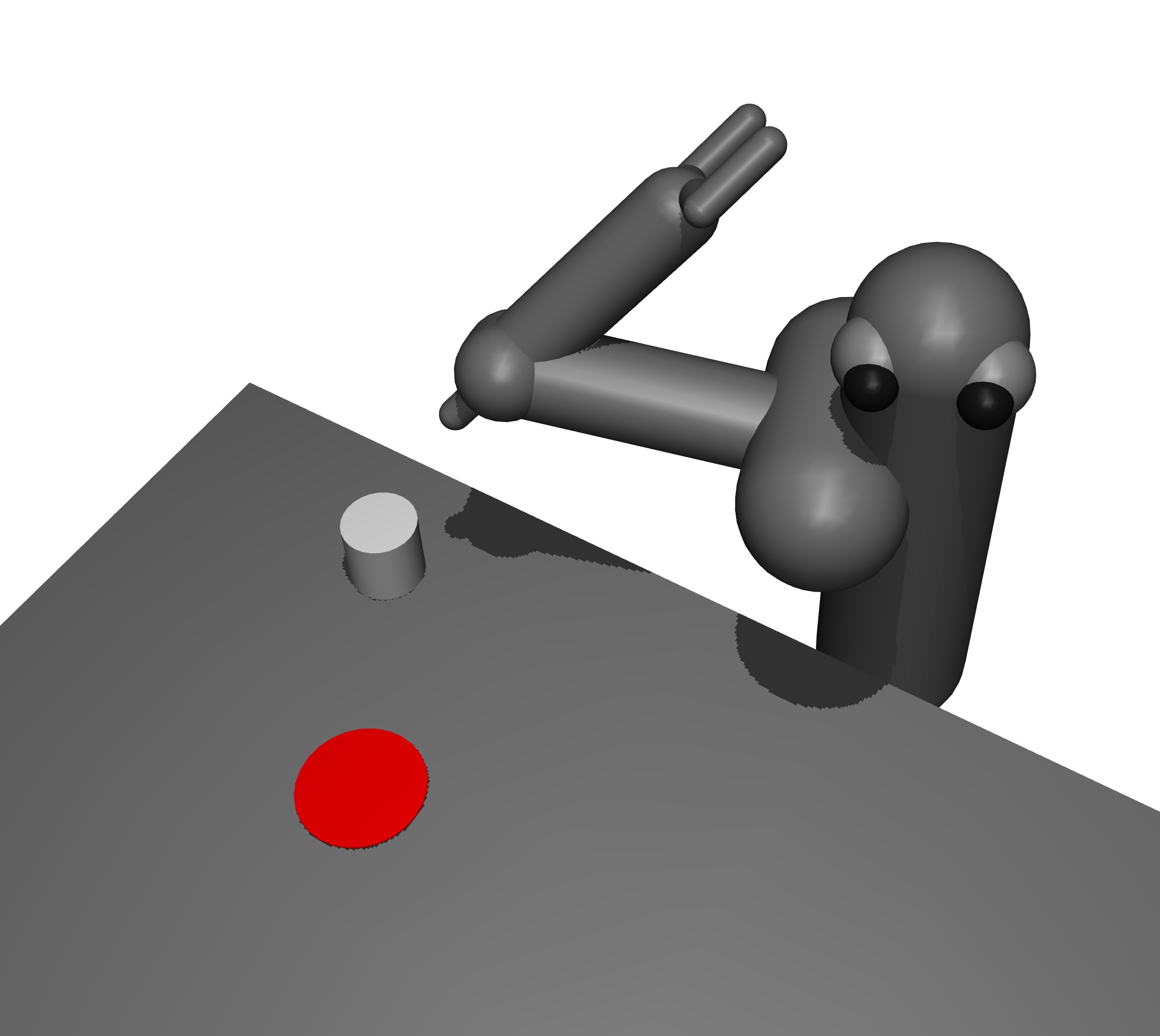}}
\subfloat[Reacher]{\includegraphics[width=0.24\textwidth]{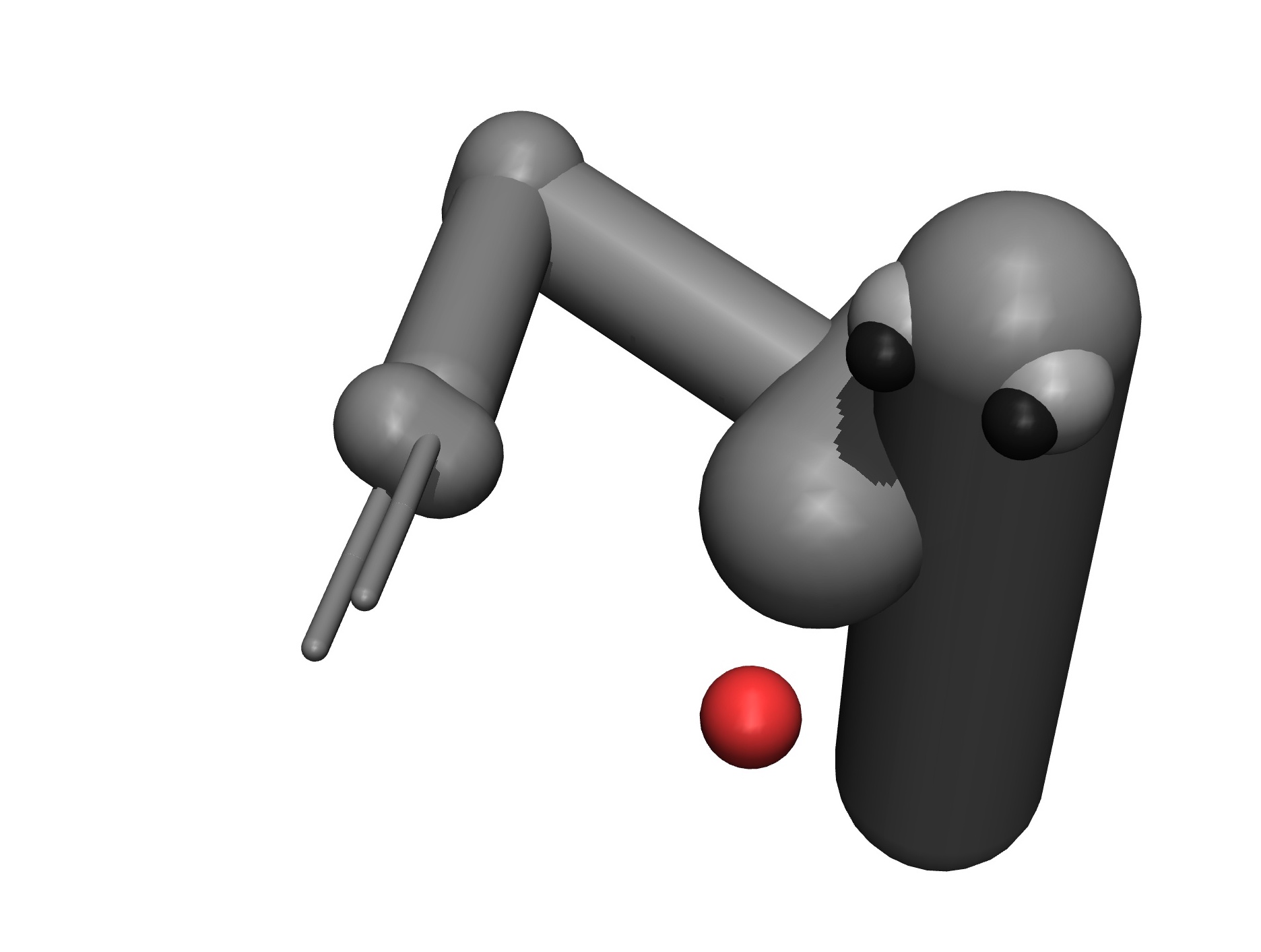}}
\subfloat[\chadded{FetchPush}]{\includegraphics[width=0.23\textwidth]{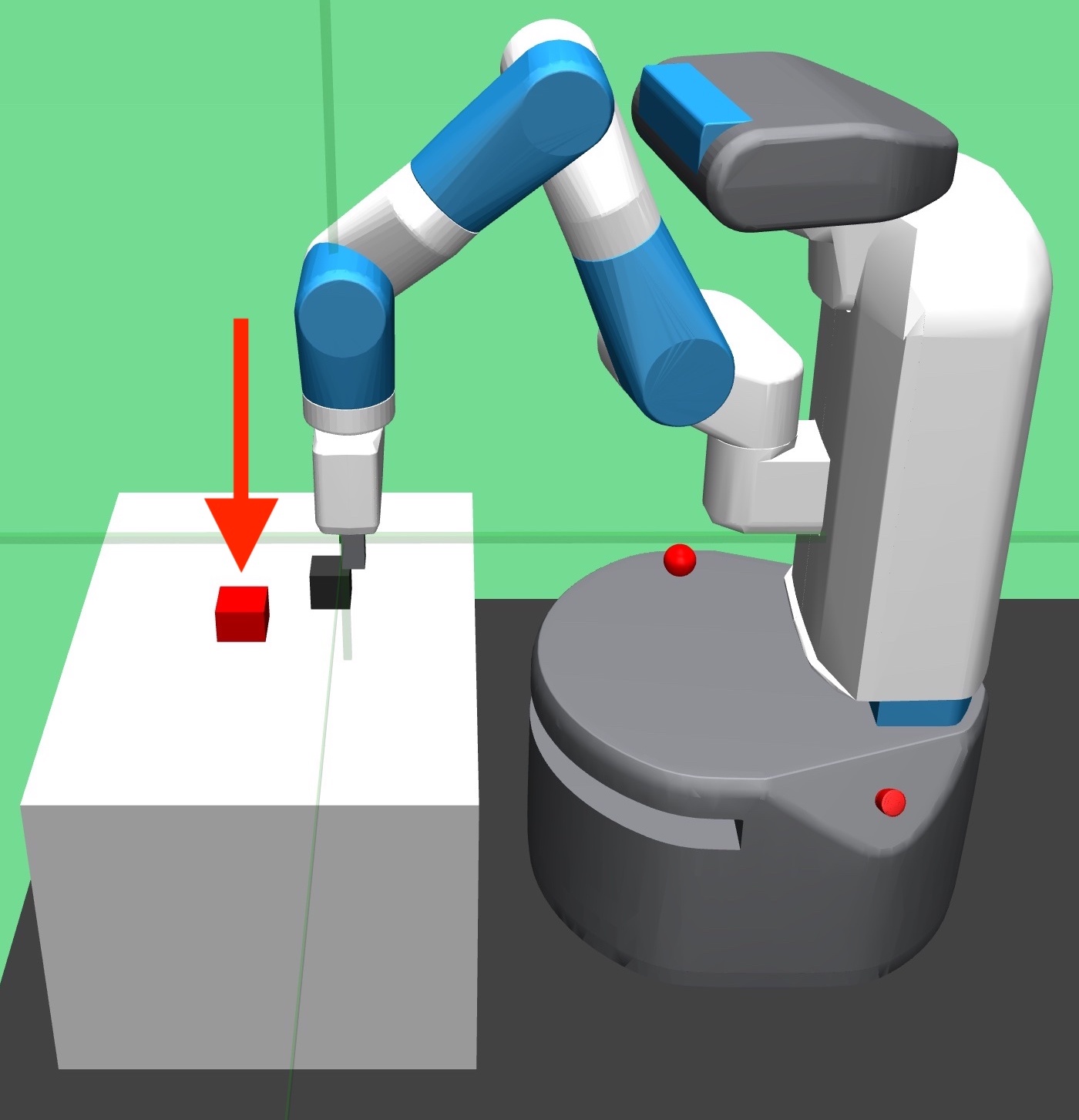}}\hspace{0.015\textwidth}
\subfloat[\chadded{FetchPickAndPlace}]{\includegraphics[width=0.23\textwidth]{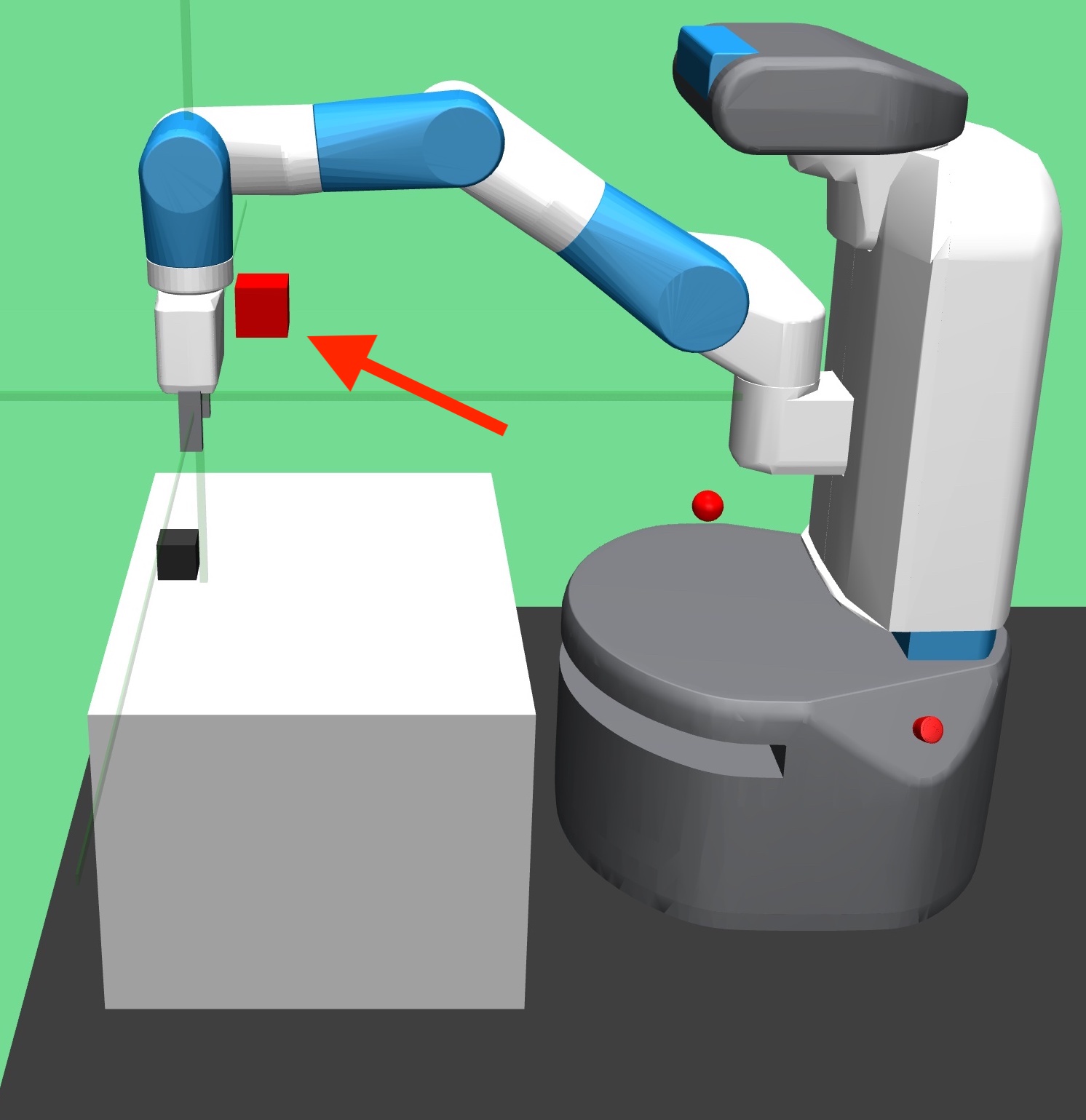}}
\caption{Environments used in our experiments. In the maze-related tasks, the goal in each task is marked with a red arrow, and the black line represents a possible trajectory from the current state to the goal.}
\label{environments}
\end{figure}
These general '$\sqsupset$'-shaped mazes have the same size of $12 \times 12$ while $20 \times 20$ is for the '$\exists$'-shaped maze. Besides, the size of the \textit{Large} Ant Maze (U-shape) is twice as large as that of the general Ant Maze (U-shape), i.e., 24 × 24. 

\begin{itemize}[leftmargin=2em]
	\item \chadded{FetchPush \cite{yang2021hierarchical, zhang2022adjacency, zeng2023ahegc}: In this environment, a block is randomly placed on the table surface and is expected to be moved to the specified location using a 7-DoF fetch manipulator arm. Here, the gripper of the manipulator arm is locked in a closed configuration.}
	\item \chadded{FetchPickAndPlace \cite{yang2021hierarchical, zhang2022adjacency, zeng2023ahegc}: The environment is similar to FetchPush. Yet, the block is randomly positioned either in mid-air or on the table surface, and the gripper of the manipulator arm can be opened to pick up the block.}
\end{itemize}

Further environment details are available in the "Supplementary Materials".

\subsection{Hyper-parameters in \textit{ACLG}}
\label{sec:hyperparams_aclg}
First, we conduct experiments on Ant Maze (U-shape) to explore the effect of hyper-parameters in ACLG: (1) the number of landmarks and (2) the balancing coefficient $\lambda^{\rm ACLG}_{\rm landmark}$.

\begin{figure}[htbp]
\captionsetup[subfloat]{format=hang, justification=centering}
\centering
\subfloat[Number of landmarks]{\includegraphics[width=0.4\textwidth]{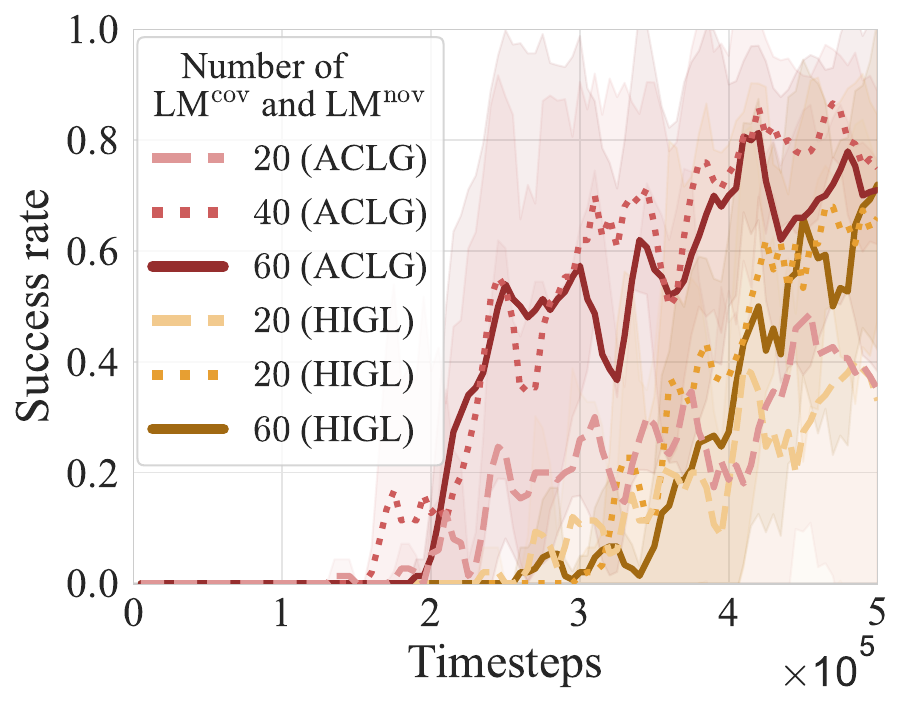}
\label{landmark_num}}\hspace{-0.3cm}
\subfloat[Varying $\lambda^{\rm ACLG}_{\rm landmark}$]{\includegraphics[width=0.40\textwidth]{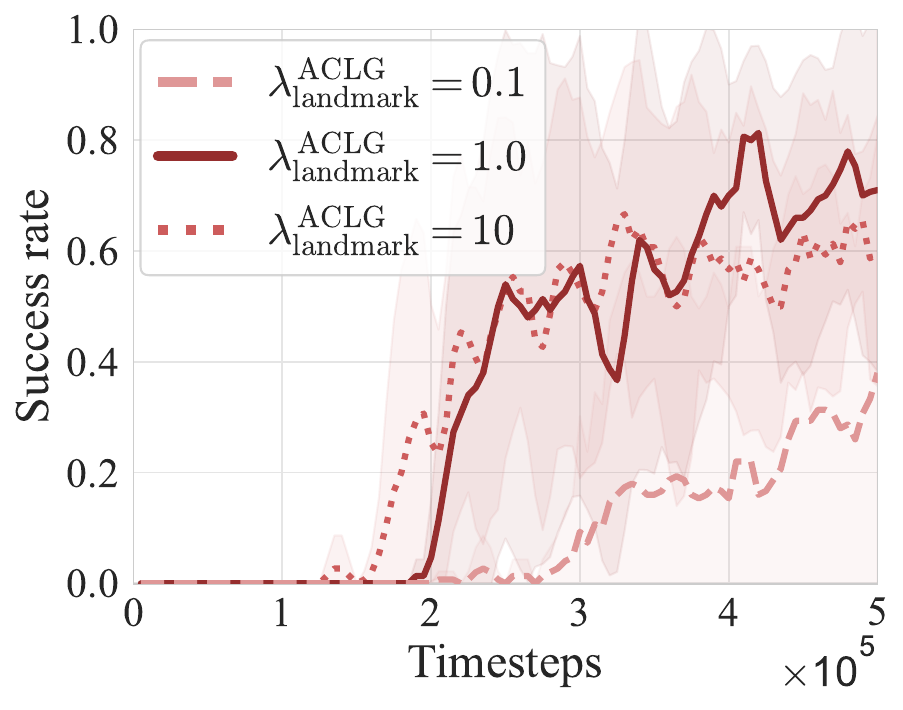}
\label{landmark_param}}
\caption{Ablation studies on landmark-related components. We measure the performance of ACLG by (a) varying number of landmarks and (b) varying balancing coefficient $\lambda^{\rm ACLG}_{\rm landmark}$ in Ant Maze (U-shape).}
\end{figure}

\paragraph{Landmark Number Selection}
Since the number of landmarks plays an important role in the graph-related method, we explored the effects of different numbers of landmarks on performance. Here, we sample the same number of landmarks for each criterion, i.e., the same number for novelty-based and novelty-based landmarks ${LM}^{\rm \chreplaced{cov}{nov}}={LM}^{\rm nov}$. As shown in Fig.~\subref*{landmark_num}, overall, the ACLG significantly outperforms the HIGL since the disentanglement between the adjacency constraint and landmark-based planning further highlights the advantages of landmarks. Finally, the setting ${LM}^{\rm \chreplaced{cov}{nov}}={LM}^{\rm nov}=60$ was adopted for further analysis.

\paragraph{Balancing Coefficient $\lambda^{\rm ACLG}_{\rm landmark}$} In Fig.~\subref*{landmark_param}, we investigate the effectiveness of the balancing coefficient $\lambda^{\rm ACLG}_{\rm landmark}$, which determines the effect of the landmark-based planning term in ACLG on performance. We find that ACLG with $\lambda^{\rm ACLG}_{\rm landmark}=1.0$ outperforms others. Moreover, ACLG with $\lambda^{\rm ACLG}_{\rm landmark}\in\{1.0, 10\}$ outperforms that with $\lambda^{\rm ACLG}_{\rm landmark}=0.1$, which shows a large value of $\lambda^{\rm ACLG}_{\rm landmark}$ helps unleash the guiding role of landmarks.

\subsection{\chadded{Hyper-parameters in \textit{GCMR}}}
\label{sec:hyperparams_gcmr}
\chadded{Next, we highlight the impact of the hyper-parameters in GCMR: the shift magnitude $\delta_{sg}$ in soft-relabeling, the penalty coefficient $\lambda_{gp}$, and the coefficient $\lambda_{osrp}$ for the rollout-based planning. We deploy ACLG+GCMR in multiple environments.}

\paragraph{\chadded{Shift Magnitude $\delta_{sg}$ for Soft-Relabeling}}

\chadded{In Fig.~\ref{goal_shift}, we conduct experiments to examine the impact of $\delta_{sg}$. Fig.~\subref*{org_goals} illustrates the original subgoals sampled from the experience
replay buffer of \textit{Large} Ant Maze (U-shape) at 0.6M steps. Fig.~\subref*{no_shift}-\subref*{shift_40} depict the relabeled subgoals without or with varying shift magnitude. By comparing Fig.~\subref*{org_goals} and Fig.~\subref*{no_shift}, it is evident that relabeling introduces out-of-distribution (OOD) subgoals, which in turn lead to instability. The OOD issue can be addressed using the soft-relabeling method, which involves shifting the original subgoals towards the relabeled ones by a certain magnitude, as shown in Fig.~\subref*{shift_10}-\subref*{shift_40}. However, it should be noted that small shift magnitudes will weaken the corrective effect. As shown in Fig.~\subref*{success_shift}, we observe that too large or little value of $\delta_{sg}$ harms the performance. Therefore, the shift magnitude $\delta_{sg}$ needs to be carefully tuned to strike a balance between suppressing the OOD issue and maintaining corrective effectiveness. In our experiments, $\delta_{sg}$ is set to 20 for the small maze and 30 for the large maze.
}

\begin{figure}[htbp]
\captionsetup[subfloat]{format=hang, justification=centering}
\centering
\subfloat[\chadded{Original subgoals}]{\includegraphics[width=0.24\textwidth]{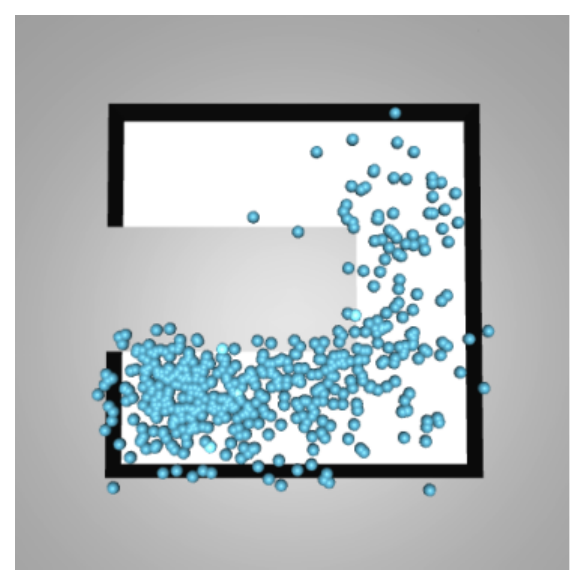}
\label{org_goals}}
\subfloat[\chadded{Relabeling with \textit{no shift}}]{\includegraphics[width=0.24\textwidth]{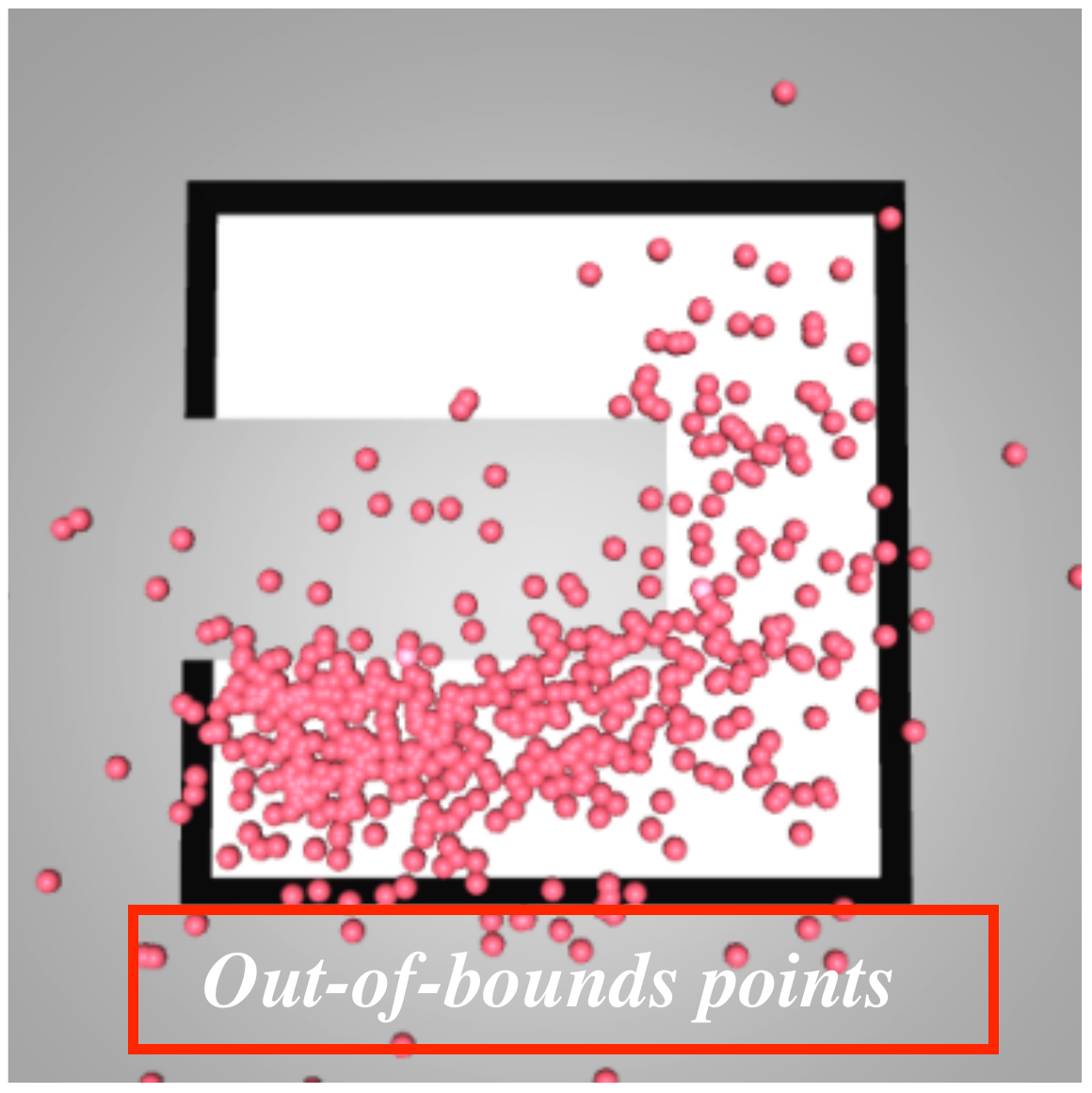}
\label{no_shift}}
\subfloat[\chadded{Relabeling with $\delta_{sg}=10$}]{\includegraphics[width=0.24\textwidth]{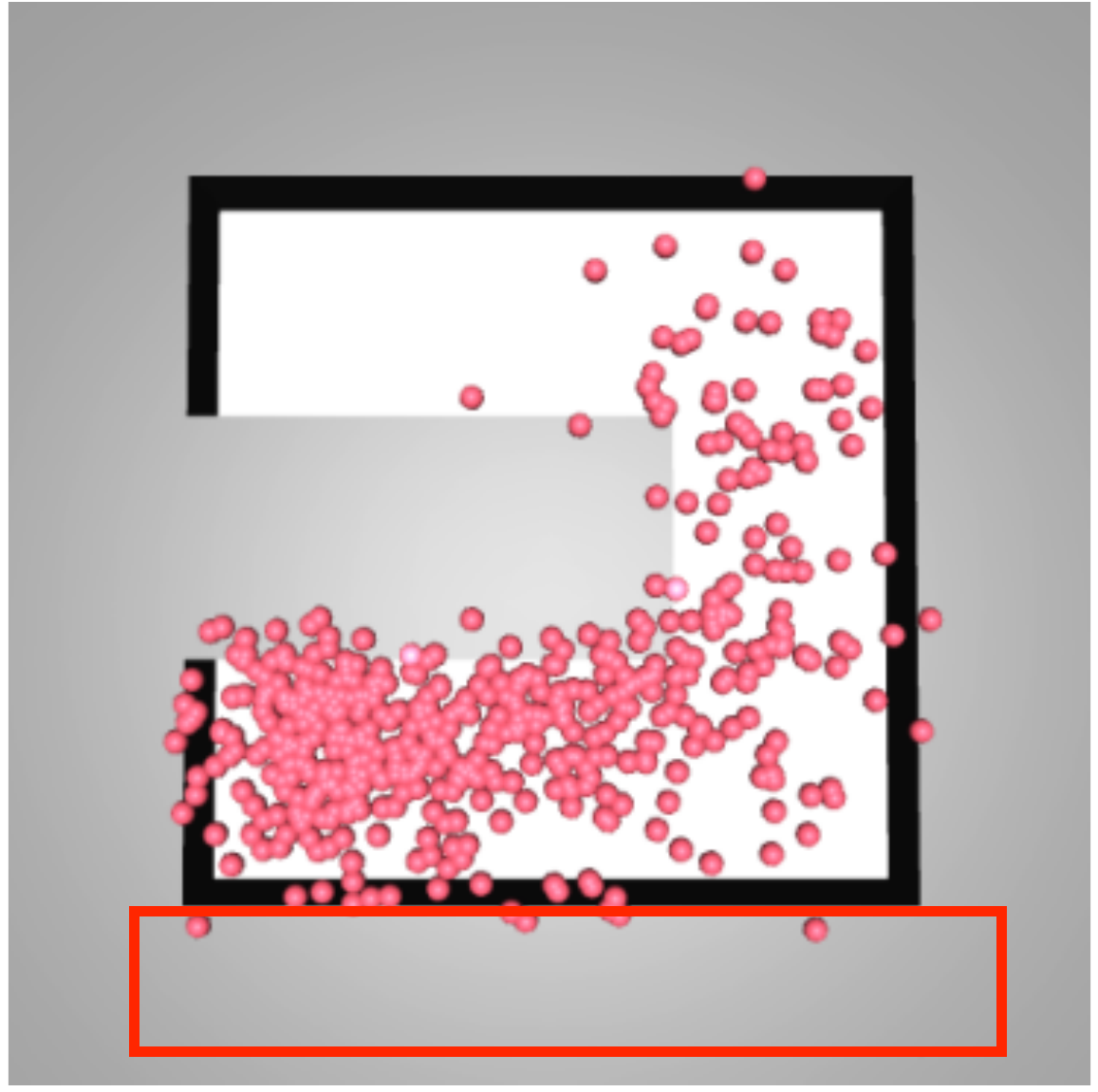}
\label{shift_10}}
\subfloat[\chadded{Relabeling with $\delta_{sg}=20$}]{\includegraphics[width=0.24\textwidth]{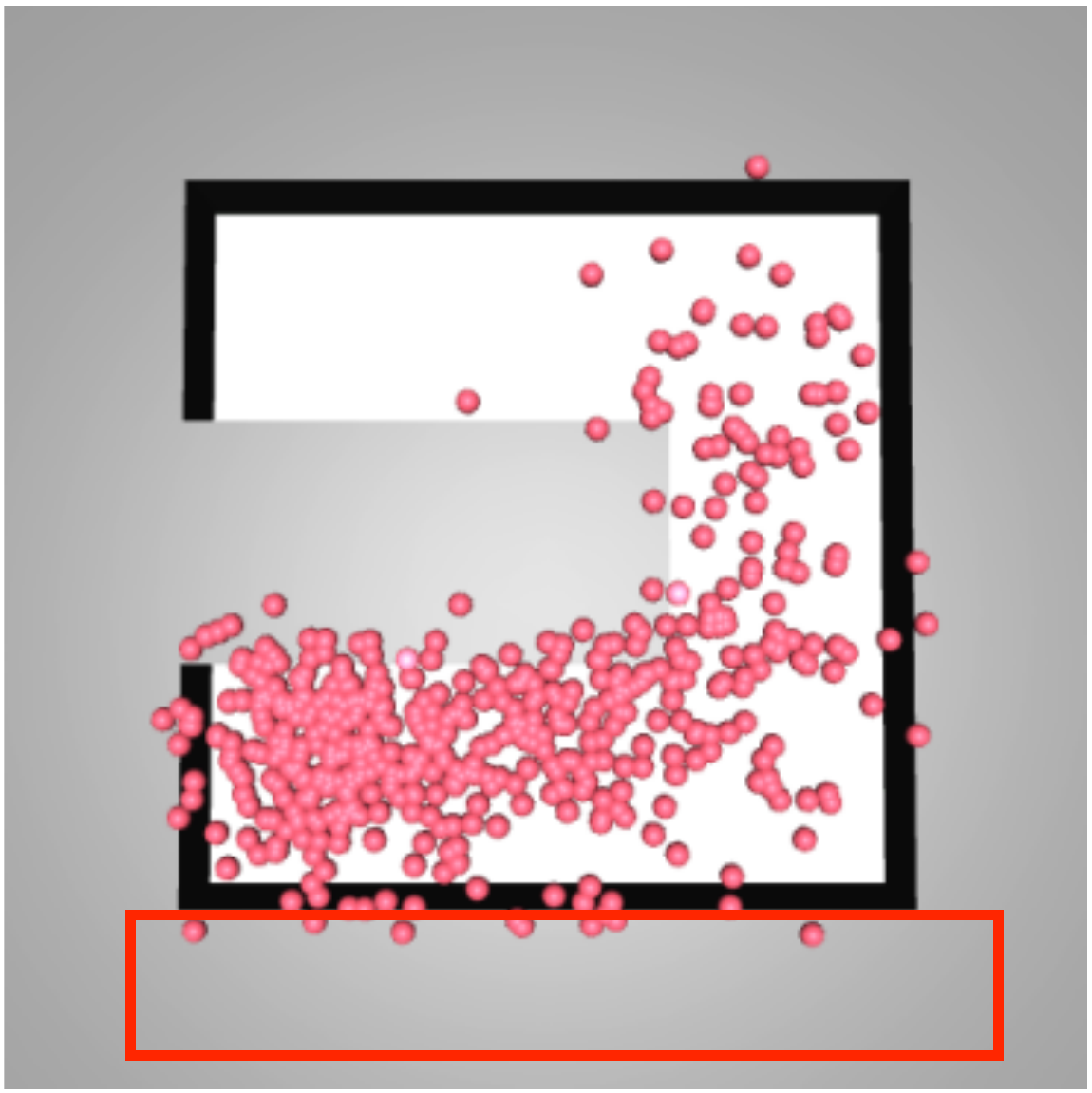}
\label{shift_20}}\\
\subfloat[\chadded{Relabeling with $\delta_{sg}=30$}]{\includegraphics[width=0.24\textwidth]{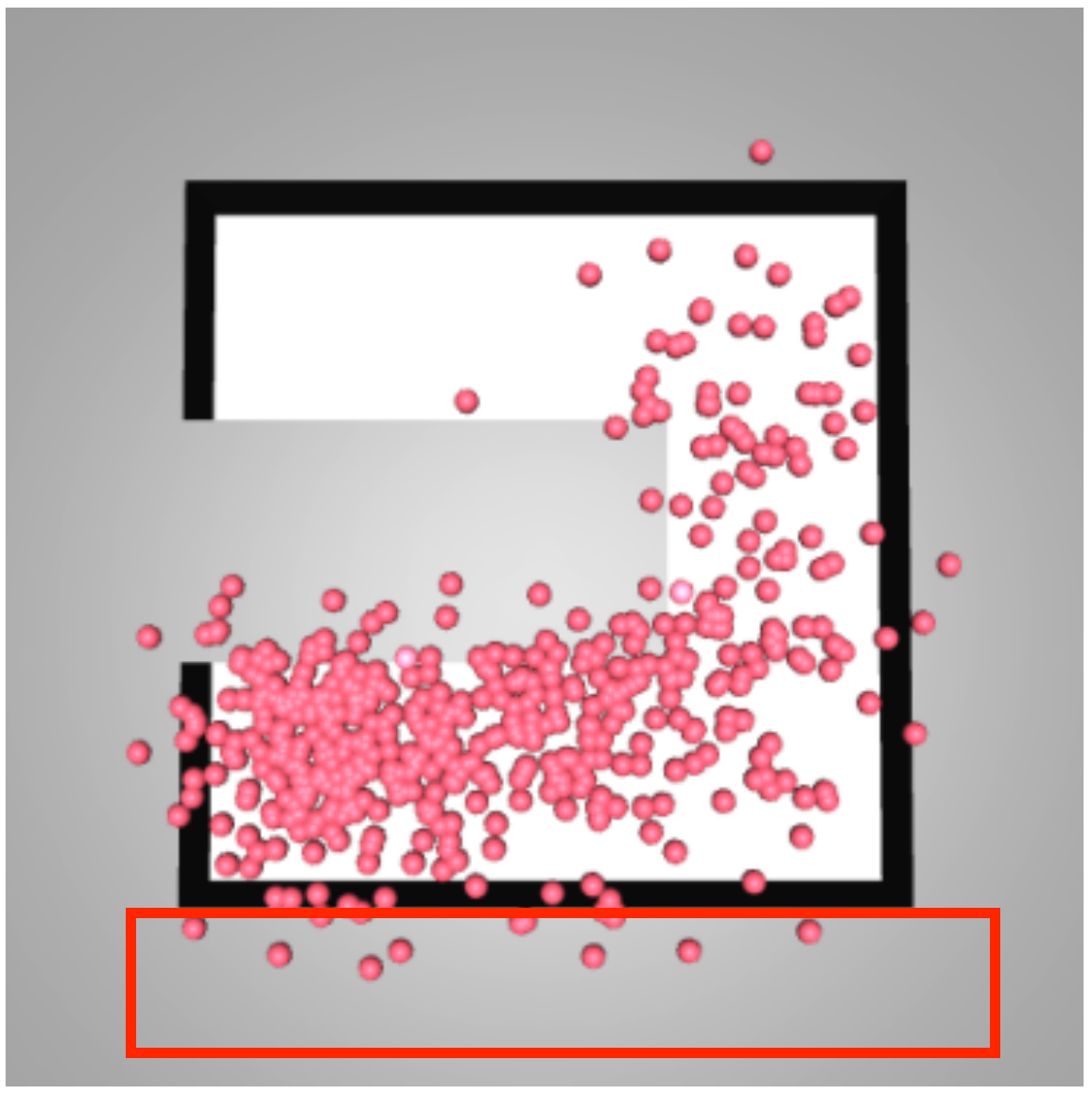}
\label{shift_30}}
\subfloat[\chadded{Relabeling with $\delta_{sg}=40$}]{\includegraphics[width=0.24\textwidth]{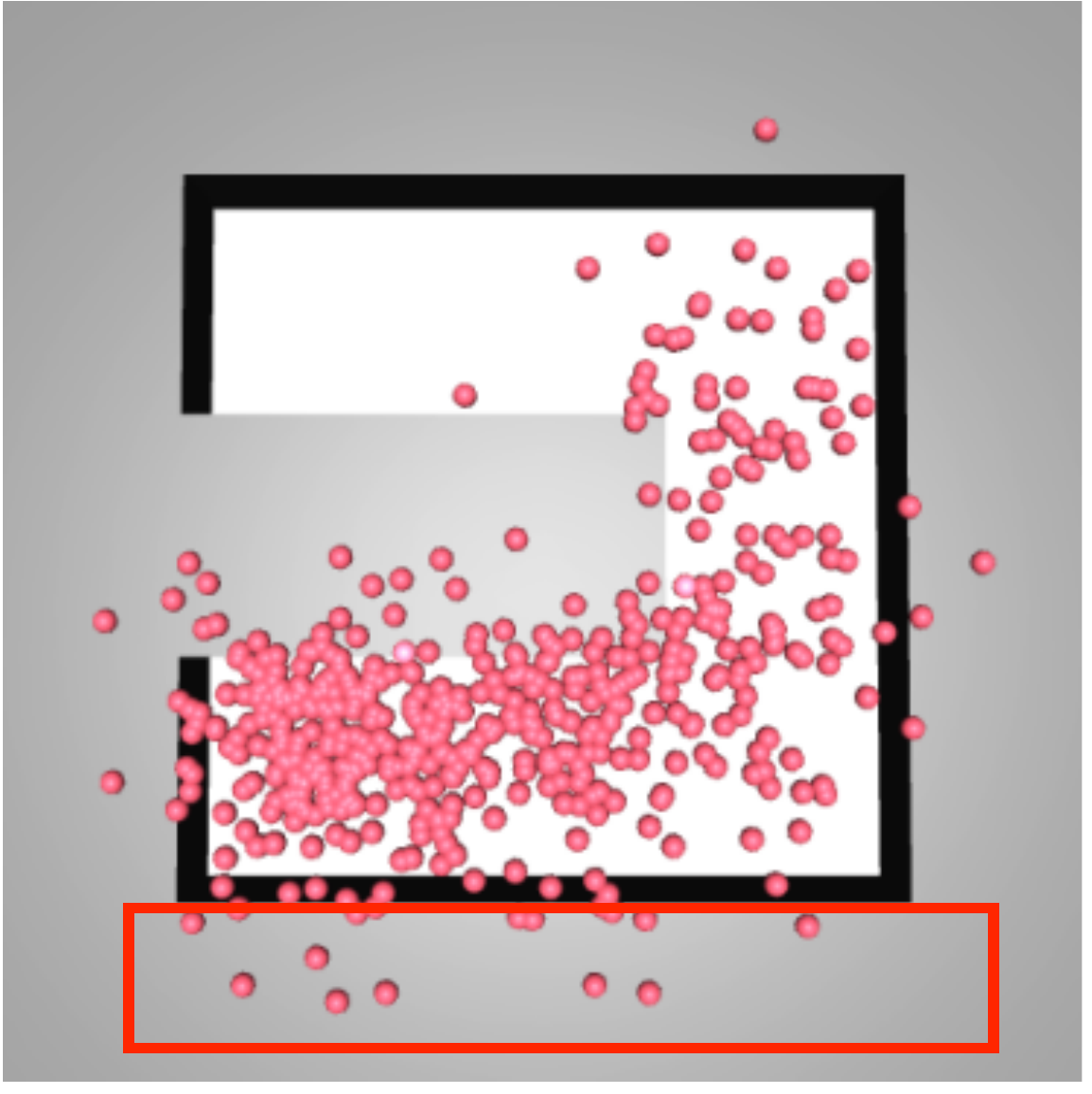}
\label{shift_40}}
\subfloat[\chadded{\textit{Large} Ant Maze (U-shape)}]{\includegraphics[width=0.5\textwidth]{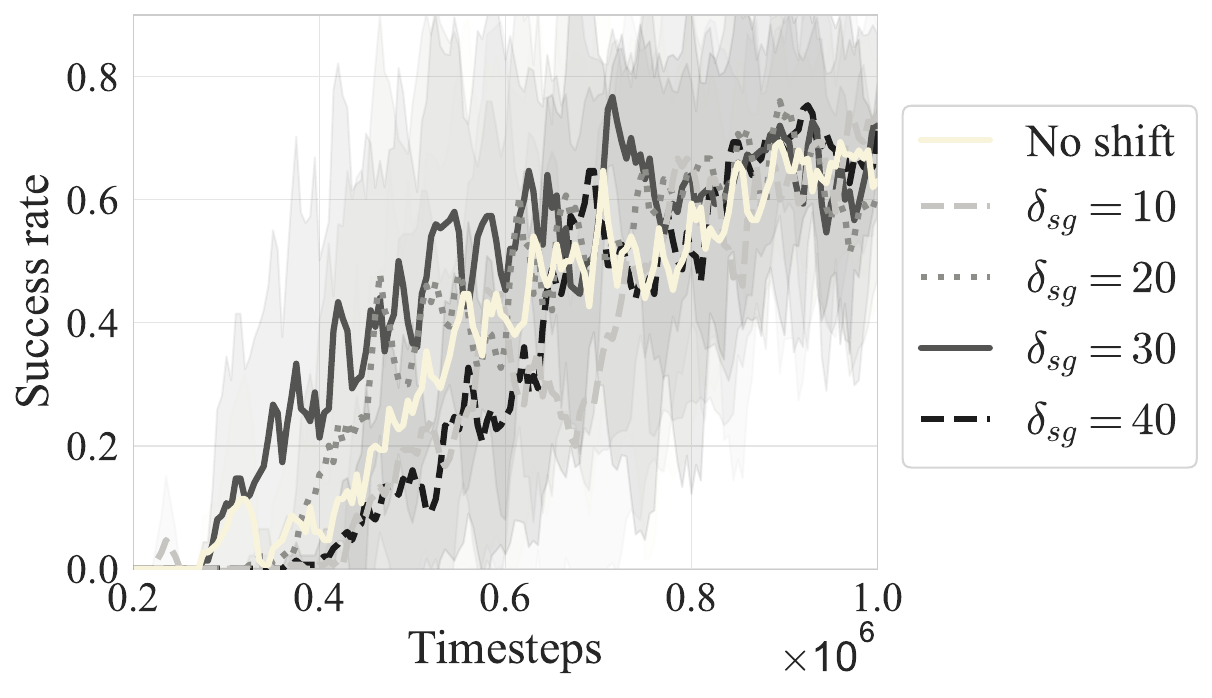}
\label{success_shift}}
\caption{\chadded{Impact of shift magnitude $\delta_{sg}$ on the performance of ACLG+GCMR in \textit{Large} Ant Maze (U-shape). \subref{org_goals} illustrates the original subgoals, while \subref{no_shift}-\subref{shift_40} depict the relabeled subgoals without or with varying shift magnitude. \subref{success_shift} plots the learning curves of ACLG+GCMR on the \textit{Large} Ant Maze (U-shape) with varying shift magnitude $\delta_{sg}$. In \subref{success_shift}, the result is averaged over five random seeds.}}
\label{goal_shift}
\end{figure}

\paragraph{\chadded{Penalty Coefficient $\lambda_{gp}$}}
\chadded{In Fig.~\ref{mgp_lambda}, we explore the impact of the penalty coefficient $\lambda_{gp}$ on the performance of ACLG+GCMR.
Overall, when $\lambda_{gp}\geq1.0$, applying the gradient penalty leads to improved asymptotic performance. Furthermore, across most environments, we consistently find that the value $\lambda_{gp}=1.0$ emerges as optimal. For highly intricate environments, the gradient penalty should be strengthened, as in the case of the Ant Maze-Bottleneck (see Fig.~\subref*{mgp_lambda2}).}

\begin{figure}[htbp]
\captionsetup[subfloat]{format=hang, justification=centering}
\centering
\subfloat[\chadded{Ant Maze (U-shape)}]{\includegraphics[width=0.4\textwidth]{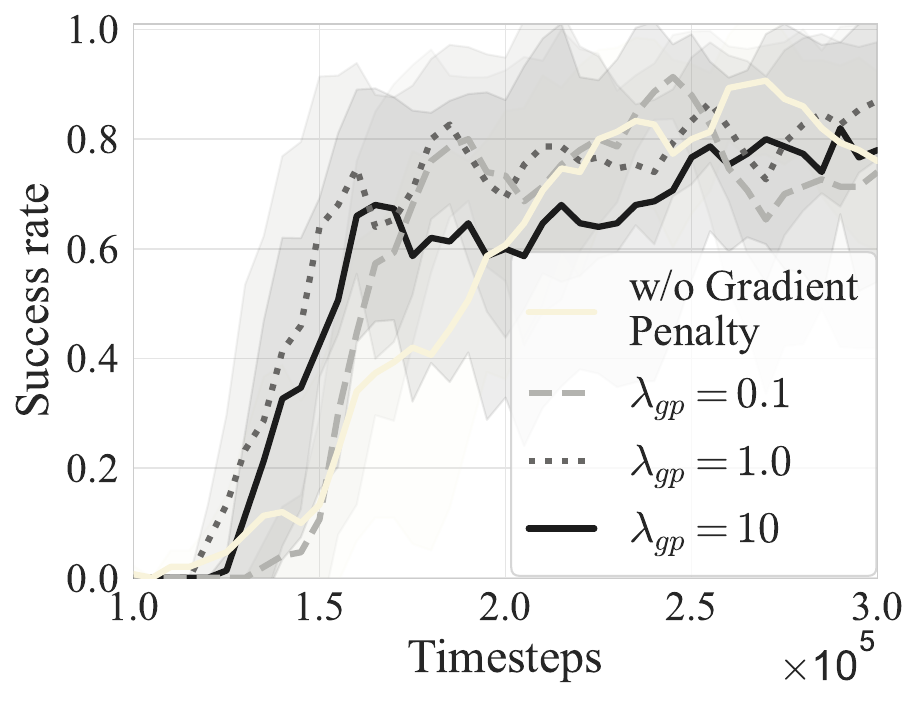}
\label{mgp_lambda1}}\hspace*{-0.35cm}
\subfloat[\chadded{Ant Maze-Bottleneck}]{\includegraphics[width=0.4\textwidth]{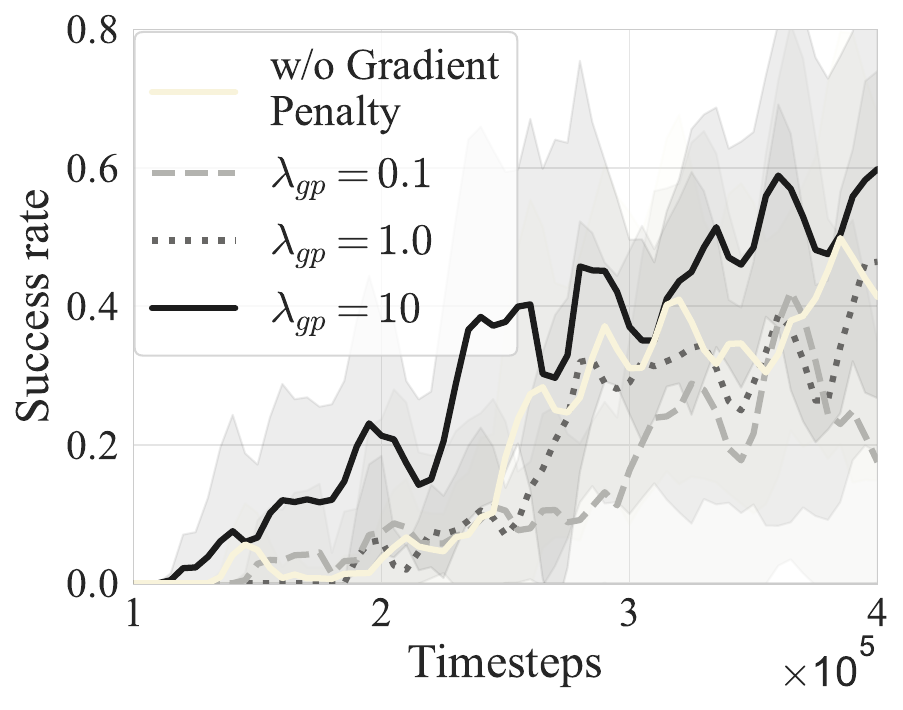}
\label{mgp_lambda2}}
\caption{\chadded{Learning curves of ACLG+GCMR with varying penalty coefficient $\lambda_{gp}$ on (a) Ant Maze (U-shape) and (b) Ant Maze-Bottleneck. Here, the success rate is averaged over five random seeds.}}
\label{mgp_lambda}
\end{figure}

\begin{figure}[htbp]
\captionsetup[subfloat]{format=hang, justification=centering}
\centering
\subfloat[\chadded{Ant Maze (U-shape)}]{\includegraphics[width=0.4\textwidth]{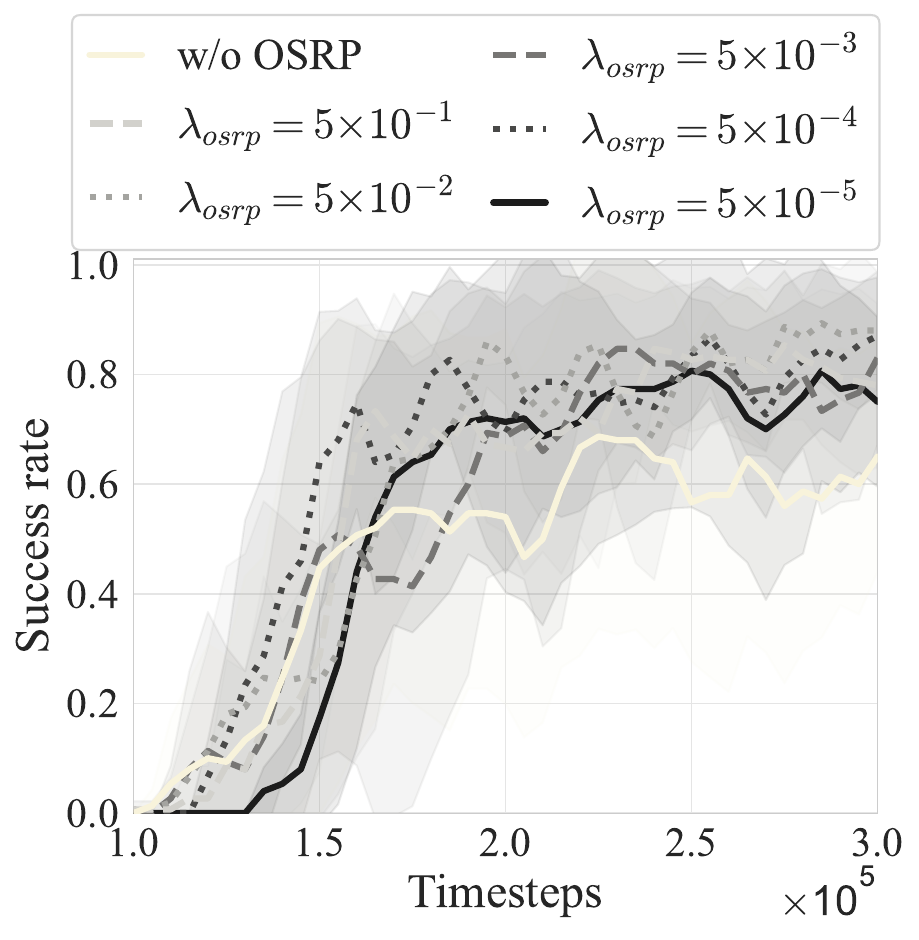}
\label{osrp_lambda1}}\hspace*{-0.33cm}
\subfloat[\chadded{Point Maze}]{\includegraphics[width=0.4\textwidth]{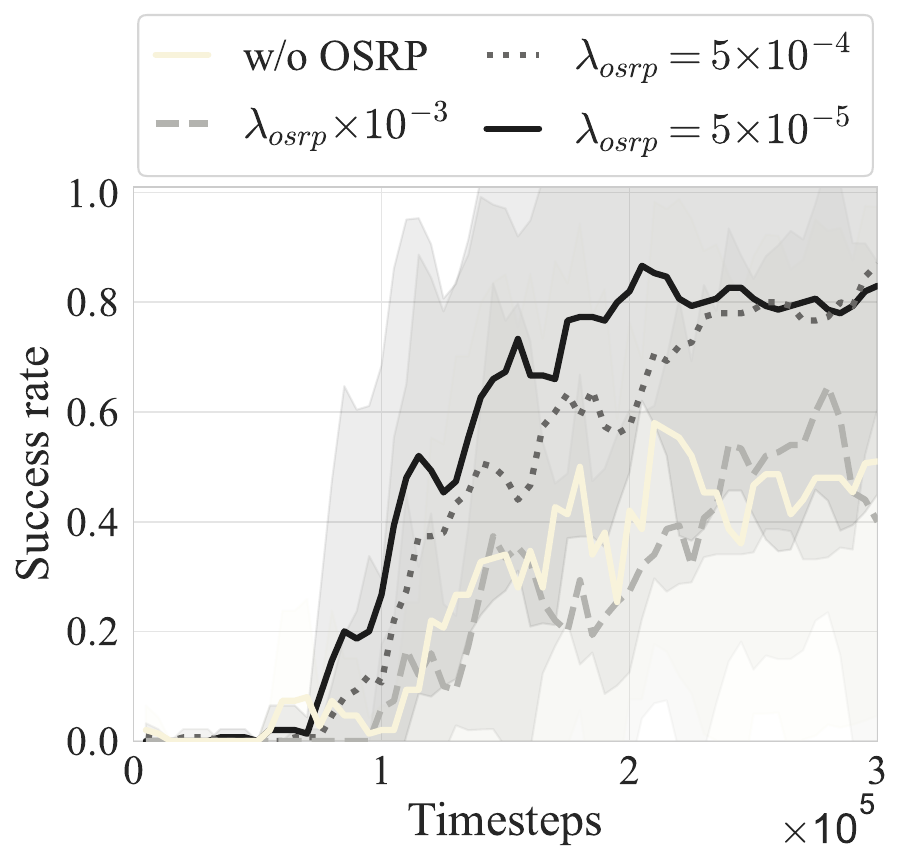}
\label{osrp_lambda2}}
\caption{\chadded{Learning curves of ACLG+GCMR with varying balancing coefficient $\lambda_{osrp}$ on (a) Ant Maze (U-shape) and (b) Point Maze. Here, the success rate is averaged over five random seeds.}}
\label{osrp_lambda}
\end{figure}

\paragraph{\chadded{Balancing Coefficient $\lambda_{osrp}$ for One-Step Rollout-based Planning}}
\chadded{In Fig.~\ref{osrp_lambda}, we investigate the impact of $\lambda_{osrp}$, which determines the effect of the proposed one-step rollout-based planning term, $\mathcal{L}_{osrp}$ (see Equation~\ref{osrp_1}), on the proposed framework. The results show that utilizing one-step rollout-based planning achieved a more stable asymptotic performance in the Ant Maze (U-shape) and Point Maze. In fact, a slightly larger $\lambda_{osrp}$ could powerfully and aggressively force the lower-level policy to follow the one-step rollout-based planning, leading to accelerated learning. However, if $\lambda_{osrp}$ is too large, it may restrict the exploration capability of the behavioral policy itself. The recommended range for the parameter $\lambda_{osrp}$ is between $5 \times10^{-5}$ and $5 \times 10^{-4}$.}

\subsection{Comparative Experiments}
To validate the effectiveness of the GCMR, we plugged it into the ACLG, the disentangled variant of HIGL, and then compared the performance of the integrated framework with that of ACLG, HIGL \cite{kim2021landmark}, HRAC \cite{zhang2020generating}, DHRL \cite{leed2022hrl}, as well as the PIG \cite{kim2022imitating}. Note that the numbers of landmarks used in these methods are different. In most tasks, HIGL employed 40 landmarks, ACLG used \chreplaced{60}{120} landmarks, DHRL utilized 300 landmarks, and PIG employed 400 landmarks (see Table 3 in the "Supplementary Materials" for details). This is in line with prior works.
Also, consistent with prior research, our experiments were performed on the above-mentioned environments with \textit{sparse} reward settings, where the agent will obtain no reward until it reaches the target area.

\begin{wrapfigure}{r}{0.5\textwidth}
\centering
\includegraphics[width=0.5\textwidth]{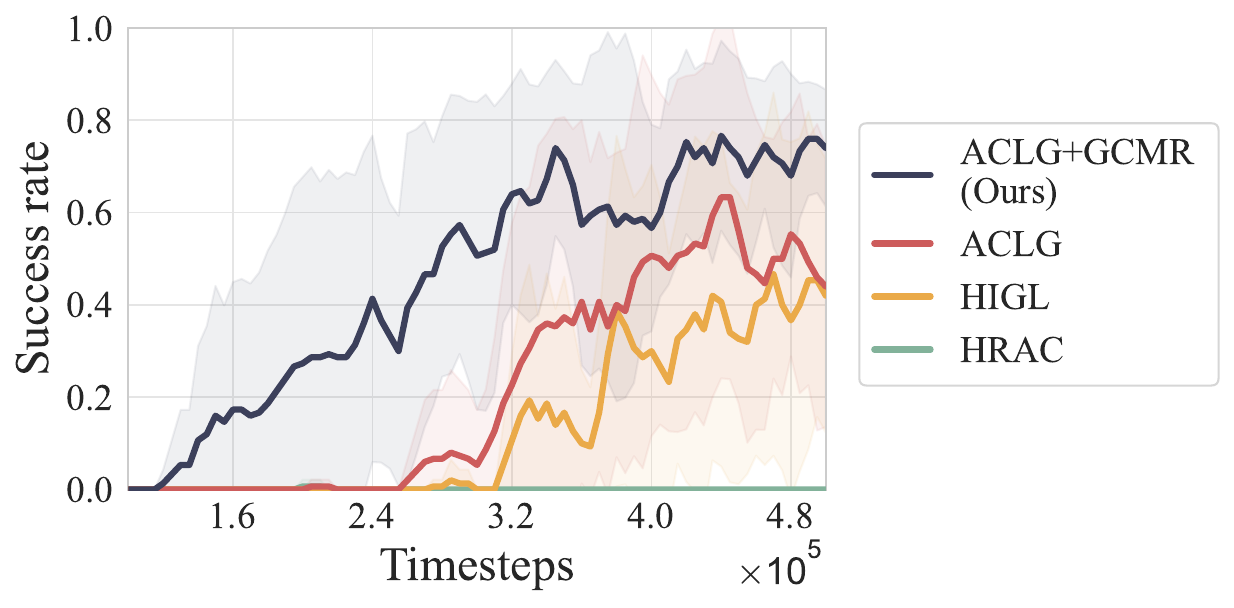}
\label{landmark_num}
\vspace{-0.8em}
\caption{The average success rate on the Ant Maze (\textit{Dense}, U-shape). Note that the comparison does not involve DHRL and PIG due to the scope and limitations of their applicability. The solid lines represent the mean across five runs. The transparent areas represent the standard deviation.}
\label{dense_compare}
\end{wrapfigure}
But we also present a discussion about dense experiments based on AntMaze (U-shape), as shown in Fig.~\ref{dense_compare}. Note that the comparison on dense reward setting did not involve DHRL and PIG due to the scope and limitations of their applicability. In the implementation, we did not use the learned dynamics model until we had sufficient transitions for sampling, which would avoid a catastrophic performance drop arising from inaccurate planning. It means that our method was enabled only if the step number of interactions was over a pre-set value. Here, the time step limit was set to $20K$ for maze-related tasks and $10K$ for robotic arm manipulation. After that, the dynamics model was trained at a frequency of $D$ steps. In the end, we evaluate their performance over 5 random seeds \chadded{(the seed number is set from 0 to 5 for all tasks)}, conducting 10 test episodes every $5K^{\rm th}$ time step.  All of the experiments were carried out on a computer with the configuration of Intel(R) Xeon(R) Gold 5220 CPU @ 2.20GHz, 8-core, 64 GB RAM. And each experiment was processed using a single GPU (Tesla V100 SXM2 32 GB). We provide more detailed experimental configurations in the "Supplementary Materials".

\paragraph{\chadded{Subgoals generated by different HRL algorithms}}
\chadded{We visualize the subgoals generated by the high-level policies of different HRL algorithms. As depicted in Fig.~\ref{subgoal_generate}, due to the introduction of adjacency constraints, HRAC\cite{zhang2020generating}, HIGL\cite{kim2021landmark}, ACLG, and ACLG+GCMR generated subgoals with better reachability (fewer outliers), compared to HIRO\cite{nachum2018data}. Moreover, HIGL\cite{kim2021landmark}, ACLG, and ACLG+GCMR displayed a preference for exploration, attributed to the use of landmark-based planning. However, the subgoals of HIGL exhibited a greater goal-reaching distance, farther from the current state. Conversely, ACLG and ACLG+GCMR achieved a better balance between reachability and planning. Fig.~\subref*{reachability_comp} provides statistical evidence supporting these findings, with the goal-reaching distance quantified using the lower-level Q functions\cite{huang2019mapping,kim2021landmark,kim2022imitating}. The results show that ACLG and ACLG+GCMR had smaller average goal-reaching distances.}
\begin{figure}[htbp]
\captionsetup[subfloat]{format=hang, justification=centering}
\centering
\subfloat[\chadded{The ant pursues a subgoal in Ant Maze}]{\includegraphics[width=0.24\textwidth]{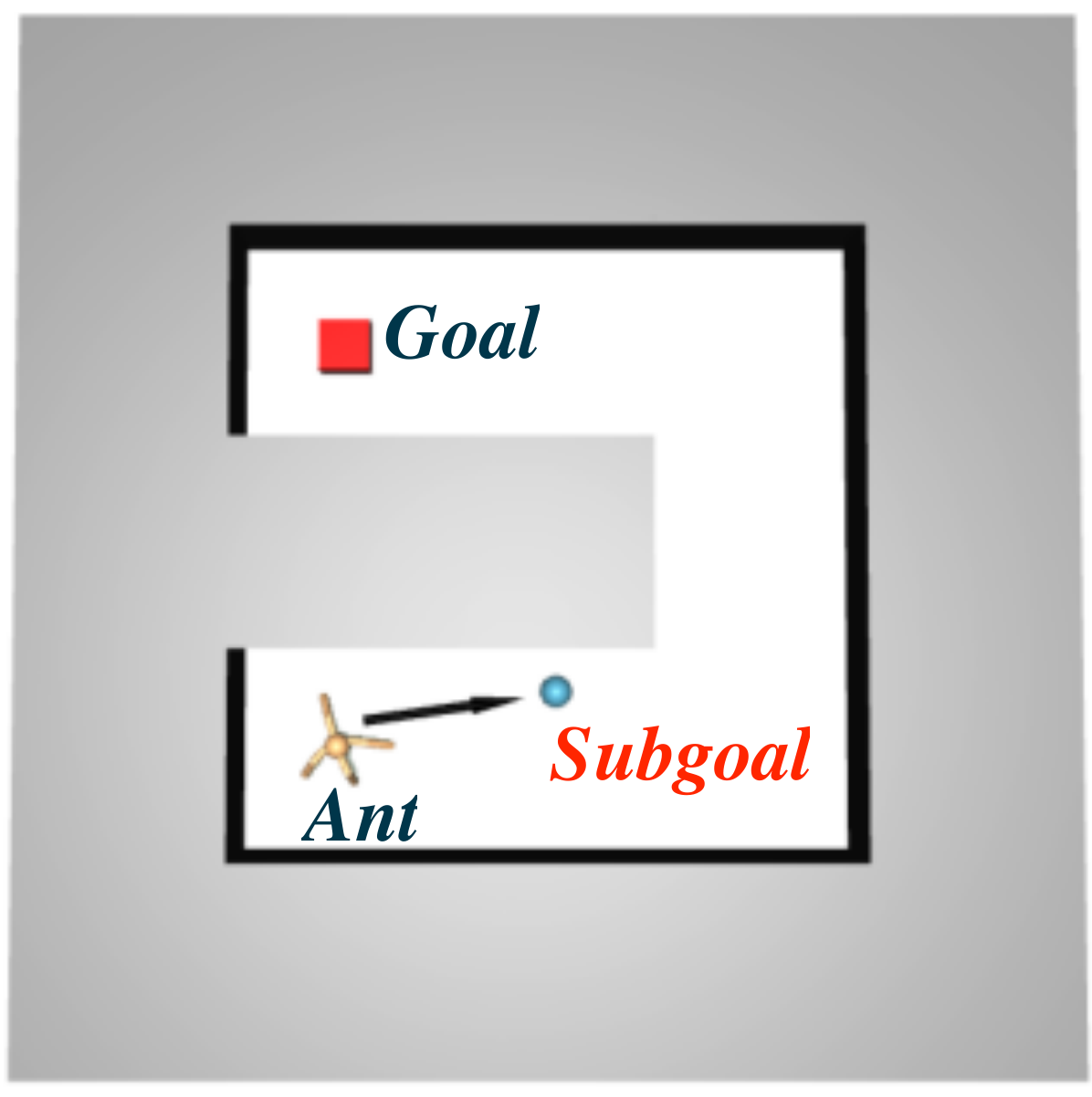}} 
\subfloat[\chadded{Subgoals generated by \textit{HIRO}\cite{nachum2018data}}]{\includegraphics[width=0.24\textwidth]{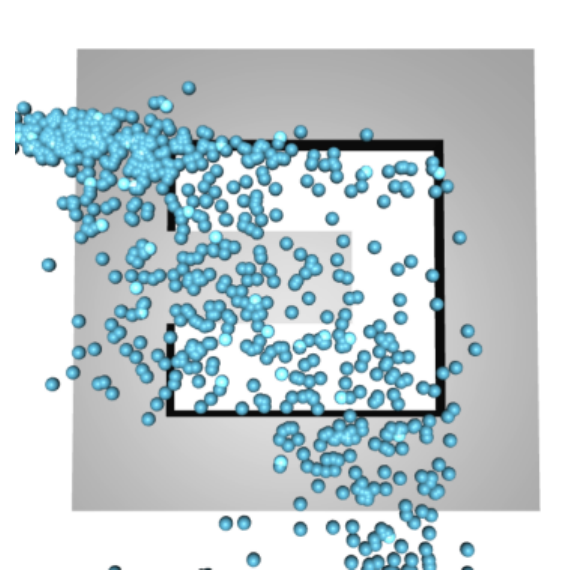}} 
\subfloat[\chadded{Subgoals generated by \textit{HRAC}\cite{zhang2020generating}}]{\includegraphics[width=0.24\textwidth]{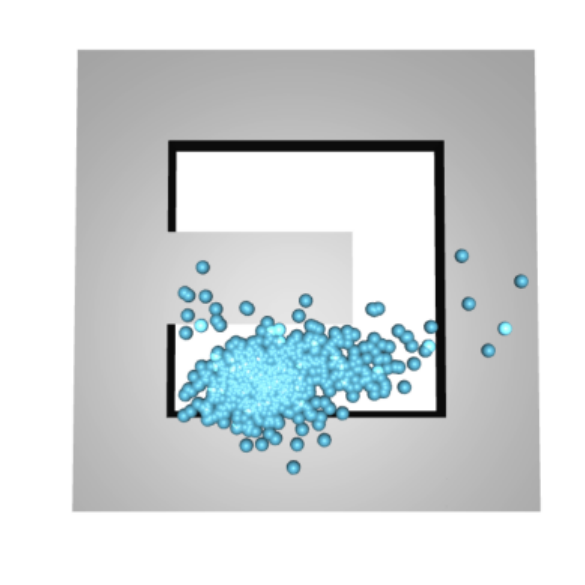}}
\subfloat[\chadded{Subgoals generated by \textit{HIGL}\cite{kim2021landmark}}]{\includegraphics[width=0.24\textwidth]{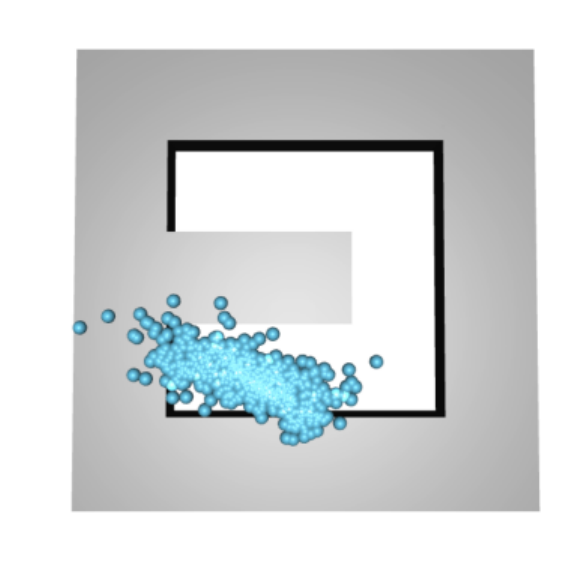}} \\
\subfloat[\chadded{Subgoals generated by \textit{ACLG}}]{\includegraphics[width=0.24\textwidth]{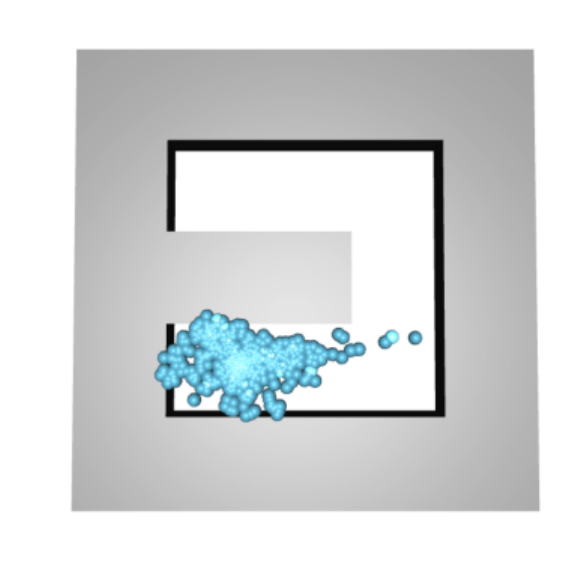}} 
\subfloat[\chadded{Subgoals generated by \textit{ACLG+GCMR}}]{\includegraphics[width=0.24\textwidth]{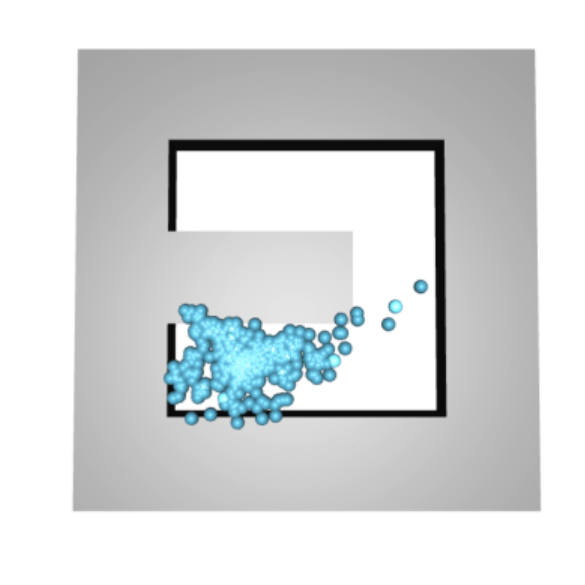}}\hspace{1.0em}
\subfloat[\chadded{Violin plots illustrating goal-reaching distance based on the lower-level value function.\cite{huang2019mapping,kim2021landmark,kim2022imitating}.}]{\includegraphics[width=0.4\textwidth]{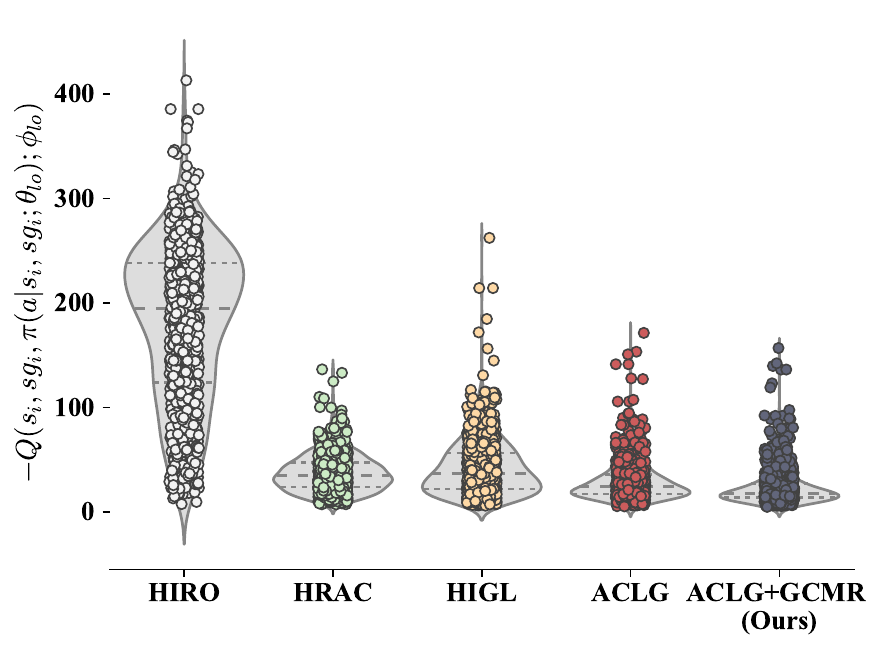}\label{reachability_comp}} 
\caption{\chadded{Visualizations and the goal-reaching distance measure of subgoals in the Ant Maze (U-shape) at 0.1M steps, based on a single run with the same random seed.}}
\label{subgoal_generate}
\end{figure}

\paragraph{Comparison results}

As shown in Fig.~\ref{compare_results}, GCMR contributes to achieving better performance and shows resistance to performance degradation. By integrating the GCMR with ACLG, we find that the proposed method outperforms the prior SOTA methods in almost all tasks. Especially in complicated tasks requiring meticulous operation (e.g., Ant Maze-Bottleneck, \textit{Stochastic} Ant Maze, and \textit{Large} Ant Maze), our method steadily improved the policy without getting stuck in local optima. \chreplaced{In most tasks}{In all Ant Maze tasks}, as shown in Fig.~\subref*{maze_u}, \subref*{maze_u_stoch}, \subref*{maze_u_large}, \subref*{maze_w}, \subref*{maze_u_bot}, \chdeleted{and }\ref{maze_point}, \chadded{\subref*{pusher_comp}, and \subref*{fetch_pp_comp}}, our method achieved a faster asymptotic convergence rate than others. There was no catastrophic failure. \chreplaced{O}{Although o}ur method slightly trailed behind the PIG\chreplaced{ in the Reacher (see Fig.~\subref*{reacher_comp}) and FetchPush (see Fig.~\subref*{fetch_push_comp}) tasks}{ and DHRL on Reacher (see Fig.~\ref{reacher_comp}) and Pusher (see Fig.~\ref{pusher_comp}) tasks, respectively}, it still achieved the second-best performance. Moreover, in Fig.~\ref{dense_compare}, we investigated the performance of the proposed method in the \textit{Dense}-reward environment, i.e., Ant Maze (\textit{Dense}, U-shape). The results demonstrate the GCMR is also effective and significantly improves the performance of ACLG. To verify whether GCMR can be solely applied to goal-reaching tasks, we conducted experiments in the Point Maze and Ant Maze (U-shape) tasks. From the experimental results depicted in Fig.~1 in the "Supplementary Materials", it can be observed that the GCMR can be used independently and achieve similar results to HIGL.


\renewcommand{\dblfloatpagefraction}{.95}
\begin{figure*}[!ht]
\captionsetup[subfloat]{format=hang, justification=centering}
\centering
\subfloat{\includegraphics[width=0.6\textwidth]{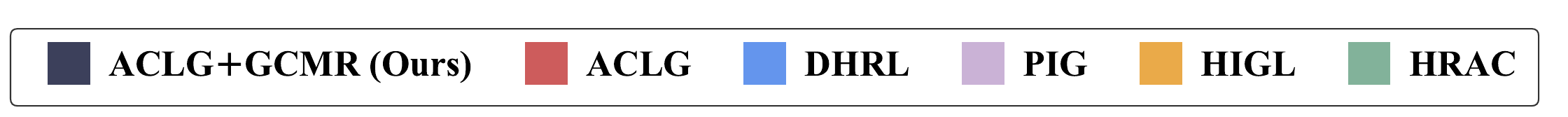}}\vspace{-4mm}
\setcounter{subfigure}{0}\\
\subfloat[Ant Maze (U-shape)]{\includegraphics[width=0.215\textwidth]{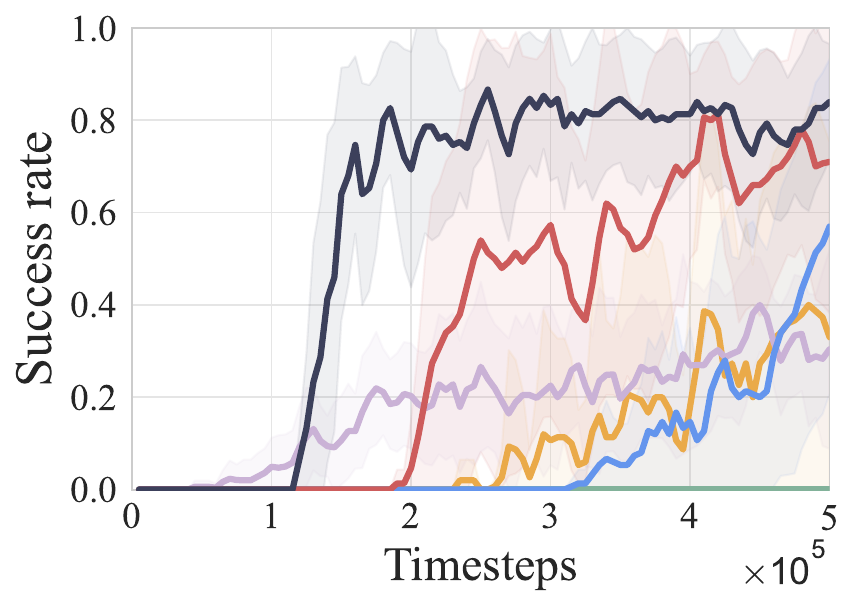}
\label{maze_u}}\hspace*{-0.8em}
\subfloat[\textit{Stochastic} Ant Maze \protect\\ (U-shape)]{\includegraphics[width=0.2\textwidth]{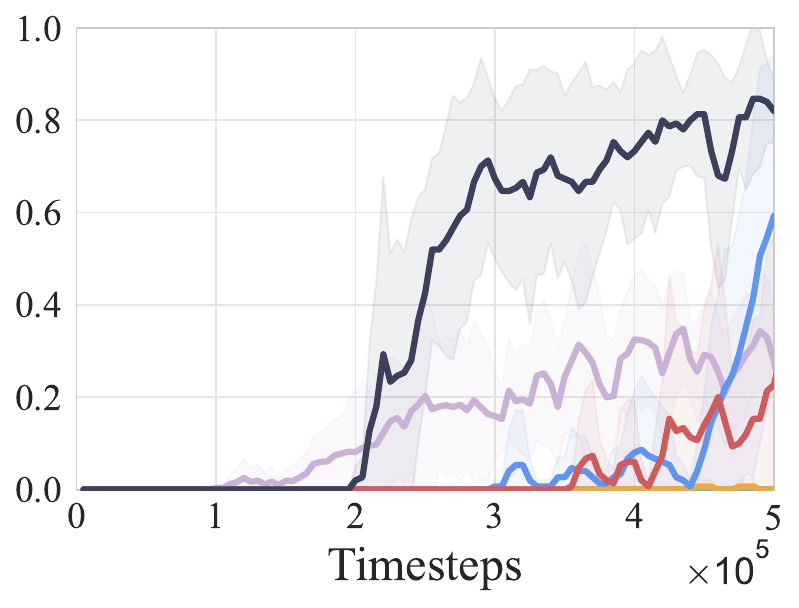}
\label{maze_u_stoch}}\hspace*{-0.8em}
\subfloat[\textit{Large} Ant Maze (U-shape)]{\includegraphics[width=0.2\textwidth]{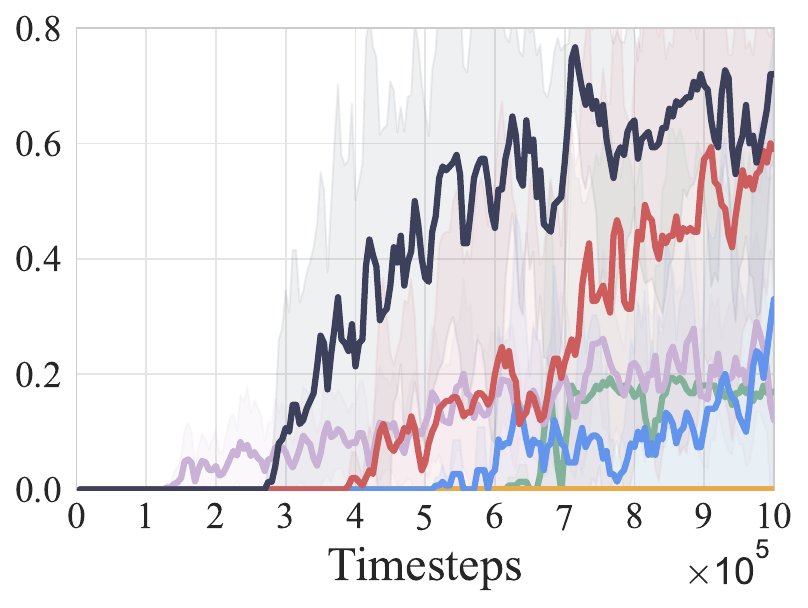}
\label{maze_u_large}}\hspace*{-0.8em}
\subfloat[Ant Maze (W-shape)]{\includegraphics[width=0.2\textwidth]{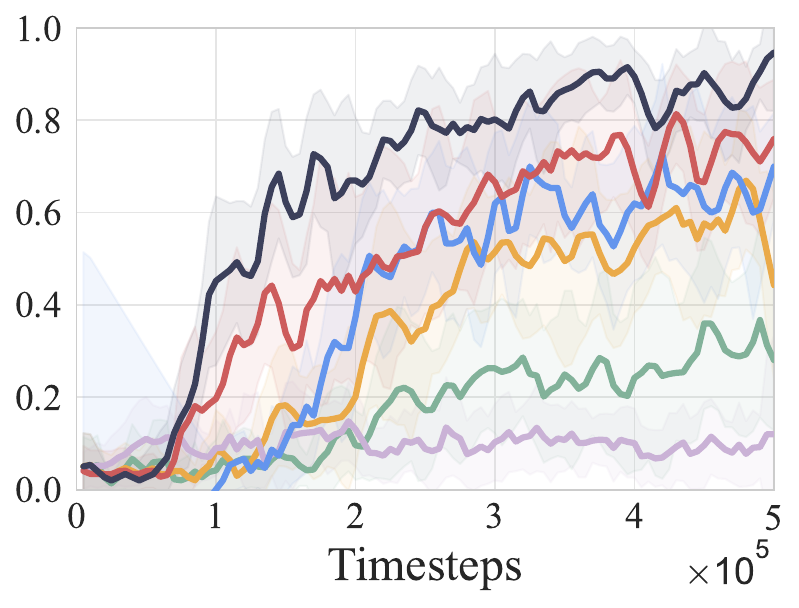}
\label{maze_w}}\hspace*{-0.8em}
\subfloat[Ant Maze-Bottleneck]{\includegraphics[width=0.2\textwidth]{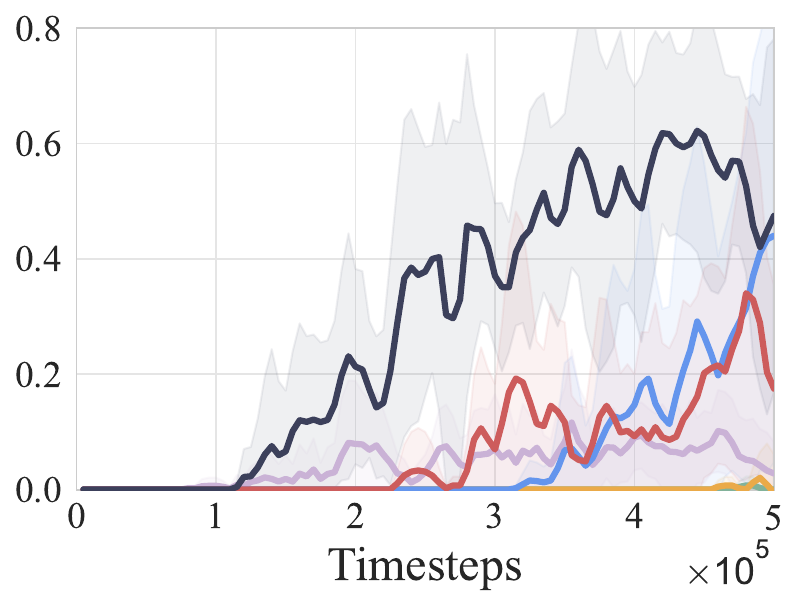}
\label{maze_u_bot}}\\ \vspace{-1.2em}
\subfloat[Point Maze]{\includegraphics[width=0.215\textwidth]{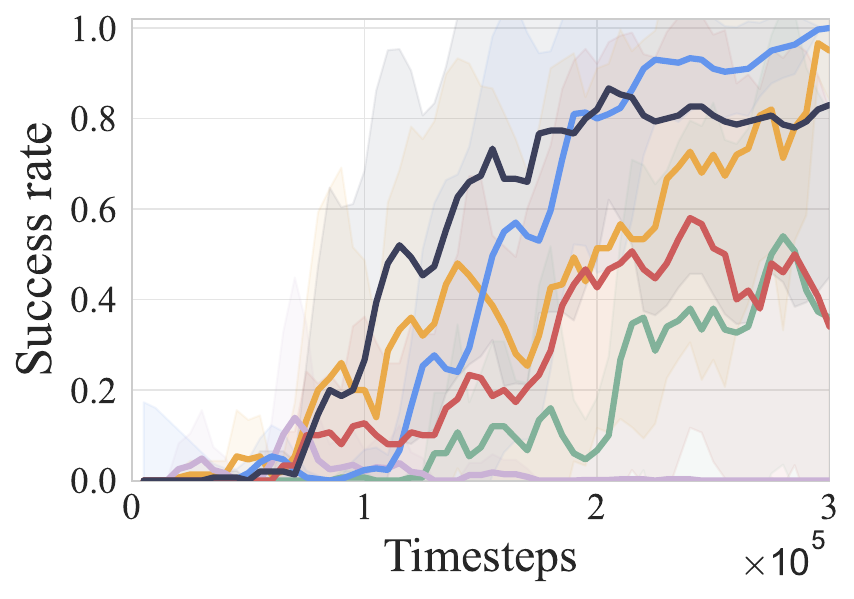}
\label{maze_point}}\hspace*{-0.8em}
\subfloat[Reacher]{\includegraphics[width=0.2\textwidth]{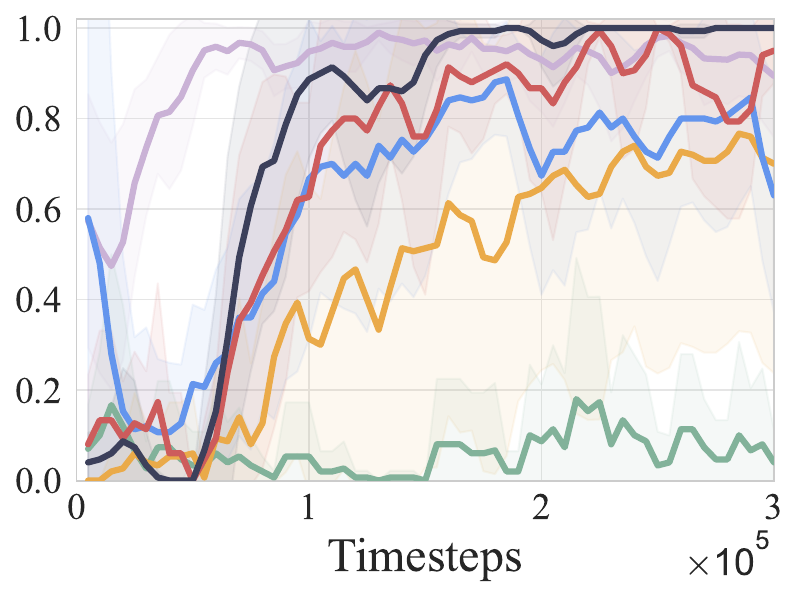}
\label{reacher_comp}}\hspace*{-0.8em}
\subfloat[Pusher]{\includegraphics[width=0.2\textwidth]{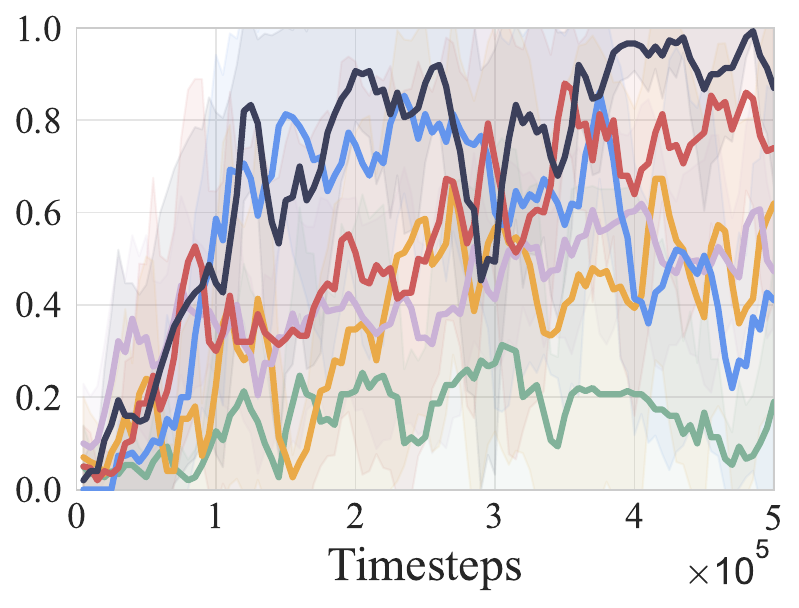}
\label{pusher_comp}}\hspace*{-0.8em}
\subfloat[FetchPickAndPlace]{\includegraphics[width=0.2\textwidth]{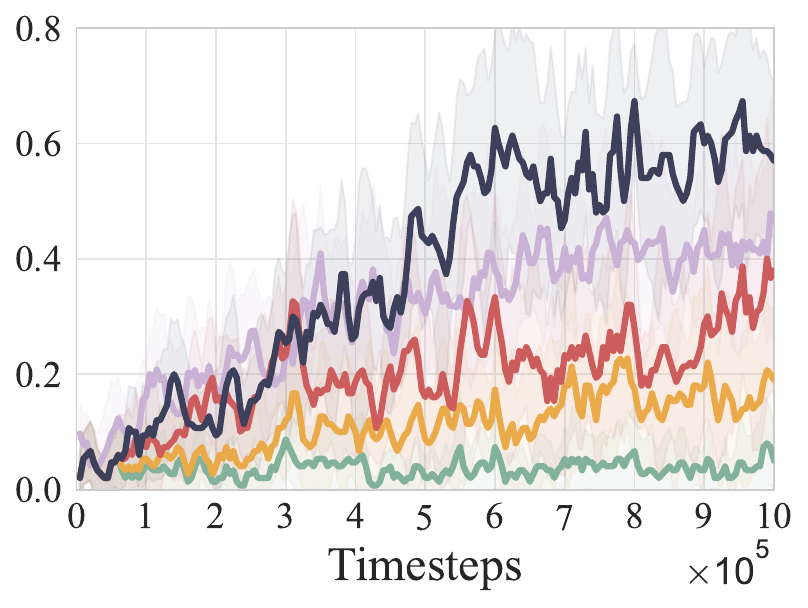}
\label{fetch_pp_comp}}\hspace*{-0.8em}
\subfloat[FetchPush]{\includegraphics[width=0.2\textwidth]{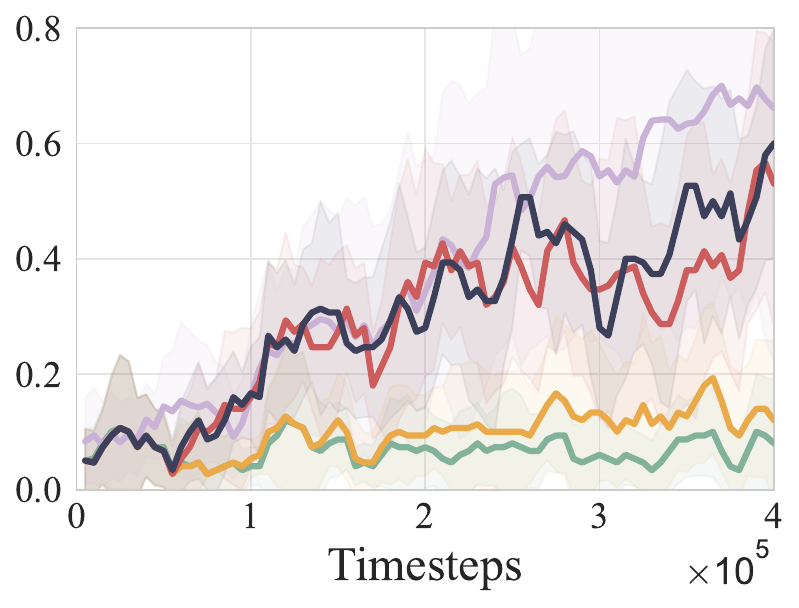}
\label{fetch_push_comp}}
\caption{The average success rate of multiple comparison methods on a set of \textit{Sparse}-reward environments. The solid lines represent the mean across five runs.}
\label{compare_results}
\end{figure*}

\paragraph{Comparison to existing goal-relabeling}
To justify the superiority of \chreplaced{the proposed \chadded{model-based off-policy correction }}{the proposed goal-relabeling approach}over the others. We compared it with various goal-relabeling technologies: (a) vanilla off-policy correction\chdeleted{(OPC)} in HIRO \cite{nachum2018data}, (b) hindsight-based goal-relabeling in HAC \cite{levy2019learning}, and (c) foresight goal inference\chdeleted{(FGI)} in MapGo \cite{zhu2021mapgo}, which is a model-based variant of vanilla hindsight-based goal-relabeling. 
\chadded{The average success rate illustrated in Fig.~\ref{relabel} highlights the significance of (both HIRO-style and HAC-style) relabeling for enhancing data efficiency, leading to accelerated learning compared to the case not using relabeling. Moreover, Fig.~\ref{relabel} also illustrates that HIRO-style relabeling outperforms HAC-style ones. Among them, the proposed relabeling with the soft-relabeling exhibits a slight advantage over the vanilla off-policy correction. However, the final performance of the proposed relabeling method is sensitive to the value of the shift magnitude $\delta_{sg}$, as shown in Fig.~\ref{goal_shift}. When soft-relabeling is not applied, the proposed relabeling exhibits inferior performance compared to the vanilla off-policy correction during the early stages of training.}
\chdeleted{Before making the comparison, we would like to first recognize the effectiveness of two tricks: the exponential weighting and soft-relabeling. \cite{janner2019trust} has been systematically explained that inaccuracies in learned models make long rollouts unreliable due to the compounding error. The gap between true returns and learned model returns cannot be eliminated, hence cumulative errors will be huge for too long rollouts. To address this, we employed an exponential weighting function along the horizontal axis to highlight shorter rollouts. Fig.~\ref{exp_weight} illustrates the significance of the exponential weighting in suppressing cumulative errors. On the other hand, }\chdeleted{soft relabeling was utilized to enhance the robustness of goal relabeling against outliers\chdeleted{, ensuring that the relabelled subgoals remain in close proximity to the original ones}. Fig.~\ref{relabel} provides insights into the impact of soft-relabeling on different goal-relabelling technologies. It can be observed that the soft-relabeling enhanced the robustness of different goal-relabelling technologies. \chdeleted{In the end, we investigated the impact of the model-based gradient penalty on different goal-relabelling technologies. As shown in Fig.~\ref{relabel}, all goal-relabelling technologies had a faster asymptotic convergence rate under gradient penalty. The result highlights the role of the gradient penalty in enhancing robustness against high-level errors, such as an unreachable or faraway goal.}}

\begin{figure}[!ht]
\centering
\includegraphics[width=0.7\textwidth]{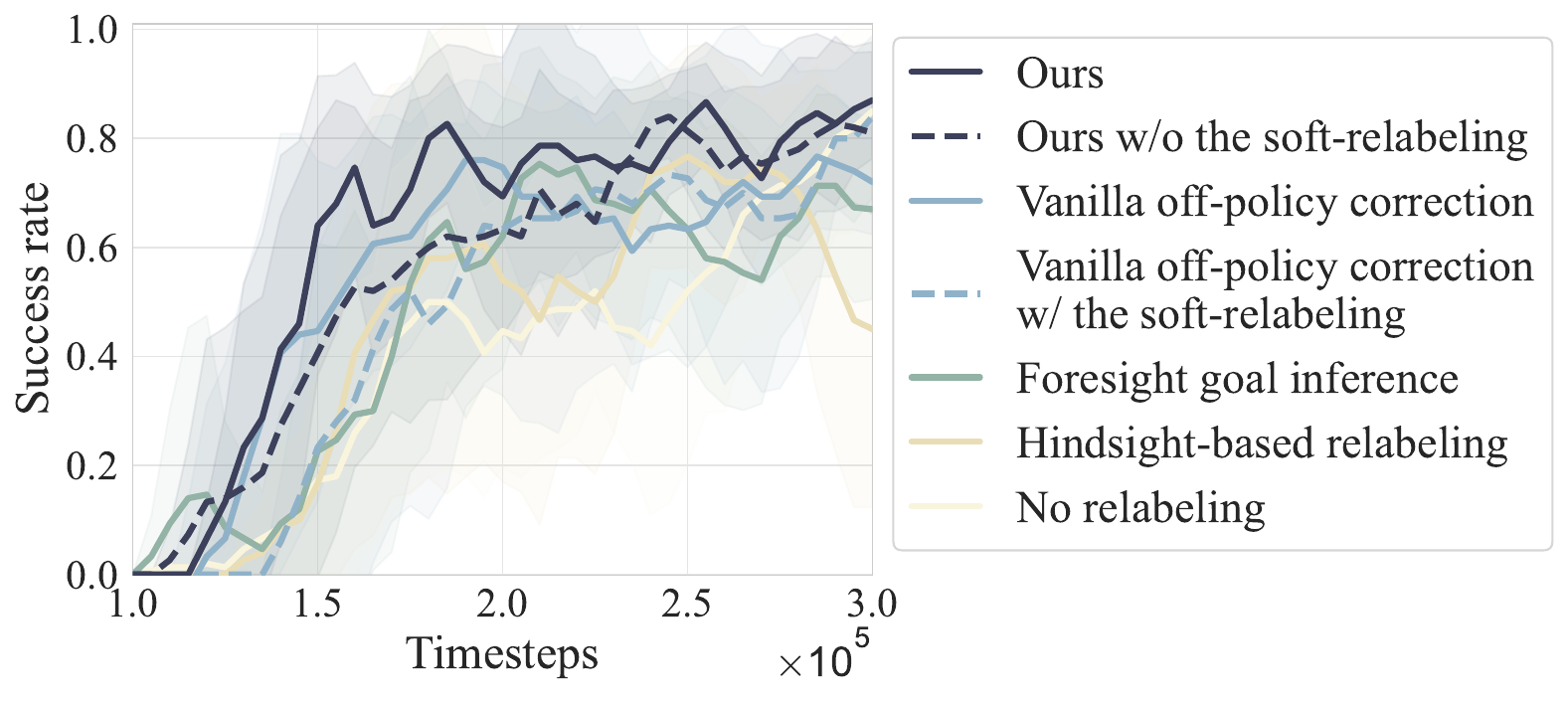}
\caption{Figure\chdeleted{ shows the importance of components related to relabeling and} compares the performance of different \chdeleted{goal-}relabeling technologies\chdeleted{: (a) the learning curves of ACLG based on the proposed relabeling method with and without exponential weighting on Ant Maze (U-shape), (b) the learning curves of ACLG based on different goal-relabeling technologies with (solid line) and without (dotted line) the soft-relabeling on the Point Maze, and (c) the learning curves of ACLG based on different goal-relabeling technologies (after applying the soft-relabeling) with (solid line) and without (dotted line) gradient penalty} on Ant Maze (U-shape). \chadded{The learning curves are plotted based on the average of over five independent runs}.}
\label{relabel}
\end{figure}

\subsection{Ablation study}
\chadded{We also investigate the effectiveness of different crucial components in our method.}
\paragraph{\chadded{Increased Critic Training Iterations in Lower-level}}
\chdeleted{In this section, we investigate the effectiveness of two crucial factors in our algorithm: the gradient penalty term and the one-step planning term. }Considering that the number of lower-level critic training iterations was increased to alleviate the impact of the gradient penalty, we additionally provide a comprehensive analysis concerning the effects of increased iterations on various alternative methods. As depicted in Fig.~\ref{crit5}, increasing the number of critic training iterations led to improved performance when compared to the original approach. Moreover, even without increasing the training iterations, ACLG+GCMR consistently outperformed other methods and overtook the ACLG with increased training iterations after several timesteps. The results demonstrate that the \chreplaced{gradient penalty}{GCMR} can enhance the robustness of HRL frameworks and prevent falling into local pitfalls.

\begin{figure}[htbp]
\captionsetup[subfloat]{format=hang, justification=centering}
\centering
\includegraphics[width=0.62\textwidth]{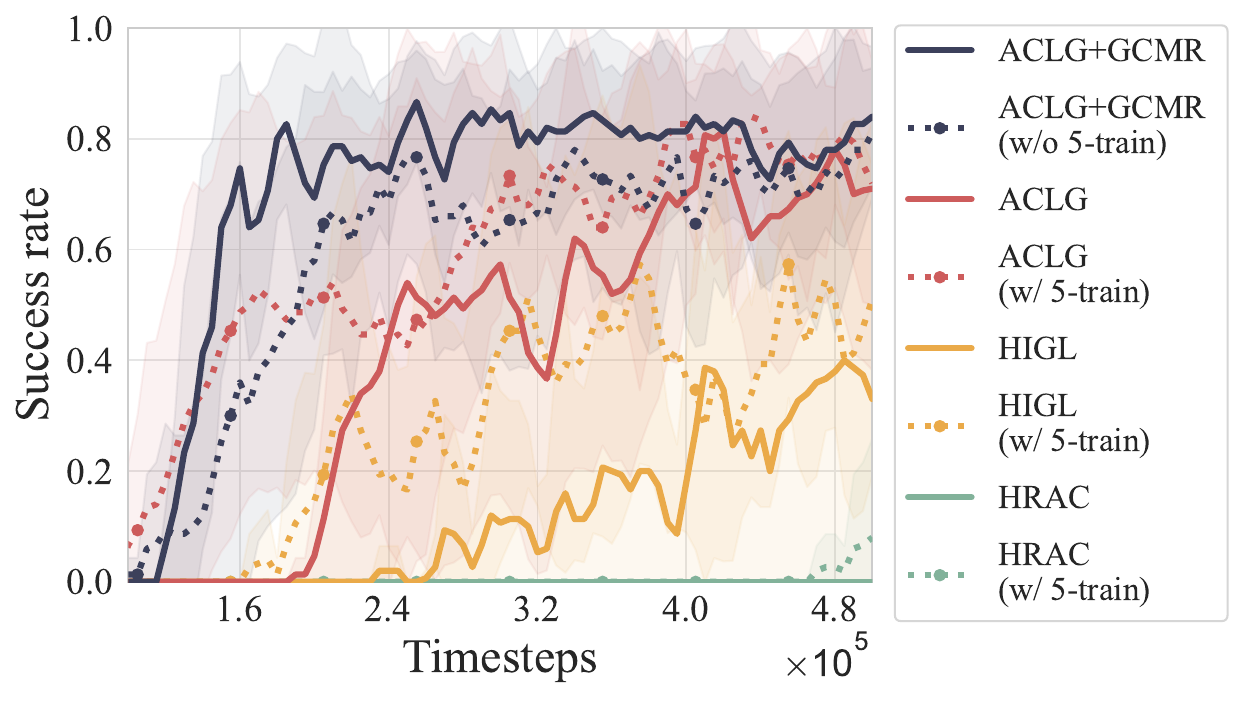}
\caption{We investigate the impact of increased training iterations for critic on various HRL methods in the Ant Maze (U-shape) environment, where "5-train" indicates that the number of training iterations of lower-level critic network is increase to 5.}
\label{crit5}
\end{figure}

\paragraph{\chadded{Ablation Study on Gradient Penalty term $\mathcal{L}_{gp}$}}
\chadded{In Fig.~\ref{weights_actor_gp} and Fig.~\ref{mgp_abs}, we investigate the effectiveness of the gradient penalty term on enhancing the generalization of the low-level policy and the impact on the final performance, respectively. Fig.~\ref{weights_actor_gp} \subref{weights_wo_mgp} and \subref{weights_w_mgp} depict the weight distribution of the final layer in the low-level actor over training steps (100K, 300K, and 500K). Intriguingly, we observe that in the constrained network with the gradient penalty, the weights fewer centred at zero compared to those in the unconstrained network. The uncertainty in the weights of the network has been evidenced to potentially enhance generalization performance\cite{blundell2015weight}. Meanwhile, in Fig.~\subref*{state_q}, the result of the value estimation indicates that even in unseen scenarios, similar state-action pairs still yield higher value estimations when applying the gradient penalty, thereby demonstrating robust generalization. Here, we shift the state distribution to simulate the unseen scenarios. Moreover, the learning curves in Fig.~\ref{mgp_abs} demonstrate that the application of gradient penalty contributes to achieving better asymptotic performance compared to not using it.}

\begin{figure}[htbp]
\captionsetup[subfloat]{format=hang, justification=raggedright}
\centering
\subfloat[\chadded{Histogram of final layer's weights in low-level actor of ACLG+GCMR \textit{without} $\mathcal{L}_{gp}$}]{\includegraphics[width=0.3\textwidth]{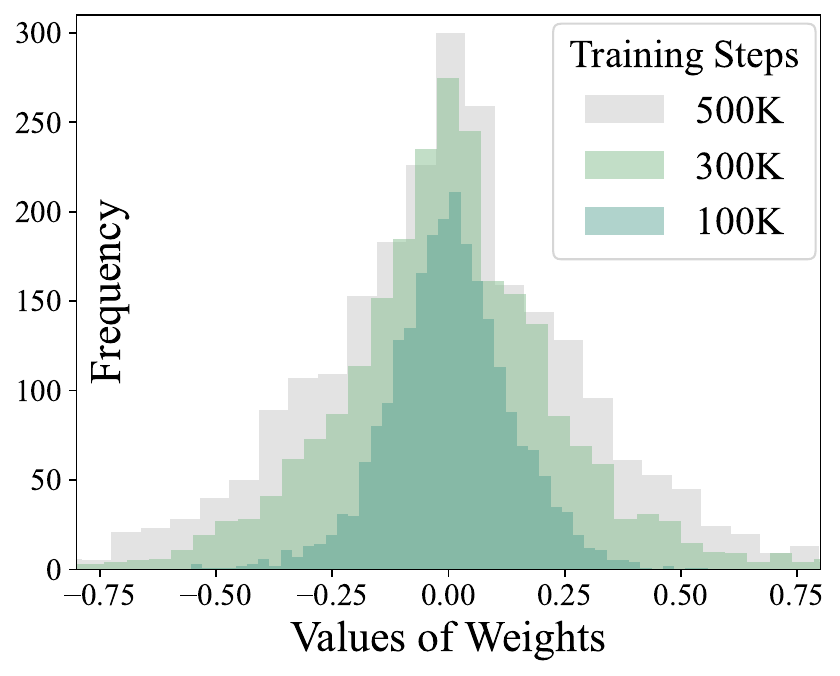}
\label{weights_wo_mgp}}\hspace*{-0.6em}
\subfloat[\chadded{Histogram of final layer's weights in low-level actor of ACLG+GCMR \textit{with} $\mathcal{L}_{gp}$}]{\includegraphics[width=0.3\textwidth]{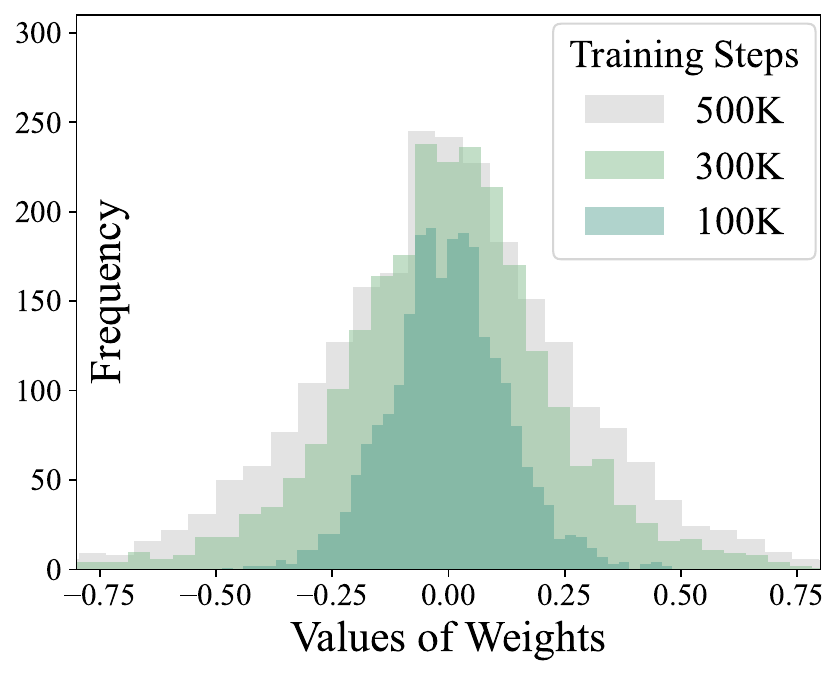}
\label{weights_w_mgp}}
\subfloat[\chadded{Value estimation based on lower-level Q-value functions}]{\includegraphics[width=0.38\textwidth]{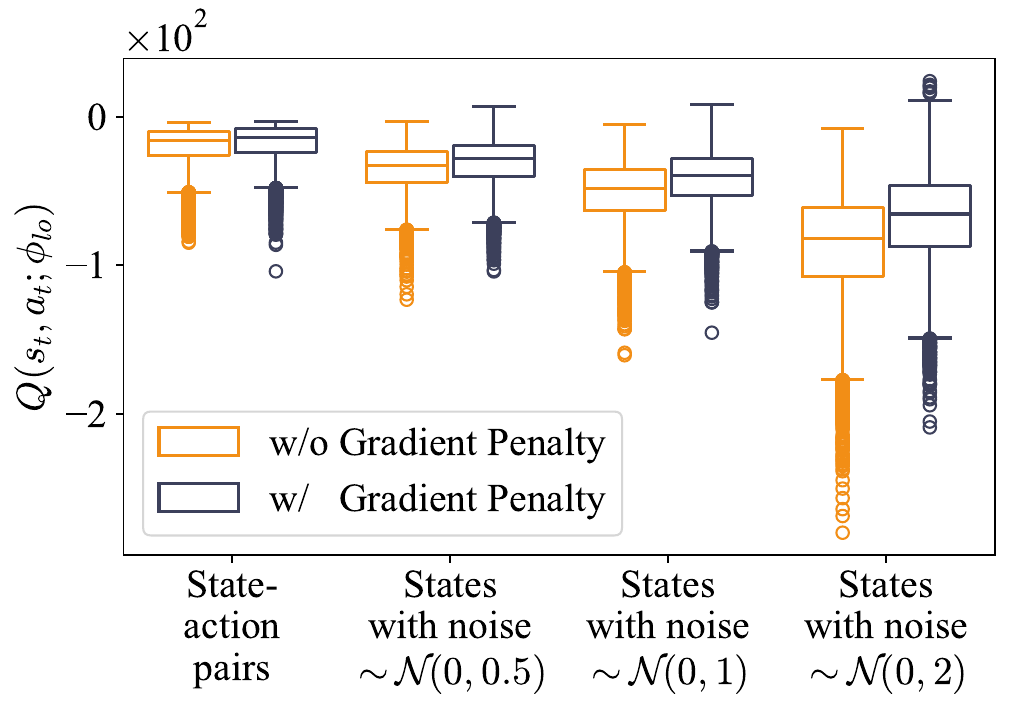}
\label{state_q}}\\
\caption{\chadded{Impact of gradient penalty on weight distribution and generalization of low-level policies. \subref{weights_wo_mgp} and \subref{weights_w_mgp} depict the weight distribution of the final layer in the low-level actor over training steps (100K, 300K, and 500K). These experiments are conducted in the Ant Maze (U-shape) task. \subref{state_q} illustrates value estimation on state-action pairs using lower-level Q functions of ACLG+GCMR w/o or w/ gradient penalty at 300K steps, in the presence of a state distribution shift. The state distribution shifts can simulate unseen scenarios.}}
\label{weights_actor_gp}
\end{figure}

\begin{figure}[!htbp]
\captionsetup[subfloat]{format=hang, justification=centering}
\centering
\subfloat[\chadded{Ant Maze (U-shape)}]{\includegraphics[width=0.4\textwidth]{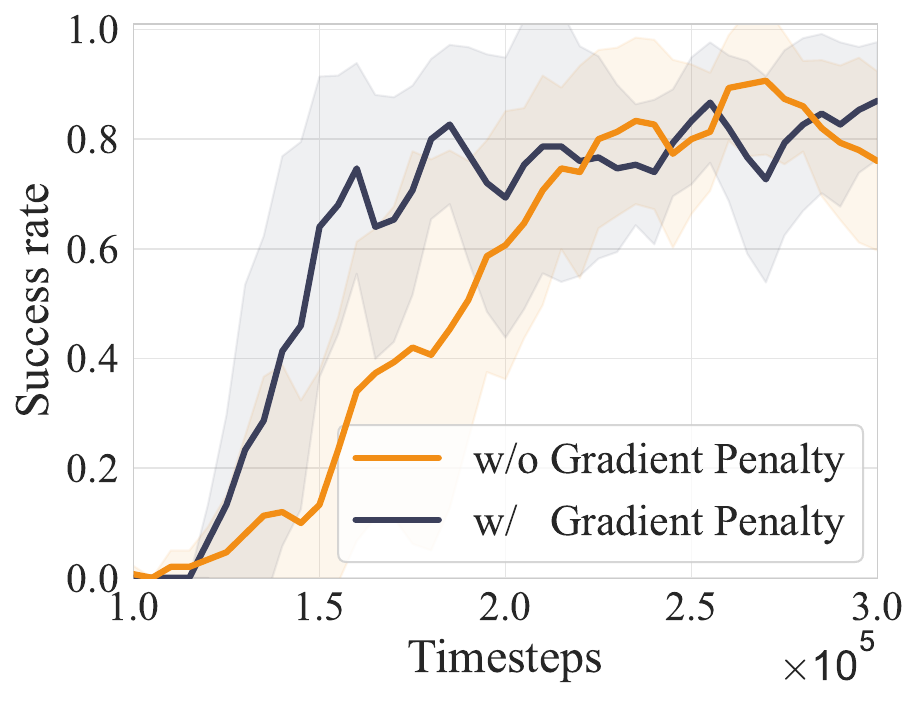}
\label{mgp_abs1}}\hspace*{-0.8em}
\subfloat[\chadded{Ant Maze-Bottleneck}]{\includegraphics[width=0.4\textwidth]{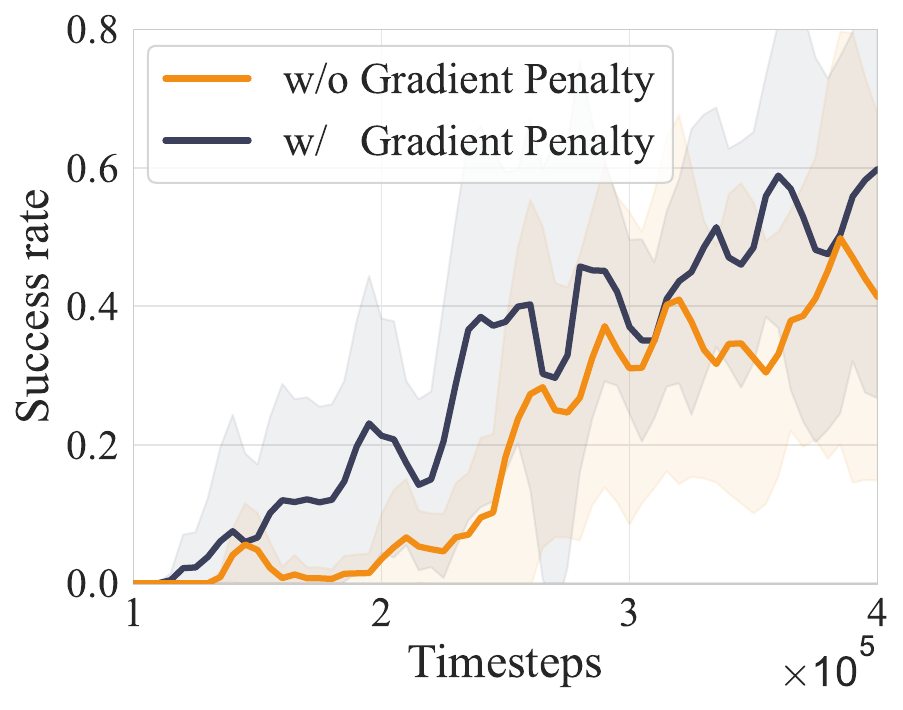}
\label{mgp_abs2}}\\
\caption{\chadded{Ablation study on gradient penalty term $\mathcal{L}_{gp}$. The success rate is averaged over 5 random seeds.}}
\label{mgp_abs}
\end{figure}

\paragraph{\chadded{Ablation Study on One-Step Rollout-based Planning term $\mathcal{L}_{osrp}$}}
\chadded{In Fig.~\ref{osrp_abs_traj} and Fig.~\ref{osrp_abs}, we clarify the role of the one-step rollout-based planning term and highlight its importance. As anticipated, in Fig.~\ref{osrp_abs_traj}, under the guidance of the one-step rollout-based planning, the agent consistently discovers smoother trajectories towards the final goal. These paths swiftly traverse the contours of higher-level Q-value, ultimately advancing towards positions of the highest higher-level Q-value. Comparing the first and second rows of Fig.~\ref{osrp_abs_traj}, we can observe that this effect is more pronounced in the Point Maze because its maximum distance within a single step exceeds that in the Ant Maze (U-shape). Also, in Fig.~\ref{osrp_abs}, we observe that applying the one-step rollout-based planning term can significantly enhance the final performance across various control tasks, demonstrating the effectiveness of the introduced planning term.}

\renewcommand{\dblfloatpagefraction}{.95}
\begin{figure}[htbp]
\captionsetup[subfloat]{format=hang, justification=centering}
\centering
\subfloat[\chadded{Ant Maze (U-shape)}]{\includegraphics[width=0.4\textwidth]{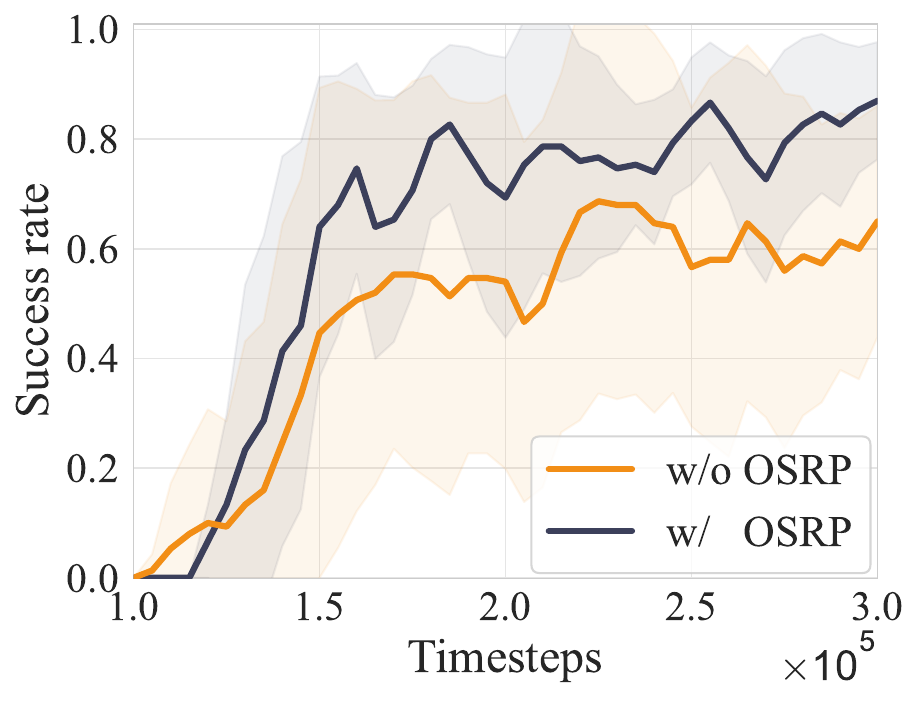}
\label{osrp_abs1}}\hspace*{-0.8em}
\subfloat[\chadded{Point Maze}]{\includegraphics[width=0.4\textwidth]{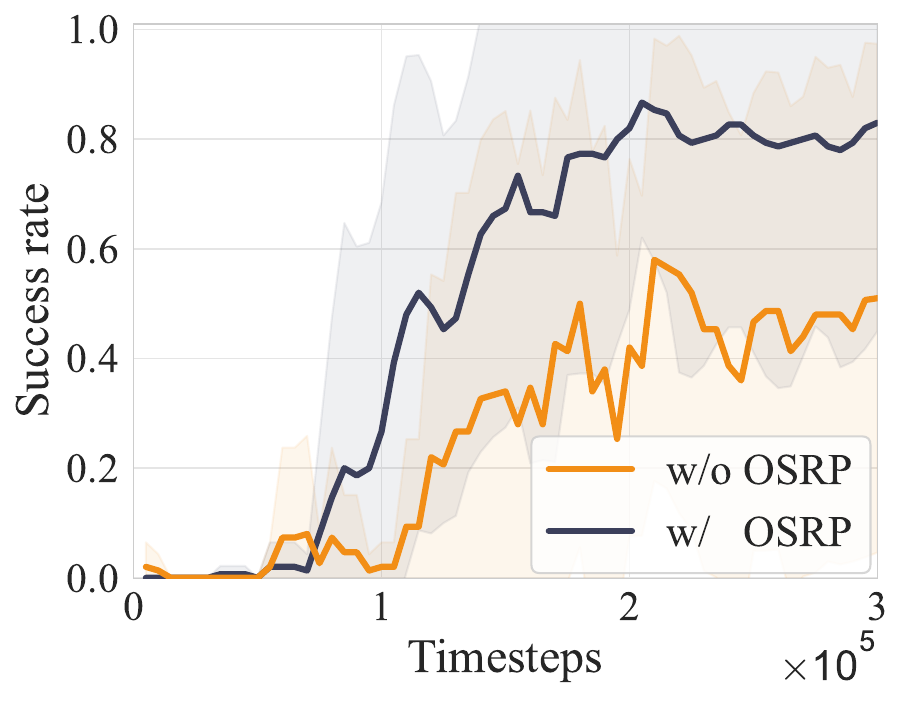}
\label{osrp_abs2}}\\
\caption{\chadded{Ablation study on one-step rollout-based planning term $\mathcal{L}_{osrp}$. The success rate is averaged over 5 random seeds.}}
\label{osrp_abs}
\end{figure}

\renewcommand{\dblfloatpagefraction}{.95}
\begin{figure}[htbp]
\captionsetup[subfloat]{format=hang, justification=centering}
\centering
\subfloat[\chadded{Ant Maze (U-shape)}]{\includegraphics[width=0.24\textwidth]{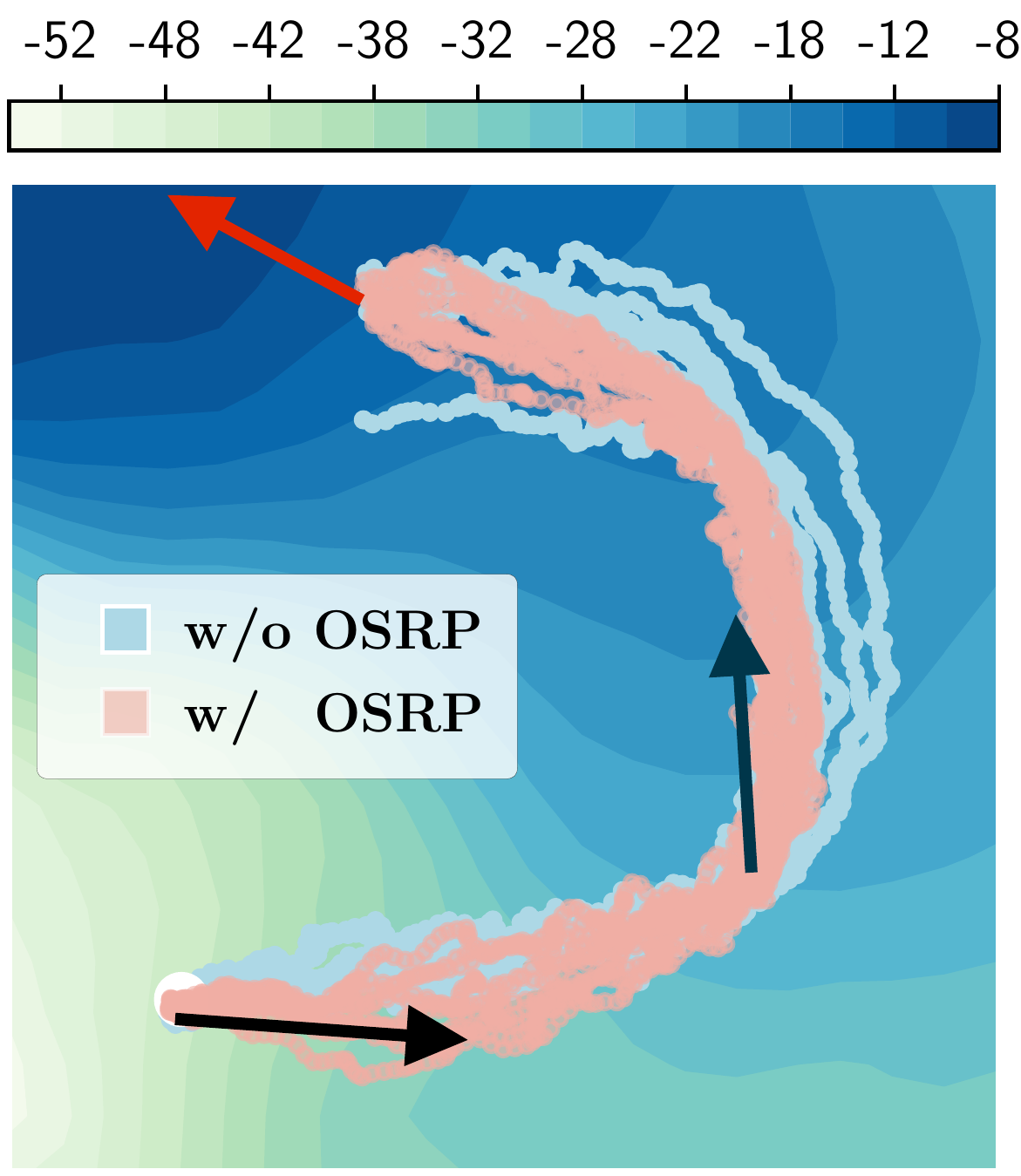}
\label{osrp_abs_traj11}}\hspace*{-0.6em}
\subfloat[\chadded{W/o OSRP}]{\includegraphics[width=0.24\textwidth]{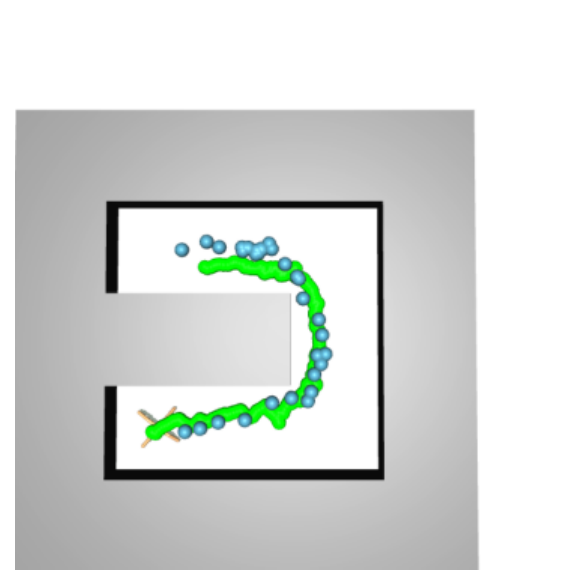}
\label{osrp_abs_traj12}}\hspace*{-0.6em}
\subfloat[\chadded{W/ OSRP}]{\includegraphics[width=0.24\textwidth]{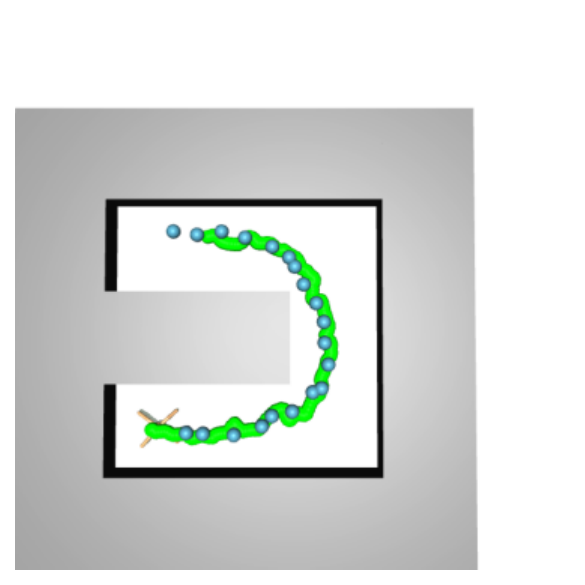}
\label{osrp_abs_traj13}}\\
\subfloat[\chadded{Point Maze}]{\includegraphics[width=0.24\textwidth]{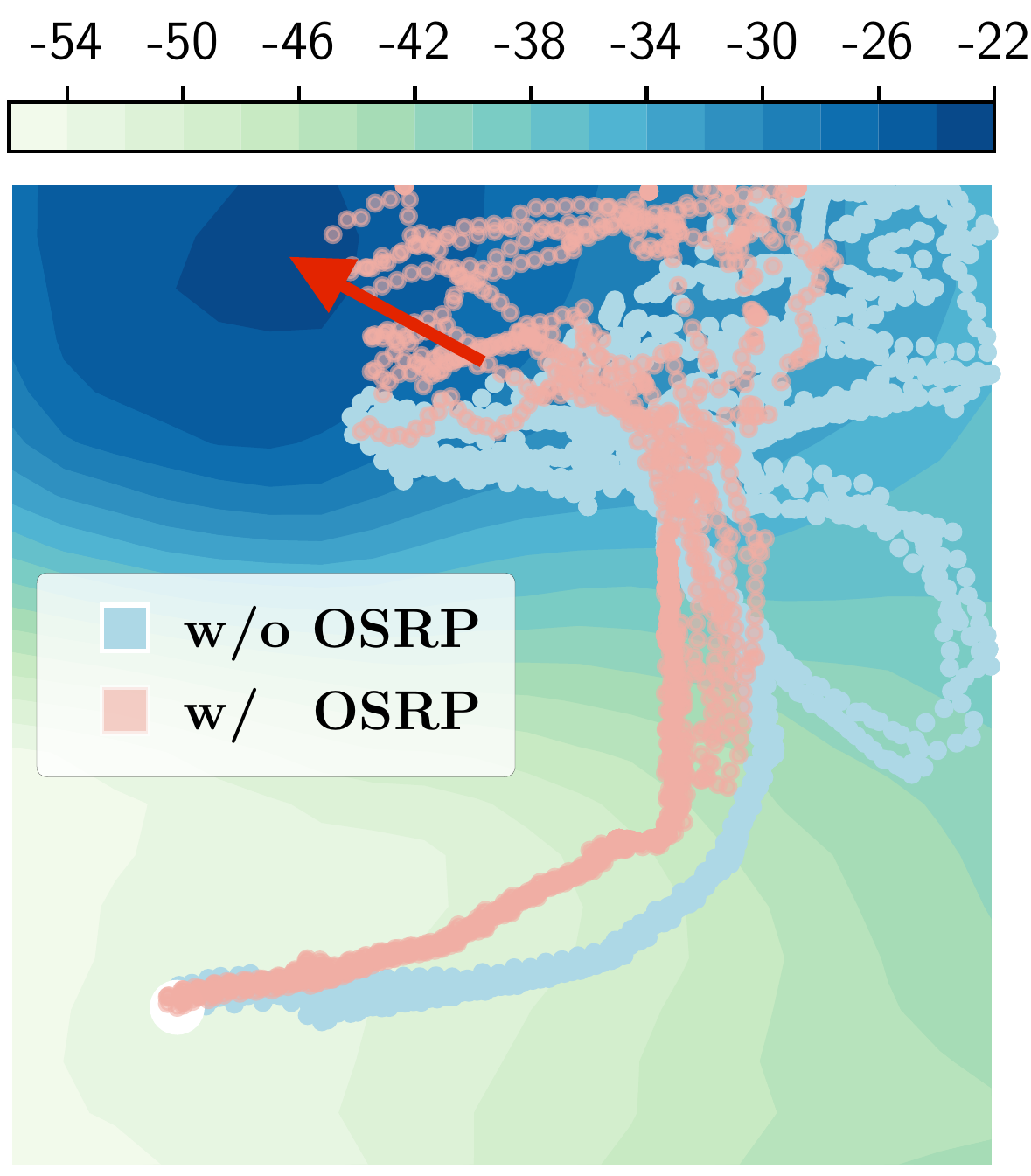}
\label{osrp_abs_traj21}}\hspace*{-0.6em}
\subfloat[\chadded{W/o OSRP}]{\includegraphics[width=0.24\textwidth]{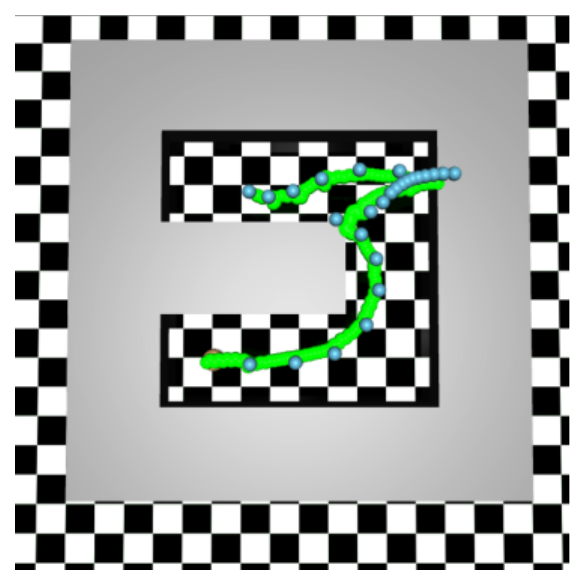}
\label{osrp_abs_traj22}}\hspace*{-0.6em}
\subfloat[\chadded{W/ OSRP}]{\includegraphics[width=0.24\textwidth]{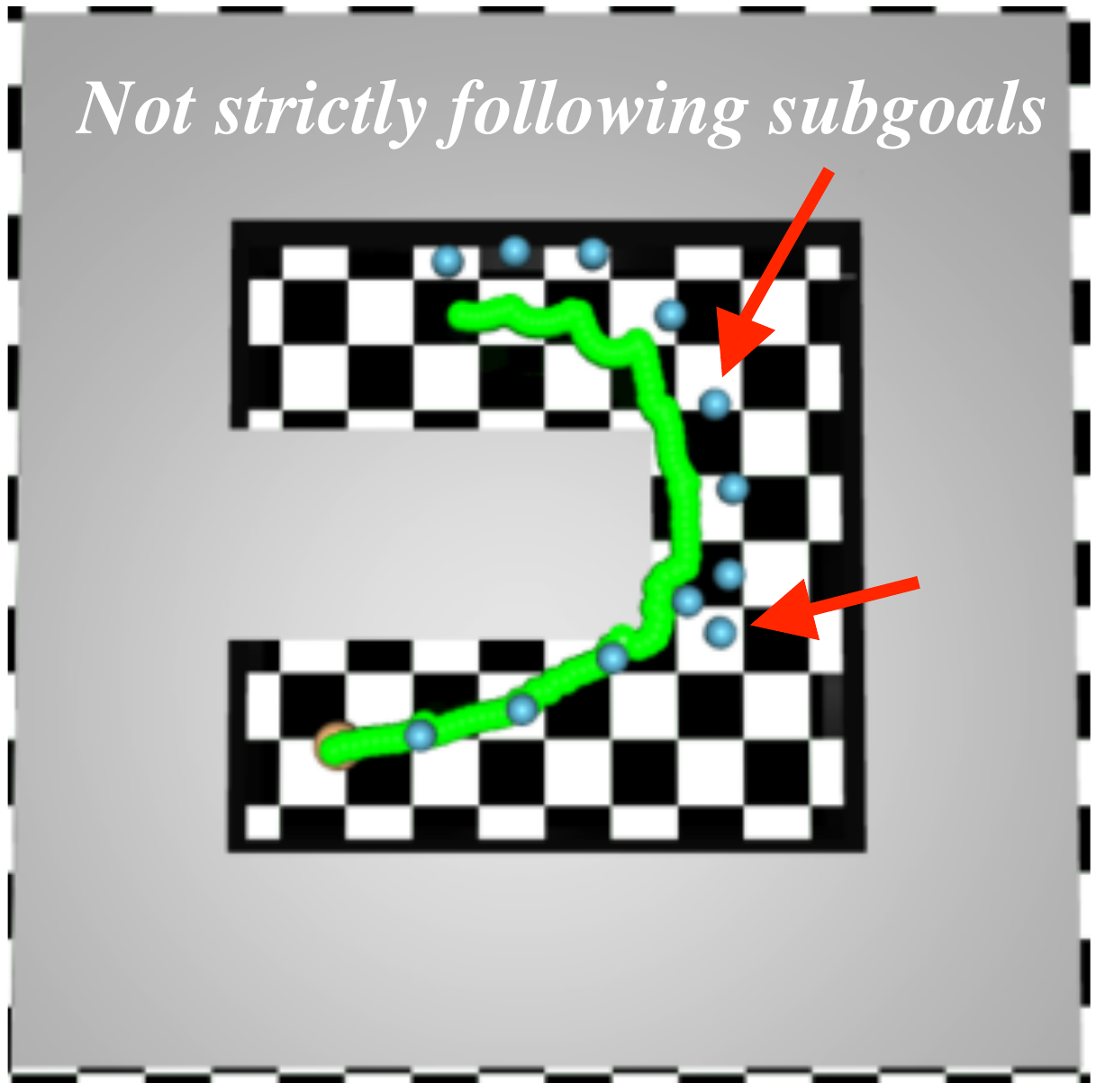}
\label{osrp_abs_traj23}}
\caption{\chadded{Trajectories of agents with or without the guidance of the one-step rollout-based planning. The policies of agents use the ACLG+GCMR and are trained for 0.3M steps. In \subref{osrp_abs_traj11} and \subref{osrp_abs_traj21}, the contours of higher-level Q-value are plotted.}}
\label{osrp_abs_traj}
\end{figure}

\section{Discussion}
\label{Discussion}
\subsection{\chadded{Accelerating reinforcement learning from the perspective of inter-level cooperation, with emphasis on the lower level.}}
\chadded{In prior works, there has been a strong emphasis on optimizing the higher-level policy, as global planning is determined by the higher-level policy. The lower-level policy is overlooked, merely seen as \textit{an "unintelligent" subordinate} to the higher level. Yet, in fact, the behavior of lower-level policy significantly influences the effectiveness of exploration and the stability of the hierarchical system because the lower-level policy interacts directly with the environment. Reinforcing the robustness of the lower-level policy can prevent catastrophic failures, such as collisions or tipping over, thereby accelerating reinforcement learning. This study reemphasizes the importance of the lower-level policy from the perspective of inter-level cooperation, drawing attention to related research on optimizing lower-level policies.}

\subsection{\chadded{Effectiveness of the gradient penalty and one-step rollout-based planning is significant, while model-based relabeling is weak.}}
\chadded{For facilitating inter-level cooperation and communication, we propose a novel goal-conditioned HRL framework, which mainly consists of three crucial components: the model-based off-policy correction for the data efficiency, the gradient penalty on the lower-level policy for the robustness, and one-step rollout-based planning for the cooperation. The experimental results indicate that the gradient penalty and one-step rollout-based planning achieved the expected effects, significantly enhancing the performance of the HRL framework. However, the model-based off-policy correction did not yield significant effects. 
The reason might be that these predicted states in correction could potentially suffer from compounding errors in long-horizon rollouts. Although soft-relabeling can mitigate this error, it is unstable due to the introduction of sensitive hyper-parameter $\delta_{sg}$ (see Fig.~\ref{goal_shift}).}

\subsection{\chadded{Limitations and future research}} This study has certain limitations. First, our experiments show that the GCMR achieved significant performance improvement, and such improvement came at the expense of more
computational cost (see Table~4 in the "Supplementary Materials" for a quantitative analysis of computational cost). However, the time-consuming issue only occurs during the training stage and will not affect the execution response time in the applications. Second, we need to clarify that the scope of applicability is off-policy goal-conditioned HRL. The effectiveness in general RL tasks or online tasks \chreplaced{still needs to be validated in future work}{had not been validated}. Third, the experimental environments used in this study have 7 or 30 dimensions. Our network architecture of transition dynamics models is relatively simple, leading to limited regression capability. Applications in complex environments that closely resemble real-world scenarios with high-dimensional observation, like the large-scale point cloud environments encountered in autonomous driving, might face limitations. This issue will be investigated in our future work. \chadded{Besides, related research on cooperative multi-robot HRL has successfully coped with extremely complex environments by enabling multiple robots to learn through the collective exchange of environmental data\cite{setyawan2022cooperative, setyawan2022depth, setyawan2022combinations}. Integrating the proposed algorithm into multi-robot HRL systems is expected to enhance the performance in complex environments further, and this will be investigated in future work. Finally, we can observe that the effect of one-step rollout-based planning is more pronounced in an environment with a higher maximum distance within a single step. This observation inspires us to delve into multi-step rollout-based planning in the future to broaden its application.}

\section{Conclusion}
\label{Conclusion}
This study proposes a new goal-conditioned HRL framework with Guided Cooperation via Model-based Rollout (GCMR), which uses the learned dynamics as a bridge for inter-level cooperation. Experimentally we instantiated several cooperation and communication mechanisms to improve the stability of hierarchy, achieving both data efficiency and learning efficiency. To our knowledge, very few prior works have discussed the model exploitation problem in goal-conditioned HRL.  This research not only provides a SOTA HRL algorithm but also demonstrates the potential of integrating the learned dynamics model into goal-conditioned HRL, which is expected to draw the attention of researchers to such a direction.

\bibliographystyle{plainnat}
\bibliography{neurips_2023}

\clearpage

\appendix

\section{Lipschitz Property of the Q-function w.r.t. action}
\label{sec:mgp_proof}
In this appendix, we provide a brief proof for \textbf{Proposition 1}. More detailed proof can be found in \chreplaced{\cite{blonde2022lipschitzness, gao2022robust}}{\cite{gao2022robust}}. 
We start out with a lemma that helps with subsequent derivation.
\begin{lemma} Assume policy gradients w.r.t. input actions in an MDP admit a bound at any time $t$: $\Vert \frac{\partial \pi^*(a_{t+1}|s_{t+1})}{\partial a_t}\Vert_F \leq L_{\pi}$.
Then the following holds for any non-negative integer $c$ and $t$:
\begin{equation}
\begin{aligned}
\big\vert \nabla_{a_t} &\mathbb{E}_{s_{t+c}|s_t}[r^*(s_{t+c}, a_{t+c})]\big\vert \\
&\leq L_{\pi} \mathbb{E}_{s_{t+c}|s_t} \big\vert \nabla_{a_{t+1}} \mathbb{E}_{s_{t+c}|s_{t+1}}[r^*(s_{t+c}, a_{t+c})]\big\vert
\end{aligned}
\end{equation}
\label{lemma:1}
\end{lemma}

\begin{proof}
\begin{equation}
\begin{aligned}
&\big\vert \nabla_{a_t} \mathbb{E}_{s_{t+c}|s_t}[r^*(s_{t+c}, a_{t+c})] \\ & \quad= \big\vert \nabla_{a_t} \mathbb{E}_{s_{t+1}|s_t} \mathbb{E}_{s_{t+c}|s_{t+1}}[r^*(s_{t+c}, a_{t+c})]\cdot \frac{\partial a_{t+1}}{\partial a_{t}} \big\vert \\
&\quad\leq \big\vert \nabla_{a_t} \mathbb{E}_{s_{t+1}|s_t} \mathbb{E}_{s_{t+c}|s_{t+1}}[r^*(s_{t+c}, a_{t+c})] \big\vert \cdot \big\vert \frac{\partial a_{t+1}}{\partial a_{t}} \big\vert \\
&\quad\leq \big\vert \nabla_{a_t} \mathbb{E}_{s_{t+1}|s_t} \mathbb{E}_{s_{t+c}|s_{t+1}}[r^*(s_{t+c}, a_{t+c})] \big\vert \cdot L_{\pi} \\
&\quad= \mathbb{E}_{s_{t+1}|s_t} \big\vert \nabla_{a_t} \mathbb{E}_{s_{t+c}|s_{t+1}}[r^*(s_{t+c}, a_{t+c})] \big\vert \cdot L_{\pi}
\end{aligned}
\end{equation}
\end{proof}

\begin{remark}
Lemma \ref{lemma:1} gives the discrepancy of reward gradients starting from adjacent states. We can apply this lemma sequentially and infer the upper bound of reward gradients:
\begin{equation}
\begin{aligned}
&\big\vert \nabla_{a_t} \mathbb{E}_{s_{t+c}|s_t}[r^*(s_{t+c}, a_{t+c})]\big\vert \\
&= \big\vert \nabla_{a_t} \mathbb{E}_{s_{t+1}|s_t} \mathbb{E}_{s_{t+c}|s_{t+1}}[r^*(s_{t+c}, a_{t+c})]\cdot \frac{\partial a_{t+1}}{\partial a_{t}} \big\vert \\
&\leq \mathbb{E}_{s_{t+1}|s_t} \big\vert \nabla_{a_t} \mathbb{E}_{s_{t+c}|s_{t+1}}[r^*(s_{t+c}, a_{t+c})] \big\vert \cdot L_{\pi} \\
&\leq \mathbb{E}_{s_{t+1}|s_t} \mathbb{E}_{s_{t+2}|s_{t+1}} \dots \mathbb{E}_{s_{t+c}|s_{t+c-1}} \\
&\qquad\qquad\qquad \big\vert \nabla_{a_t} \mathbb{E}_{s_{t+c}|s_{t+c}}[r^*(s_{t+c}, a_{t+c})] \big\vert \cdot (L_{\pi})^c \\
&= \mathbb{E}_{s_{t+c}|s_t} \big\vert \nabla_{a_t} r^*(s_{t+c}, a_{t+c}) \big\vert \cdot (L_{\pi})^c
\end{aligned}
\end{equation}
\label{remark:1}
\end{remark}

\begin{prop}
Let $\pi^*(a_t|s_t)$ and $r^*(s_t,a_t)$ be the policy and the reward function in an MDP. Suppose there are the upper bounds of Frobenius norm of the policy and reward gradients w.r.t. input actions, i.e., $\Vert \frac{\partial \pi^*(a_{t+1}|s_{t+1})}{\partial a_t}\Vert_F \leq L_{\pi} < 1 $ and $\Vert \frac{\partial r^*(\chreplaced{s_{t}, a_{t}}{s_{t+1}, a_{t+1}})}{\partial a_t}\Vert_F \leq L_{r}$. Then the gradient of the learned Q-function w.r.t. action can be upper-bounded as:
\begin{equation}
\Vert \nabla_{a_t}Q_{\pi^*}(s_t,a_t) \Vert_F \leq \frac{\sqrt{N}L_r}{1-\gamma L_{\pi}}
\end{equation}
Where $N$ denotes the dimension of the action and $\gamma$ is the discount factor.
\end{prop}

\begin{proof}
\begin{equation}
\begin{aligned}
\Vert \nabla_{a_t}&Q_{\pi^*}(s_t,a_t) \Vert^2_F \\
&= \sum_{i=0}^N\left(\nabla_{a^i_t}Q_{\pi^*}(s_t,a_t)\right)^2 \\
&=\sum_{i=0}^N\left(\sum_{c=0}^{\infty}\gamma^c \nabla_{a^i_t} \mathbb{E}_{s_{t+c}|s_t}[r^*(s_{t+c}, a_{t+c})]\right)^2 \\
&\leq \sum_{i=0}^N\left(\sum_{c=0}^{\infty}\gamma^c \big\vert \nabla_{a^i_t} \mathbb{E}_{s_{t+c}|s_t}[r^*(s_{t+c}, a_{t+c})]\big\vert\right)^2
\end{aligned}
\label{proof:2}
\end{equation}
Meanwhile according to Remark \ref{remark:1}, we have:
\begin{equation}
\begin{aligned}
\big\vert \nabla_{a_t}& \mathbb{E}_{s_{t+c}|s_t}[r^*(s_{t+c}, a_{t+c})]\big\vert \\
&\leq \mathbb{E}_{s_{t+c}|s_t} \big\vert \nabla_{a^i_t} r^*(s_{t+c}, a_{t+c}) \big\vert \cdot (L_{\pi})^c \\
&\leq \mathbb{E}_{s_{t+c}|s_t} L_r \cdot (L_{\pi})^c\\
&=L_r \cdot (L_{\pi})^c
\end{aligned}
\end{equation}
Replacing the above gradient term, then the formula (\ref{proof:2}) can be rewritten as:
\begin{equation}
\begin{aligned}
\Vert \nabla_{a_t}&Q_{\pi^*}(s_t,a_t) \Vert^2_F \\
&\leq \sum_{i=0}^N\left(\sum_{c=0}^{\infty}\gamma^c \big\vert \nabla_{\chreplaced{a_t^i}{a_t}} \mathbb{E}_{s_{t+c}|s_t}[r^*(s_{t+c}, a_{t+c})]\big\vert\right)^2 \\
&\leq \sum_{i=0}^N\left(\sum_{c=0}^{\infty}\gamma^c \cdot L_r \cdot L_{\pi}^c\right)^2 = 
N \left(L_r\sum_{c=0}^{\infty}(\gamma \cdot L_{\pi})^c\right)^2 \\
&= N \left(\frac{L_r}{1-\gamma L_{\pi}} \right)^2
\end{aligned}
\end{equation}
The above inequality on the sqrt function then implies:
\begin{equation}
\begin{split}
\Vert \nabla_{a_t}Q_{\pi^*}(s_t,a_t) \Vert_F \leq \frac{\sqrt{N}L_r}{1-\gamma L_{\pi}}
\end{split}
\end{equation}
Which completes the proof.
\end{proof}

\newpage
\section{Algorithm table}

We provide the pseudo code below for this algorithm. Python-based implementation is available at \url{https://github.com/HaoranWang-TJ/GCMR_ACLG_official}.

\begin{algorithm}
\caption{Guided Cooperation via Model-based Rollout (GCMR)}\label{alg:cap}
\begin{algorithmic}
\Statex \textbf{Input:}\\
\begin{itemize}
\item \textbf{Key hyper-parameters}: the number of candidate goals $k$, gradient penalty loss coefficient $\lambda_{gp}$, one-step planning term coefficient$\lambda_{osrp}$, soft update rate of the shift magnitude within relabeling $\epsilon$.
\item \textbf{General hyper-parameters}: the subgoal scheme $\eta$ (set to 0/1 for the relative/absolute scheme), training batch number $BN$, higher-level update frequency $H_c$, learning frequency of dynamics models $D_c$, initial steps without using dynamics models $t_{dm}$, usage frequencies of gradient penalty and planning term $GP_c$, $OP_c$.
\end{itemize}
\State Initialize all actor and critic networks with random parameters $\theta_{lo}$, $\theta_{hi}$, $\phi_{lo}$, $\phi_{hi}$.
\State Initialize the dynamics models $\Gamma_{\xi}$.
\State $\mathcal{D}_{lo} \gets \emptyset$, $\mathcal{D}_{hi} \gets \emptyset$ \Comment{Initialize replay buffers}
\While{True}
\State $t \gets 0$
\State Reset the environment and get the state $s_t$ and episode terminal
signal $done$.
\Repeat
\If{$t \equiv 0$ (mod $c$)}
\State Generate subgoal $sg_t \sim \pi \left(sg|s_t, g;\theta_{hi}\right)$.
\Else
   \State Obtain \chreplaced{new subgoals}{goal} through the transition function $sg_t = sg_{t-1} + (\neg\eta) \cdot \varphi(s_{t-1} - s_t)$.
\EndIf 
\State $a_t \sim \pi \left(a|s_t,sg_t;\theta_{lo}\right)$ \Comment{Sample lower-level action}
\State $s_{t+1}$, $r_t$ $\gets {\rm env.step}(a_t)$ \protect\\ \Comment{Perform action $a_t$ in the environment}
\State $\mathcal{D}_{lo} \gets \mathcal{D}_{lo} \cup\{\tau_{lo}\}$, $\mathcal{D}_{hi} \gets \mathcal{D}_{hi} \cup\{\tau_{hi}\}$ \protect\\ \Comment{Store transitions into buffers}
\State $t \gets t + 1$
\Until{$done$ is $true$.}
\If{$t > t_{dm}$}

\If{$t \equiv 0$ (mod $D_c$)}
\State Train the dynamics models $\Gamma_{\xi}$.
\EndIf

\State Randomly sample experiences from replay buffers.
\State Relabel subgoals via the rollout-based off-policy correction.

\If{$t \equiv 0$ (mod $GP_c$)}
\RepeatN{$5$}
\State $\mathcal{L}(\phi_{lo}) \gets \mathcal{L}_{gp}(\phi_{lo}) + \mathcal{L}(\phi_{lo})$ \protect\\ \Comment{Plug the gradient penalty into critic loss}
\EndRepeatN
\EndIf
\If{$t \equiv 0$ (mod $OP_c \times H_c$)}
\State $\mathcal{L}(\theta_{lo})  \gets \mathcal{L}_{osrp} + \mathcal{L}(\theta_{lo})$ \protect\\ \Comment{Plug the planning term into actor loss}
\EndIf

\EndIf
\EndWhile
\end{algorithmic}
\end{algorithm}

\newpage
\section{Environment Details}
\label{sec:environment_setting}
Most experiments were conducted under the same environments as that in \cite{kim2021landmark}, including \textit{\textbf{Point Maze}}, \textit{\textbf{Ant Maze
(W-shape)}}, \textit{\textbf{Ant Maze (U-shape)}}, \textit{\textbf{Pusher}}, and \textit{\textbf{Reacher}}. Further detail is available in public repositories \footnote{Our code is available at \url{https://github.com/HaoranWang-TJ/GCMR_ACLG_official}\label{link_us}} \footnote{\url{https://github.com/junsu-kim97/HIGL.git}}. Besides, we introduced a more challenging locomotion environment, i.e., Ant Maze-Bottleneck, to validate the stability and robustness of the proposed GCMR in long-horizon and complicated tasks requiring delicate controls.
\paragraph{Ant Maze-Bottleneck}
The Ant Maze-Bottleneck environment was first introduced in \cite{leed2022hrl}, which provided implementation details in public repositories in \footref{link_us} \footnote{\url{https://github.com/jayLEE0301/dhrl_official.git}}. In this study, we resized it to be the same as the other mazes. Specifically, the size of the environment is $12 \times 12$. At training time, a goal point was selected randomly from a two-dimensional planar where both $x$- and $y$-axes range from -2 to 10. At evaluation time, the goal was placed at (0, 8) (i.e., the top left corner). A minimum threshold of competence was set to an L2 distance of 2.5 from the goal. Each episode would be terminated after 600 steps.

\section{Experimental Configurations}
\subsection{Network Structure}
For a fair comparison, we adopted the same architecture as \cite{zhang2020generating, kim2021landmark}. In detail, all actor and critic networks had two hidden layers composed of a fully connected layer with 300 units and a ReLU nonlinearity function. The output layer in actor networks had the same number of cells as the dimension of the action space and normalized the output to $[-1, 1]$ using the $tanh$ function. After that, the output was rescaled to the range of action space. 

For the dynamics model, the ensemble size $B$ was set to 5, a recommended value in \cite{chua2018deep}. Each of the ensembles was instantiated with three layers, similar to the above actor network. Yet, each hidden layer had 256 units and was followed by the Swish activation \cite{ramachandran2018searching}. Note that units of the output layer were twice as much as the actor because the action distribution was represented with the mean and covariance of a Gaussian.

The Adam optimizer was utilized for all networks.

\subsection{Hyper-Parameter Settings}
\label{sec:params_setting}
In Table \ref{table:1}, we list common hyper-parameters used across all environments. Hyper-parameters that differ across the environments are presented in Table \ref{table:2}. These hyper-parameters involving HIGL remained the same as proposed by \cite{kim2021landmark}. Besides, the update speed of the shift magnitude of goals was set at 0.01.
\begin{table}[H]
  \caption{Hyper-parameters across all environments.}
  \label{table:1}
  \centering
  \begin{threeparttable}
  \begin{tabular}{c|cc}
    \toprule
    \multirow{2}{*}{Hyper-parameters}    & \multicolumn{2}{c}{Value}\\
    \cmidrule(r){2-3}
     & Higher-level  & Lower-level     \\
     \midrule
    Actor learning rate     & 0.0001 & 0.0001      \\
    Critic learning rate     & 0.001      & 0.001  \\
    Soft update rate & 0.005       & 0.005  \\
    $\gamma$ & 0.99       & 0.95  \\
    Reward scaling & 0.1       & 1.0  \\
    Training frequency $H_c$ &10 & 1\\
    Batch size& 128       & 128 \\
    Candidates' number  & 10 & \\
    $\lambda^{\rm ACLG}_{\rm landmark}$ & 1.0 &\\
    Shift Magnitude $\delta_{sg}$ & 20$\sim$30\tnote{a} &\\
    $\lambda_{gp}$ & &1.0$\sim$10\tnote{b} \\
    $\lambda_{osrp}$ & & 0.0005 $\sim$ 0.00005\tnote{c} \\
         \makecell{$\lambda_{gp}$ usage frequencies\\ $GP_c$} &  & 5      \\
     \makecell{$\lambda_{osrp}$ usage frequencies\\ $OP_c$} &  & 10      \\
    
    \midrule
     & \multicolumn{2}{c}{Dynamics model}     \\
     \midrule
     Ensemble number & \multicolumn{2}{c}{5} \\
     Learning rate & \multicolumn{2}{c}{0.005} \\
     Batch size& \multicolumn{2}{c}{256} \\
     Training epochs & \multicolumn{2}{c}{20 $\sim$ 50} \\
     \bottomrule
  \end{tabular}
  \begin{tablenotes}
 \item[a] $\delta_{sg}$ is set to 20 for small mazes, such as Ant Maze (U-shape, W-shape), and 30 for larger mazes, such as the Large Ant Maze (U-shape).
\item[b] Only $\lambda_{gp}=10$ for the Ant Maze-Bottleneck, while $\lambda_{gp}$ are set to 1.0 for others.
\item[c] Only $\lambda_{osrp}=0.00005$ for the FetchPush and FetchPickAndPlace, while $\lambda_{osrp}$ are set to 0.0005 for others.
\end{tablenotes}   
\end{threeparttable}
\end{table}

\begin{table}[H]
  \caption{Hyper-parameters that differ across the environments.}
  \label{table:2}
  \centering
  \begin{threeparttable}
  \begin{tabular}{c|cc}
    \toprule
    \multirow{2}{*}{Hyper-parameters}     & Maze-based\tnote{a}& Arm-based\tnote{b} \\
    \cmidrule(l){2-3}
     & \multicolumn{2}{c}{Higher-level}\\
     \midrule
    High-level action frequency     & 10 & 5      \\
    Exploration strategy & \makecell{Gaussian\\($\sigma=1.0$)} & \makecell{Gaussian\\($\sigma=0.2$)}      \\
    \midrule
     & \multicolumn{2}{c}{Lower-level}     \\
     \midrule
     Exploration strategy & \makecell{Gaussian\\($\sigma=1.0$)} & \makecell{Gaussian\\($\sigma=0.1$)}      \\
     $\lambda_{\rm adj}$ & 20.0 & 0     \\
    \midrule
     & \multicolumn{2}{c}{Dynamics model}     \\
     \midrule
     Training frequency $D_c$ & 2000 & 500 \\
     Initial steps w/o model $t_{dm}$ &20000&10000\\
     \bottomrule
  \end{tabular}
  \begin{tablenotes}
\item[a] Maze-based environments include the Point Maze, Ant Maze (U/W-shape), and Ant Maze-Bottleneck.
\item[b] Arm-based environments are the Pusher, Reacher, FetchPush, and FetchPickAndPlace.
\end{tablenotes}   
\end{threeparttable}
\end{table}

\begin{table}[H]
  \caption{Number of Landmarks}
  \label{table:landmark_num}
  \centering
  \begin{threeparttable}
  \begin{tabular}{c|c|c|ccc}
    \toprule
    Environments & GCMR$+$ACLG\tnote{*}&ACLG\tnote{*}&HIGL\tnote{*}&DHRL&PIG\\
    \midrule
     Ant Maze (U-shape)&\multicolumn{2}{c|}{60$+$60}&20$+$20&300&400\\
     Ant Maze (W-shape)&\multicolumn{2}{c|}{60$+$60}&60$+$60&300&400\\
     Point Maze&\multicolumn{2}{c|}{60$+$60}&20$+$20&300&200\\
     Pusher and Reacher&\multicolumn{2}{c|}{20$+$20}&20$+$20&300&80\\
    \makecell[c]{FetchPush and \\ FetchPickAndPlace}&\multicolumn{2}{c|}{60$+$60}&20$+$20&300&80\\
     \bottomrule
  \end{tabular}
  \begin{tablenotes}
\item[*] Where $+$ connects the numbers of coverage-based landmarks and novelty-based landmarks.
\end{tablenotes}   
\end{threeparttable}
\end{table}

\section{Additional Experiments}

\subsection{GCMR Solely Employed for Goal-Reaching Tasks}
\label{sec:gcmr_sole}
In our experiments, the proposed GCMR was used as an additional plugin. Of course, GCMR can be solely applicable to goal-reaching tasks. We conducted experiments in the Point Maze and Ant Maze (U-shape) tasks. As shown in Figure~\ref{gcmr_sole}, the results indicate that GCMR can be used independently and achieve similar results to HIGL. Meanwhile, in Figure~\ref{gcmr_sole_params}, we investigate how the gradient penalty term and the one-step planning term impact the final performance when solely using the GCMR method in Ant Maze (U-shape). Figure~\ref{gcmr_sole_params} illustrates similar conclusions as those of Figure~2 in the "\textit{\textbf{Main Paper}}".

\begin{figure}[htbp]
\centering
\subfloat{\includegraphics[width=0.7\textwidth]{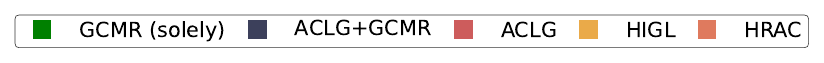}}\vspace{-5mm}
\setcounter{subfigure}{0}
\subfloat[Point Maze]{\includegraphics[width=0.4\textwidth]{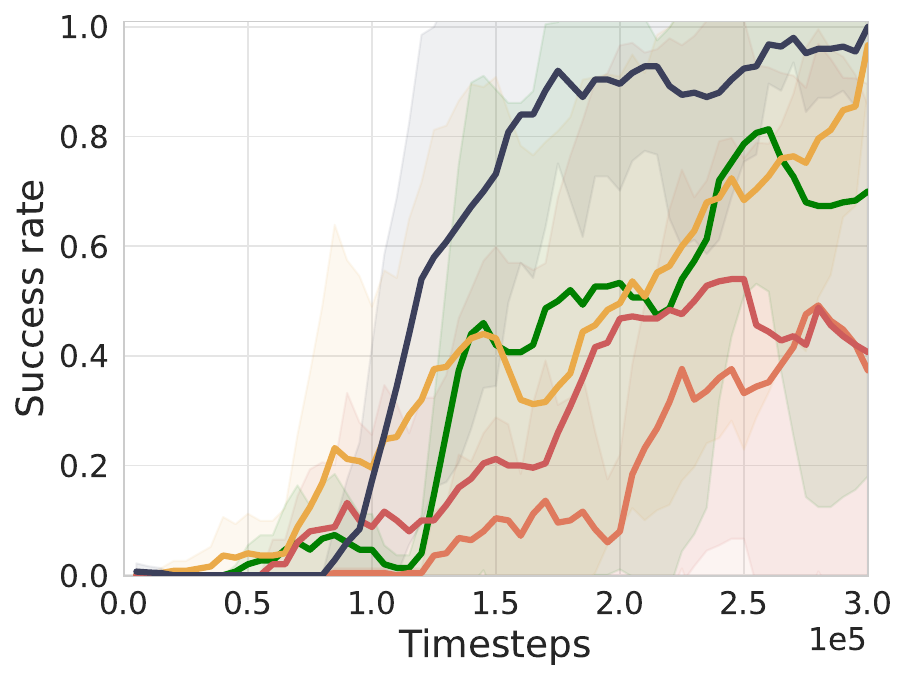}}
\subfloat[Ant Maze (U-shape)]{\includegraphics[width=0.4\textwidth]{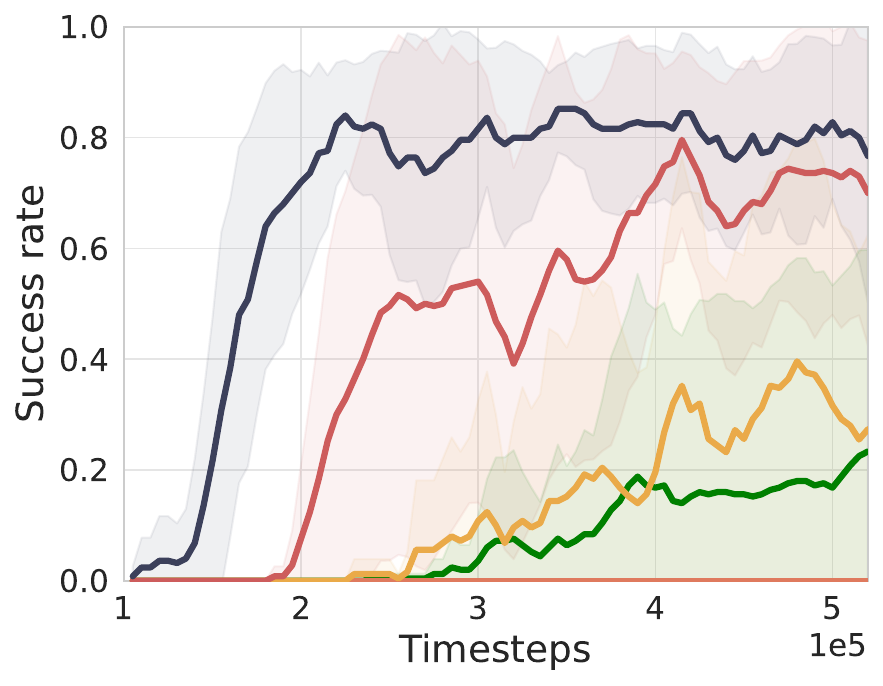}}
\caption{Performance when solely using the GCMR method in environments (a) Point Maze and (b) Ant Maze (U-shape). The solid lines represent the mean across five runs. The transparent areas represent the standard deviation.}
\label{gcmr_sole}
\end{figure}

\begin{figure}[H]
\centering
\subfloat[Varying $\lambda_{gp}$]{\includegraphics[width=0.4\textwidth]{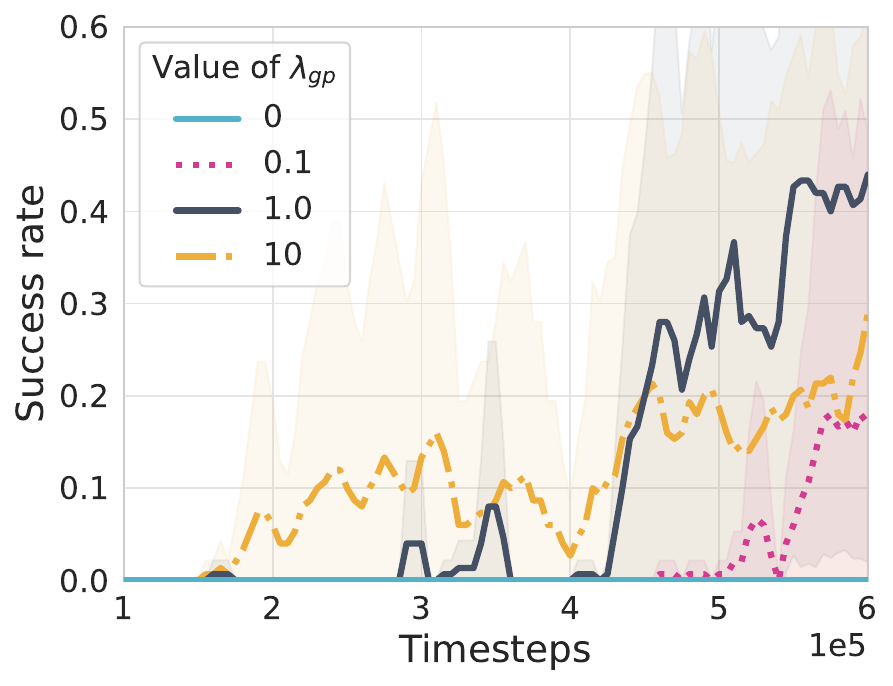}
\label{gcmr_sole_mgp_lambda}}
\subfloat[Varying $\lambda_{osrp}$]{\includegraphics[width=0.4\textwidth]{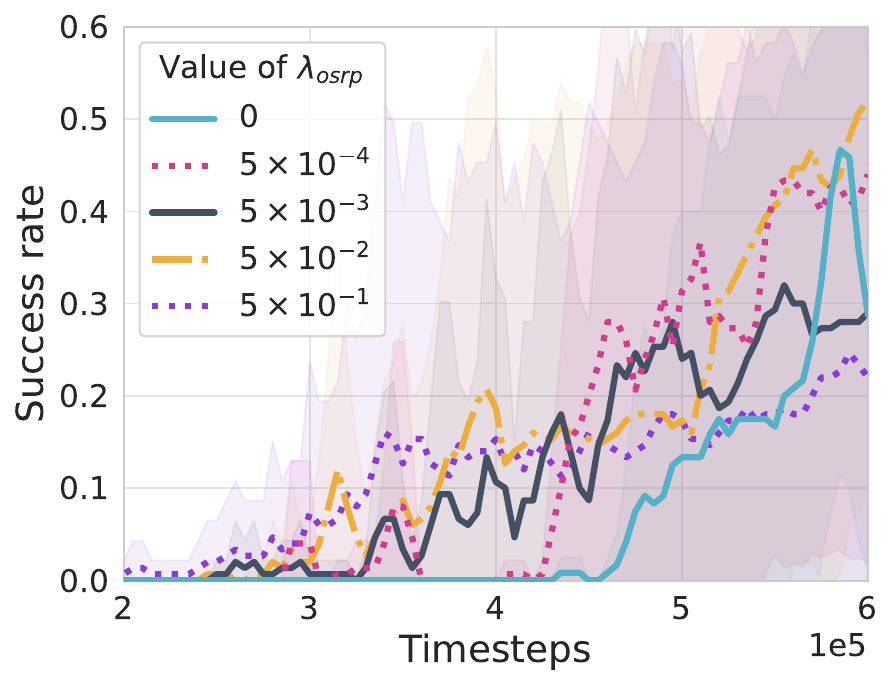}
\label{gcmr_sole_osrp_lambda}}\\
\caption{Impact of varying (a) $\lambda_{gp}$ and (b) $\lambda_{osrp}$ in the Ant Maze (U-shape) environment, when solely using the GCMR method.}
\label{gcmr_sole_params}
\end{figure}

\subsection{Quantification of Extra Computational Cost}
\label{sec:quan_comp_cost}
The extra computational cost is primarily composed of three aspects: 
\begin{itemize}
\item (1) \textbf{dynamic model training}, 
\item (2) model-based gradient penalty and one-step rollout planning in \textbf{low-level policy training}, 
\item (3) goal-relabeling in \textbf{high-level policy training}.
\end{itemize}
Therefore, we ran both methods ACLG+GCMR and ACLG on Ant Maze (U-shape) for $0.7 \times 10^5$ steps to compare their runtime in these aspects.

\begin{table}[H]
  \caption{Quantification of Extra Computational Cost}
  \label{table:quan_comp_cost}
  \centering
  \begin{tabular}{c|cccc}
    \toprule
    Training time (s) & Total&Dynamic model&Low-level policy&High-level policy\\
    \midrule
     ACLG+GCMR & 5065.18 &16.51&3603.01&854.64 \\
     ACLG & 1845.63 &0&632.97&614.01 \\
     \bottomrule
  \end{tabular}
\end{table}

\end{document}